\documentclass{article}

\usepackage{microtype}
\usepackage{graphicx}
\usepackage{subcaption}
\usepackage{booktabs} 

\usepackage{hyperref}

\usepackage[preprint]{icml2026}

\usepackage{amsmath,amsfonts,bm}

\def\eqref#1{equation~\ref{#1}}

\def\1{\bm{1}}

\DeclareMathAlphabet{\mathsfit}{\encodingdefault}{\sfdefault}{m}{sl}
\SetMathAlphabet{\mathsfit}{bold}{\encodingdefault}{\sfdefault}{bx}{n}

\usepackage{hyperref}
\usepackage{url}
\usepackage{comment}
\usepackage{algorithm}
\usepackage{algorithmic}
\usepackage{amssymb}
\usepackage{amsthm}

\usepackage{graphicx}
\usepackage{subcaption}
\usepackage{booktabs}
\usepackage[table]{xcolor}

\usepackage{amsmath}
\usepackage{amssymb}
\usepackage{mathtools}
\usepackage{amsthm}

\usepackage[capitalize,noabbrev]{cleveref}

\theoremstyle{plain}
\newtheorem{theorem}{Theorem}[section]

\newtheorem{lemma}[theorem]{Lemma}

\theoremstyle{definition}

\newtheorem{assumption}[theorem]{Assumption}
\theoremstyle{remark}
\newtheorem{remark}[theorem]{Remark}

\usepackage[textsize=tiny]{todonotes}

\icmltitlerunning{One Intervention Per Component Is Enough}

\begin{document}

\twocolumn[
  \icmltitle{One Intervention per Component is Enough:\\ Towards Identifiability in Linear Stochastic Dynamics from Steady State}



  \icmlsetsymbol{equal}{*}

  \begin{icmlauthorlist}
    \icmlauthor{Saber Salehkaleybar}{Leiden}

  \end{icmlauthorlist}

  \icmlaffiliation{Leiden}{Leiden Institute of Advanced Computer Science, Leiden University, The Netherlands}

  \icmlcorrespondingauthor{Saber Salehkaleybar}{s.salehkaleybar@liacs.leidenuniv.nl}

  \icmlkeywords{Machine Learning, ICML}

  \vskip 0.3in
]



\printAffiliationsAndNotice{}  

\begin{abstract}
We study the problem of recovering the parameters of a multivariate Ornstein–Uhlenbeck (OU) process from steady-state observational and interventional data. In many applications, such as large-scale gene perturbation experiments, only stationary “snapshot” measurements are available, making standard stochastic differential equation estimation methods that rely on time-series trajectories inapplicable.
We first establish an identifiability result: one intervention per strongly connected component (SCC) of the drift graph suffices to recover all OU process parameters generically up to a global scaling factor. This holds provided that the SCC condensation graph is connected with a single root and certain spectral nondegeneracy assumptions hold. We propose a recursive learning algorithm that orders SCCs topologically and, for each component, isolates its marginal dynamics and solves a linear system derived from the steady-state moment equations, leveraging parameters recovered for upstream components.
Building on this theoretical foundation, we propose a regularized least-squares estimator that jointly minimizes residuals of the steady-state mean and covariance equations across observational and interventional data. Experimental results validate our theoretical findings in recovering parameters of the underlying OU process.
\end{abstract}

\section{Introduction}

Understanding the structure and parameters of dynamical systems from data is a fundamental problem in many scientific domains, including systems biology~\citep{bansal2006inference,marbach2012wisdom}, neuroscience~\citep{paninski2010new}, and economics~\citep{hamilton2020time}. In many settings, such as gene regulatory networks, the underlying dynamics can be modeled by a stochastic differential equation (SDE) whose steady-state distribution encodes both the interaction structure and kinetic parameters~\citep{gardiner2009stochastic,villaverde2016structural}. Accurately recovering these parameters enables causal representation of interactions among variables, and predictions under unseen perturbations~\citep{pearl2009causality}.
However, in experimental biology and related areas, data are often available only in the form of ``snapshot'' measurements~\citep{cao2019single,schiebinger2019optimal}. The time-series trajectories (required by many SDE estimation methods~\citep{zhang2024trajectory,ohstable}) are expensive or infeasible to obtain at scale. In such cases, interventions such as gene knockdown~\citep{datlinger2017pooled} provide additional information by selectively intervening on parts of the system, potentially resolving non-identifiability issues that arise from observational data alone~\citep{peters2017elements}.

A substantial body of work addresses causal inference in linear stochastic systems. 
\citep{varando2020graphical,dettling2024lasso} focus on recovering drift structures from steady-state data using the Lyapunov equation. These approaches rely solely on observational measurements and do not leverage interventions, which can limit identifiability. More recently, \cite{lorch2024causal}, focused on fitting the observational/interventional distributions without providing parameter recovery guarantees.
Intervention-aware models, such as in~\citep{rohbeck2024bicycle}, achieve good empirical prediction, but the identifiability results are very restricted (see the related work section for more details).

In this paper, our aim is to incorporate interventions into the learning process, establishing identifiability results on recovering the parameters of SDEs and developing learning algorithms grounded in these theoretical results. Our contributions are as follows:
   \begin{itemize}
    \item We establish an identifiability result for recovering the parameters of a multivariate Ornstein–Uhlenbeck (OU) process from steady-state observational and interventional data. We show that a single intervention per SCC of the drift graph is sufficient to recover the parameters up to a global scaling\footnote{In OU processes, from the stationary distribution, parameters are identifiable at most up to a global scaling, an intrinsic ambiguity that cannot be resolved from stationary data alone.} if the SCC condensation graph is connected with a single root and certain spectral/rank nondegeneracy assumptions (Assumptions \ref{ass:moment-nondeg} and \ref{assum:lambda_W}) hold. To prove this, we first consider the case where the drift matrix is a single SCC and show that by solving the linear system of equations for the first and second moments (based on observational data and data from one intervention), we can recover all the parameters up to a global scaling (Theorem \ref{th:Single SCC}).   
    
    \item 
    For the general case with multiple SCCs (Theorem \ref{th:multiSCC}), we provide a recursive algorithm that topologically orders the SCCs and, for each component, isolates its marginal dynamics. This reduces the problem to a sequence of single-SCC cases. For each SCC, the algorithm solves a linear system (derived from the steady-state moment equations for the observational and interventional settings) using parameters recovered for upstream components. We also provide an example that if the SCC condensation graph has more than one root, then it is impossible to learn all the parameters with a single global scaling.

    \item We show that changes in the steady-state means under single-node interventions
reveal the SCC-level structure of the graph. In particular,
Theorem~\ref{th:learning_SCC_mu} characterizes the set of nodes whose means
change after an intervention as the downstream region of the intervened SCC.
Consequently, one intervention per SCC is sufficient to recover the SCC
decomposition and a topological order of the DAG over SCCs
(Remark~\ref{remak:scc_recovery}). This structural recovery step does not
require the spectral/rank nondegeneracy assumptions used for
parameter identification.
    
    \item Building on these theoretical findings, we formulate a regularized least-squares optimization problem that jointly minimizes the squared residuals of the steady-state mean and covariance equations across both observational and interventional data. Empirical results validate the identifiability results in recovering parameters and predicting unseen interventions.
\end{itemize}

\section{Problem Formulation}
\textbf{Ornstein--Uhlenbeck process}: The Ornstein--Uhlenbeck (OU) process is a continuous-time stochastic process that satisfies the stochastic differential equation (SDE):
\begin{equation}
d\mathbf{x} = (-\boldsymbol{\Lambda} \mathbf{x} + \mathbf{b})\, dt + \boldsymbol{\sigma} \, d\mathbf{W}_t,   
\label{eq:OU process}
\end{equation}
where $\mathbf{x} \in \mathbb{R}^n$ is the state vector, $\boldsymbol{\Lambda} \in \mathbb{R}^{n \times n}$ is a drift matrix (assumed positive stable),
$\mathbf{b} \in \mathbb{R}^n$ is a constant vector representing external input,
 $\boldsymbol{\sigma} \in \mathbb{R}^{n \times n}$ is the diffusion matrix, and
 $d\mathbf{W}_t$ is an $n$-dimensional Wiener process.

The stationary distribution of the OU process has a multivariate normal distribution with the following mean and covariance:
\begin{itemize}
  \item Mean:
  \begin{equation}
   \boldsymbol{\mu} = \mathbb{E}[\mathbf{x}_\infty] = \boldsymbol{\Lambda}^{-1} \mathbf{b},
   \label{eq:mean_steady}
  \end{equation}
  \textcolor{black}{where $\mathbf{x}_{\infty}$ denotes a random vector distributed according to the stationary distribution of the OU process.}
  \item Covariance:
  \begin{equation}
    \boldsymbol{\Lambda} \boldsymbol{\Sigma} + \boldsymbol{\Sigma} \boldsymbol{\Lambda}^\top = \boldsymbol{\sigma} \boldsymbol{\sigma}^\top,
    \label{eq:lyapanouv}
  \end{equation}
  where $\boldsymbol{\Sigma} \in \mathbb{R}^{n \times n}$ is the steady-state covariance matrix: $\boldsymbol{\Sigma} = \mathbb{E}[(\mathbf{x}_\infty - \boldsymbol{\mu})(\mathbf{x}_\infty - \boldsymbol{\mu})^\top].$
\end{itemize}
We assume that the diffusion power matrix is diagonal and positive: $\mathbf{D} := \boldsymbol{\sigma} \boldsymbol{\sigma}^{\!\top} = \operatorname{diag}(d_1,\dots,d_n) \succ 0$.

\textbf{Intervention:} We define an intervention on $i$-th coordinate of $\mathbf{x}$ as follows, where the $i$-th row of the drift matrix $\boldsymbol{\Lambda}$ is modified to remove influence from all other variables, while preserving its self-regulation. Formally, we define the modified drift matrix $\widetilde{\boldsymbol{\Lambda}}^{(i)}$ as:
\[
\widetilde{\Lambda}^{(i)}_{kj} =
\begin{cases}
\Lambda_{kj}, & i \neq k, \\
0, & i = k,\, j \neq i, \\
\Lambda_{ii}, & i = k,\, j = i.
\end{cases}
\]
That is, we zero out all off-diagonal entries in the $i$-th row, but retain $\Lambda_{ii}$, analogous to a knockout perturbation that isolates a variable from its regulators. This notion of intervention is often considered in the causal SDE literature, see, e.g., \cite{hansen2014causal} and \cite{boeken2024dynamic}, where some post-intervention SDEs are formulated in this form. 

The modified dynamic under intervention is: $
d\mathbf{x} = (-\widetilde{\boldsymbol{\Lambda}}^{(i)} \mathbf{x} + \mathbf{b})\, dt + \boldsymbol{\sigma} \, d\mathbf{W}_t.$
In the intervened system, the mean is denoted by $\boldsymbol{\mu}^{(i)}$ and satisfies:
\begin{equation}
\boldsymbol{\mu}^{(i)} = \mathbb{E}[\mathbf{x}_\infty^{(i)}] = \left( \widetilde{\boldsymbol{\Lambda}}^{(i)} \right)^{-1} \mathbf{b},
\label{eq:mean_intervened}
\end{equation}
assuming $\widetilde{\boldsymbol{\Lambda}}^{(i)}$ is invertible. 

Similarly, the steady-state covariance $\boldsymbol{\Sigma}^{(i)}$ satisfies the Lyapunov equation:
\begin{equation}
\widetilde{\boldsymbol{\Lambda}}^{(i)} \boldsymbol{\Sigma}^{(i)} + \boldsymbol{\Sigma}^{(i)} \left( \widetilde{\boldsymbol{\Lambda}}^{(i)} \right)^\top = \boldsymbol{\sigma} \boldsymbol{\sigma}^\top.
\label{eq:Lyapanouv_intervened} 
\end{equation}

We assume that, for every intervention considered,
\(\widetilde{\boldsymbol{\Lambda}}^{(i)}\) is positive stable, so the interventional stationary mean and covariance are well-defined.

Our goal is to recover the parameters of OU process, i.e., $\mathbf{\Lambda}, \mathbf{b}$, and $\boldsymbol{D}$ from the observational mean and covariance $(\boldsymbol{\mu}, \mathbf{\Sigma})$ and a collection of interventional means and covariances $\{(\boldsymbol{\mu}^{(i)}, \mathbf{\Sigma}^{(i)})\}_{i\in \mathcal{I}}$ where $\mathcal{I}$ is the set of coordinates intervened on. Please note that the parameters are identifiable only up to a global scaling because the stationary first-moment and second-moment equations are invariant under global scaling of parameters.

\textbf{Graph definitions:} In the following, we briefly review the graph-theoretic concepts used throughout the paper:
\\
\textbf{- Strongly connected components.}
For a directed graph $G$, a \emph{strongly connected component (SCC)} is a maximal subset of
nodes $C$ such that for every pair $u,v \in C$ there exists a directed path from
$u$ to $v$ and a directed path from $v$ to $u$. The SCCs of $G$ form a partition of its nodes.
\\
\textbf{- Condensation graph.}
Given the SCCs $C_1,\dots,C_K$ of $G$, the \emph{condensation graph} is a directed graph whose nodes correspond to the SCCs and which contains a directed
edge $C_i \to C_j$ whenever there exists an edge in $G$ from some node in $C_i$ to some node in
\begingroup\color{black}
$C_j$. By construction, the condensation graph is always a directed acyclic graph (DAG). A root (or source) SCC has no incoming edge from another SCC. Connectedness of this DAG refers to its underlying undirected graph; with finitely many SCCs, a unique root reaches every SCC.
\endgroup
\\
\textbf{- Topological order over SCCs.}
A \emph{topological order} of the SCCs is any ordering of the
nodes of the condensation graph such that all directed edges point from earlier to later
components. Equivalently, $C_i$ may appear before $C_j$ in the order if and only if there is no
directed path from $C_j$ to $C_i$ in the condensation graph.

\section{Identifiability Results}
For a drift matrix $\boldsymbol{\Lambda}$, define the associated directed graph $G(\boldsymbol{\Lambda})$ by considering an edge $j \to i$ if and only if $\Lambda_{ij} \ne 0$. In the following, first we consider that $G(\boldsymbol{\Lambda})$ is strongly connected. Under some spectral/rank non-degeneracy assumptions, we show that the parameters of the OU process can be recovered generically up to some global scaling by just having one intervention on any coordinate.
We then consider the case where \(G(\boldsymbol{\Lambda})\) has multiple SCCs. In this setting, we first show that changes in interventional means identify the SCCs and a topological ordering over the SCC condensation DAG\footnote{This structural step does not require the spectral/rank nondegeneracy assumptions.}. Given this
\begingroup\color{black}
SCC-level structure, we prove generic identification up to a single global scaling in the graph-compatible model class, under the unique-root, intervention-coverage, and isolated-realization hypotheses of Theorem~\ref{th:multiSCC}.
\endgroup

\subsection{Single SCC}
\begin{theorem}
  Consider the OU process in \eqref{eq:OU process} with true parameters $\mathbf{\Lambda}, \mathbf{b}$, and $\boldsymbol{D}$. Moreover, suppose that we have access to the observational steady-state mean and covariance $(\boldsymbol{\mu}, \boldsymbol{\Sigma})$ and an interventional mean and covariance $(\boldsymbol{\mu}^{(i)}, \boldsymbol{\Sigma}^{(i)})$ ($i$ can be any coordinate). If the graph $G(\boldsymbol{\Lambda})$ is strongly connected and certain spectral/rank non-degeneracy assumptions hold (See Assumption \ref{ass:moment-nondeg} and Assumption \ref{assum:lambda_W} in Appendix \ref{app:singleSCC}), generically\footnote{Please see the definition of genericity in  Appendix~\ref{app:genericity}.}, any other parameter triple $(\widehat{\boldsymbol{\Lambda}}, \widehat{\mathbf{b}}, \widehat{\mathbf{D}})$ that yields the same observational and interventional moments must satisfy:
\[
\widehat{\boldsymbol{\Lambda}} = c \boldsymbol{\Lambda}, \quad
\widehat{\mathbf{b}} = c \mathbf{b}, \quad
\widehat{\mathbf{D}} = c \mathbf{D},
\]
for some scalar $c > 0$.
\label{th:Single SCC}
\end{theorem}

All the proofs of theorems (if not given in the main body) are available in the appendix. The key idea in the proof is based on forming a system of linear equations according to \eqref{eq:mean_steady}, \eqref{eq:lyapanouv}, \eqref{eq:mean_intervened}, and \eqref{eq:Lyapanouv_intervened}, showing that any possible solution for this set of equations should be a scale of the true parameters. In particular, let
\[
\Theta
=
\begin{bmatrix}
\operatorname{vec}(\boldsymbol{\Lambda})\\[2pt]
\mathbf{b}\\[2pt]
\mathbf{d}
\end{bmatrix}
\in\mathbb{R}^{p},
\]
where $p:=n^{2}+2n$, and
$\mathbf d:=(d_{1},\dots ,d_{n})^{\!\top},$
where $d_{k}:=D_{kk}>0$. Moreover, the operator $\operatorname{vec}$ stacks the columns of its matrix
argument. Therefore, $\operatorname{vec}(\boldsymbol{\Lambda})\in\mathbb{R}^{n^{2}}$. Now, we can write the linear equations in the following form: $\mathbf{A}\,\Theta = \mathbf{0}$ where $\mathbf{A}\in\mathbb{R}^{\,m\times p}$ (refer to Appendix \ref{app:linear_system} for the definition of $\mathbf{A}$) and $m=n^{2}+3n$. In the proof, we show that every admissible solution of this linear system is a
positive scaling of the true parameters.

\begin{remark}
The nondegeneracy assumptions used in the proofs include spectral/rank conditions on moment-derived linear
systems. We do not
prove their genericity in this work; instead, we
evaluate them numerically in
Appendix~\ref{app:Spec_assumption}. The precise statements of the assumptions and
their roles in the proofs are deferred to Appendix~\ref{app:proofs} and
Appendix~\ref{app:spec_assumption_role}, respectively. Establishing genericity
of these conditions alone is left as a
direction for future work. For SCCs whose graph is a directed cycle,
Appendix~\ref{app:spec_assumption_role} provides additional justification for
the relevant injectivity condition.
\end{remark}

\subsection{General Case}
In the general case, where the graph 
$G(\mathbf{\Lambda})$ has multiple SCCs, assume that the intervention set 
$\mathcal{I}$ contains at least one intervention targeting a variable inside each SCC. Moreover, the DAG of SCCs is connected and there is exactly one root SCC. Under these conditions and certain spectral/rank non-degeneracy assumptions (see  Assumptions~\ref{ass:moment-nondeg} and ~\ref{assum:lambda_W}, in the appendix), we show that the parameters of the OU process can be recovered up to a global scaling. The key idea is to first recover a topological ordering over the SCCs by inspecting which means change under each intervention; this structural recovery step does not require the spectral/rank
nondegeneracy assumptions. The following theorem allows us to identify the set of variables in each SCC, as well as a topological ordering over the SCCs. With this structural information in hand, we then recursively identify the model parameters for each SCC by conditioning on previously resolved components, thereby reducing the multi-SCC case to a sequence of single-SCC problems.

\begin{theorem}\label{th:learning_SCC_mu}
Let $C$ be an SCC in $G(\boldsymbol{\Lambda})$ and let $i\in C$. 
Consider the intervention on $i$ and let $\boldsymbol{\mu}$ and $\boldsymbol{\mu}^{(i)}$ be the observational and interventional steady-state means, respectively. 

\smallskip
\noindent\emph{(Non-null case).} If the $i$-th row is non-null, i.e., $\exists\,j\neq i$ with $\Lambda_{ij}\neq 0$, then, generically, for the set of parameters $(\boldsymbol{\Lambda},\mathbf b)$ consistent with $G(\boldsymbol{\Lambda})$,
\[
\mu^{(i)}_k \neq \mu_k 
\quad\Longleftrightarrow\quad
k\in \mathrm{Desc}(C),
\]
where $\mathrm{Desc}(C)$ denotes the set of nodes in SCCs reachable from $C$ in the DAG of SCCs (including $C$ itself).

\noindent\emph{(Null case).} If the $i$-th row has no off-diagonals, i.e., $\Lambda_{ij}=0$ for all $j\neq i$, then $\boldsymbol{\mu}^{(i)}=\boldsymbol{\mu}$.
\end{theorem}

\begin{remark}
\label{remak:scc_recovery}
For an intervention on $i\in C$, define $R_i:=\{\,j:\ \mu^{(i)}_j\neq \mu_j\,\}$. Suppose there is at least one intervention in each SCC. Based on the above theorem, intersecting and differencing the sets $R_i$ across interventions identifies the SCCs and yields a topological order for the DAG over SCCs (see Appendix \ref{app:rec_SCC} for more details).
\end{remark}

\begin{remark}
\begingroup\color{black}
The same SCC-recovery argument extends to non-null soft row interventions (at least one per SCC), where
\endgroup
the intervened drift differs from \(\boldsymbol{\Lambda}\) only in row \(i\), but
the changed row is not necessarily obtained by hard zeroing. The set of coordinates whose means change is again generically the
downstream region of the intervened SCC. The proof is given in
Appendix~\ref{app:learning_SCC_soft}.
\end{remark}

\begin{theorem}
  Consider the OU process in \eqref{eq:OU process} with true parameters 
  $\boldsymbol{\Lambda}, \mathbf{b}, \mathbf{D}$. Suppose we have access to the 
  observational steady-state mean and covariance $(\boldsymbol{\mu}, \boldsymbol{\Sigma})$, 
  as well as to interventional means and covariances $(\boldsymbol{\mu}^{(i)}, \boldsymbol{\Sigma}^{(i)})$ 
  for at least one intervention in each SCC of $G(\boldsymbol{\Lambda})$. 
\begingroup\color{black}
  Assume that, for each non-singleton SCC and its selected intervention,
the admissible isolated SCC model admits a realization satisfying
Assumptions~\ref{ass:moment-nondeg} and~\ref{assum:lambda_W}.
If the DAG over SCCs of $G(\boldsymbol{\Lambda})$ is connected with a single
root SCC, then, generically, any other admissible parameter triple
\endgroup
$(\widehat{\boldsymbol{\Lambda}},
  \widehat{\mathbf b},
  \widehat{\mathbf D})$
compatible with the same graph $G(\boldsymbol{\Lambda})$
  that yields the same observational and interventional moments must satisfy
  \[
  \widehat{\boldsymbol{\Lambda}} = c \boldsymbol{\Lambda}, \quad
  \widehat{\mathbf{b}} = c \mathbf{b}, \quad
  \widehat{\mathbf{D}} = c \mathbf{D},
  \]
  for some scalar $c > 0$.
\label{th:multiSCC}
\end{theorem}
\begingroup\color{black}


\begin{remark} The graph-compatibility assumption in Theorem~\ref{th:multiSCC} can be replaced by the requirement that every competing model under consideration satisfies the mean-support property of Theorem~\ref{th:learning_SCC_mu} for each performed non-null intervention, with descendant sets computed in its own graph. Keeping the remaining hypotheses unchanged, matching interventional means then forces competing drifts to respect the recovered block ordering (i.e., a competing graph cannot contain an edge from a later block to an earlier block). The single-SCC assumptions provide injectivity on the full target-block matrix space, so the recursive reconstruction does not require competitors to have exactly the same edges within blocks or between earlier and later blocks as the true graph. Although the mean-support property holds generically for each fixed graph, requiring it of every competitor is an additional restriction. Genericity of the true parameters alone does not exclude exceptional competitors that violate this property. \end{remark}
\endgroup

\begin{proof}
Based on Theorem \ref{th:learning_SCC_mu}, we can infer the SCCs and also a topological ordering over them if there is at least one intervention in each SCC. Let us denote these SCCs based on the topological ordering as $C_1,C_2,\cdots, C_K$ where $K$ is the number of components. Suppose that we already learned the parameters of the OU process in the components $C_1,C_2,\cdots,C_r$ up to some global scaling. Now, we aim for learning the parameters in $C_{r+1}$.

We partition the state vector according to the SCC decomposition of $G(\boldsymbol{\Lambda})$:
\[
\mathbf{x} =
\underbrace{\mathbf{x}_P}_{C_1 \cup \cdots \cup C_r}
\;\oplus\;
\underbrace{\mathbf{x}_T}_{C_{r+1}}
\;\oplus\;
\underbrace{\mathbf{x}_F}_{C_{r+2} \cup \cdots \cup C_K},
\]
where the operator $\oplus$ denotes concatenation of subvectors corresponding to disjoint index sets.

The steady-state mean and covariance matrices are partitioned accordingly:
\[
\boldsymbol{\mu} =
\begin{bmatrix}
\boldsymbol{\mu}_P \\
\boldsymbol{\mu}_T \\
\boldsymbol{\mu}_F
\end{bmatrix},
\qquad
\boldsymbol{\Sigma} =
\begin{bmatrix}
\boldsymbol{\Sigma}_{PP} & \boldsymbol{\Sigma}_{PT} & \boldsymbol{\Sigma}_{PF} \\
\boldsymbol{\Sigma}_{TP} & \boldsymbol{\Sigma}_{TT} & \boldsymbol{\Sigma}_{TF} \\
\boldsymbol{\Sigma}_{FP} & \boldsymbol{\Sigma}_{FT} & \boldsymbol{\Sigma}_{FF}
\end{bmatrix}.
\]
Everything inside the $P$-block is assumed known, i.e., the blocks
$\boldsymbol{\Lambda}_{PP}$, $\mathbf{b}_P$, and $\mathbf{D}_{P}$, up to the same scaling $c$.

Because $\mathbf{x}_P$ and $\mathbf{x}_T$ are jointly Gaussian, we can write:
\[
\mathbf{x}_T = \mathbf{B}\,\mathbf{x}_P + \mathbf{r}, \quad
\text{where } \mathbf{B} := \boldsymbol{\Sigma}_{TP} \boldsymbol{\Sigma}_{PP}^{-1}.  
\]
The regression matrix $\mathbf{B}$ is computable directly from the observed moments, with no dependence on model parameters. Moreover, 
the residual term $\mathbf{r} := \mathbf{x}_T - \mathbf{B} \mathbf{x}_P$ is a Gaussian variable with
\begin{equation}
\boldsymbol{\mu}_{T \mid P} = \boldsymbol{\mu}_T - \mathbf{B}\,\boldsymbol{\mu}_P,
\qquad
\boldsymbol{\Sigma}_{T \mid P} = \boldsymbol{\Sigma}_{TT} - \boldsymbol{\Sigma}_{TP} \boldsymbol{\Sigma}_{PP}^{-1} \boldsymbol{\Sigma}_{PT}.
\label{eq:muTPSigmaTP}
\end{equation}

Substituting the definitions of \( \mathbf{B} \) and \( \mathbf{r} \) into the OU dynamics yields:
\begin{equation}
\begin{split}
  \dot{\mathbf{r}}
  = 
  -\boldsymbol{\Lambda}_{TT} \mathbf{r}&
  +
  \bigl( -\boldsymbol{\Lambda}_{TP}
         - \boldsymbol{\Lambda}_{TT} \mathbf{B}
         + \mathbf{B} \boldsymbol{\Lambda}_{PP} \bigr) \mathbf{x}_P
  \\
  &+
  (\mathbf{b}_T - \mathbf{B} \mathbf{b}_P)
  +
  (\boldsymbol{\sigma}_T \dot{\mathbf{W}}_T - \mathbf{B} \boldsymbol{\sigma}_P\dot{\mathbf{W}}_P).
\end{split}
  \label{eq:rSDE_pre}
\end{equation}
\textcolor{black}{To identify the residual moment equations, we use the cross-covariance block of the Lyapunov equation:}
\[
\textcolor{black}{
\boldsymbol{\Lambda}_{TP}\boldsymbol{\Sigma}_{PP}
+
\boldsymbol{\Lambda}_{TT}\boldsymbol{\Sigma}_{TP}
+
\boldsymbol{\Sigma}_{TP}\boldsymbol{\Lambda}_{PP}^{\!\top}
=
\mathbf 0.
}
\]
\textcolor{black}{Using $\mathbf B=\boldsymbol{\Sigma}_{TP}\boldsymbol{\Sigma}_{PP}^{-1}$, this gives}
\[
\textcolor{black}{
\boldsymbol{\Lambda}_{TP}
+
\boldsymbol{\Lambda}_{TT}\mathbf B
=
-\mathbf B\boldsymbol{\Sigma}_{PP}
\boldsymbol{\Lambda}_{PP}^{\!\top}
\boldsymbol{\Sigma}_{PP}^{-1}.
}
\]
\textcolor{black}{Combining this with the parent Lyapunov equation}
\[
\textcolor{black}{
\boldsymbol{\Lambda}_{PP}\boldsymbol{\Sigma}_{PP}
+
\boldsymbol{\Sigma}_{PP}\boldsymbol{\Lambda}_{PP}^{\!\top}
=
\mathbf D_P,
}
\]
\textcolor{black}{we obtain}
\begin{equation}
\textcolor{black}{
\boldsymbol{\Lambda}_{TP}
+
\boldsymbol{\Lambda}_{TT}\mathbf B
=
\mathbf B\boldsymbol{\Lambda}_{PP}
-
\mathbf B\mathbf D_P\boldsymbol{\Sigma}_{PP}^{-1}.
}
\label{eq:cancel_identity}
\end{equation}
\textcolor{black}{ For the mean, using}
\[
\textcolor{black}{
\boldsymbol{\Lambda}_{TP}\boldsymbol{\mu}_P
+
\boldsymbol{\Lambda}_{TT}\boldsymbol{\mu}_T
=
\mathbf b_T,
\qquad
\boldsymbol{\mu}_T
=
\boldsymbol{\mu}_{T\mid P}
+
\mathbf B\boldsymbol{\mu}_P,
}
\]
\textcolor{black}{together with \eqref{eq:cancel_identity}, gives}
\begin{equation}
\textcolor{black}{
\boldsymbol{\Lambda}_{TT}\boldsymbol{\mu}_{T\mid P}
=
\mathbf b_T
-
\mathbf B
\left(
\mathbf b_P
-
\mathbf D_P\boldsymbol{\Sigma}_{PP}^{-1}\boldsymbol{\mu}_P
\right).
}
\label{eq:residual_mean_equation}
\end{equation}
\textcolor{black}{Similarly, for the residual covariance,}
\begin{equation}
\textcolor{black}{
\boldsymbol{\Lambda}_{TT}\boldsymbol{\Sigma}_{T\mid P}
+
\boldsymbol{\Sigma}_{T\mid P}\boldsymbol{\Lambda}_{TT}^{\!\top}
=
\mathbf D_T+\mathbf B\mathbf D_P\mathbf B^{\!\top}.
}
\label{eq:residual_cov_equation}
\end{equation}

Now, based on the first and second moments of the residual which are given above, we can identify the parameters of component $T$ and also $\boldsymbol{\Lambda}_{TP}$ up to the same scaling. In particular, we have:
\\
\textbf{$\mathbf{\Lambda}_{TT}$:}
\textcolor{black}{Note that the corresponding diffusion power matrix of the residual moment equation is $\mathbf{D}_T+\mathbf{B}\mathbf{D}_P\mathbf{B}^{\!\top}$, and therefore it is not necessarily diagonal. Nevertheless, the proof of Theorem~\ref{th:Single SCC} can be adapted to recover $\boldsymbol{\Lambda}_{TT}$ up to the same scaling $c$ by solving the residual moment equations; see Appendix~\ref{app:multiSCC}.}
\\
\textbf{$\mathbf{\Lambda}_{TP}$:}
\textcolor{black}{Having $\boldsymbol{\Lambda}_{PP}$, $\mathbf D_P$, and $\boldsymbol{\Lambda}_{TT}$ up to the same scaling $c$, ~\eqref{eq:cancel_identity} gives}
\[
\textcolor{black}{
\boldsymbol{\Lambda}_{TP}
=
\mathbf B\boldsymbol{\Lambda}_{PP}
-
\mathbf B\mathbf D_P\boldsymbol{\Sigma}_{PP}^{-1}
-
\boldsymbol{\Lambda}_{TT}\mathbf B.
}
\]
\textcolor{black}{Therefore, $\boldsymbol{\Lambda}_{TP}$ is recovered with the same scaling.}
\\
\textbf{$\mathbf{b}_T$ and $\boldsymbol{D}_T$:}
According to \eqref{eq:mean_steady}, we have
\[
\mathbf{b}_T
=
\boldsymbol{\Lambda}_{TT}\boldsymbol{\mu}_T
+
\boldsymbol{\Lambda}_{TP}\boldsymbol{\mu}_P .
\]
Since we recovered $\boldsymbol{\Lambda}_{TT}$ and $\boldsymbol{\Lambda}_{TP}$ with the same scaling factor, $\boldsymbol{b}_T$ is identifiable with the same scaling from the above equation.

Regarding $\boldsymbol{\sigma}_T$ (or diffusion power matrix $\mathbf{D}_T$), from the Lyapunov equation for the covariance matrix of residual (in other words, $\boldsymbol{\Sigma}_{T|P}$), we have:
\[
\boldsymbol{\Lambda}_{TT} \boldsymbol{\Sigma}_{T \mid P}
+ \boldsymbol{\Sigma}_{T \mid P} \boldsymbol{\Lambda}_{TT}^{\!\top}
= \mathbf{D}_T + \mathbf{B}\,\mathbf{D}_P\,\mathbf{B}^{\!\top}.
\]
\textcolor{black}{Therefore,}
\[
\textcolor{black}{
\mathbf D_T
=
\operatorname{diag}\!\left(
\boldsymbol{\Lambda}_{TT}\boldsymbol{\Sigma}_{T\mid P}
+
\boldsymbol{\Sigma}_{T\mid P}\boldsymbol{\Lambda}_{TT}^{\!\top}
-
\mathbf B\mathbf D_P\mathbf B^{\!\top}
\right),
}
\]
\textcolor{black}{and hence $\mathbf D_T$ is learned with the same scaling. This completes the recursive step and the proof.}
\end{proof}

\begin{remark}
The connectedness assumption in Theorem~\ref{th:multiSCC} is necessary. 
If the DAG over SCCs is disconnected, then each disconnected part of the DAG can only be 
identified up to its own scaling factor. Moreover, if there are multiple root SCCs, 
the parameters of the OU process may not be recovered up to a single global scaling. 
An example illustrating this case is provided in Appendix~\ref{app:Example}.
\end{remark}

Building on the above, we design a recursive algorithm that proceeds in two stages (the pseudo-code is given in Algorithm \ref{alg:ROULI}):
\begin{itemize}
    \item In the first stage, we use the changes in interventional means to identify the SCCs of the drift graph \(G(\boldsymbol{\Lambda})\), along with a topological ordering over these components. This structural information is inferred using Theorem~\ref{th:learning_SCC_mu}.
    
    \item In the second stage, we iterate over the SCCs in topological order. For each component \(T\), we treat the union of previously processed SCCs as \(P\), and condition on \(\mathbf{x}_P\) to isolate the marginal dynamics of \(\mathbf{x}_T\). Using the conditional moments \((\boldsymbol{\mu}_{T\mid P}, \boldsymbol{\Sigma}_{T\mid P})\), we first recover \(\boldsymbol{\Lambda}_{TT}\) up to a scaling. Then, leveraging the structure of the OU dynamics, we recover \(\boldsymbol{\Lambda}_{TP},\; \boldsymbol{b}_T\), and \(\boldsymbol{D}_T\) up to the same scaling. The procedure continues recursively until all components have been identified.
\end{itemize}

\begin{algorithm}[t]
\caption{Recursive OU Learning with Interventions}
\begin{algorithmic}[1]
\STATE \textbf{Input:} Observational and interventional means and covariances \((\boldsymbol{\mu}, \boldsymbol{\Sigma}), \{(\boldsymbol{\mu}^{(i)}, \boldsymbol{\Sigma}^{(i)})\}_{i \in \mathcal{I}}\)

\vspace{1mm}
\STATE \textbf{Phase 1: Learn SCCs}
\STATE Identify SCCs, \(C_1, C_2, \dots, C_K\), and a topological ordering over them from mean changes using Theorem~\ref{th:learning_SCC_mu}

\vspace{1mm}
\STATE \textbf{Phase 2: Recover parameters per SCC}
\FOR{each component \(T = C_j\) in topological order}
    \STATE Let \(P := C_1 \cup \cdots \cup C_{j-1}\) be the union of previously processed SCCs
    \STATE Compute conditional mean \(\boldsymbol{\mu}_{T \mid P}\) and covariance \(\mathbf{\Sigma}_{T \mid P}\) according to \eqref{eq:muTPSigmaTP}
    \STATE Recover \(\boldsymbol{\Lambda}_{TT}\) using the interventional moments with the same scaling in $P$
    \STATE Recover \(\boldsymbol{\Lambda}_{TP},\; \boldsymbol{b}_T,\; \boldsymbol{D}_T\) up to the same scaling using cross-covariances and stationary conditions
\ENDFOR
\STATE \textbf{Output:} Drift matrix \(\boldsymbol{\Lambda}\), input vector \(\boldsymbol{b}\), and diffusion matrix \(\boldsymbol{D}\) (up to a global scaling)
\end{algorithmic}
\label{alg:ROULI}
\end{algorithm}

\begin{remark}
Rather than directly running Algorithm~\ref{alg:ROULI}, one can alternatively
construct a global linear system from the observational and interventional
first- and second-moment equations, as in the proof of
Theorem~\ref{th:Single SCC}. If this system has a one-dimensional null space,
then all parameters are identified up to a global scaling. Theorems
\ref{th:Single SCC} and~\ref{th:multiSCC} provide conditions under which this
one-dimensional ambiguity is the only ambiguity, while
Theorem~\ref{th:learning_SCC_mu} provides the SCC-level structural information
used by the recursive procedure.
\end{remark}

\begin{remark}\label{rem:moment_invariance}
Although the parameters $(\boldsymbol{\Lambda}, \boldsymbol{b}, \boldsymbol{D})$ are identifiable
only up to a global scaling, this ambiguity does not affect the prediction of stationary
moments. In particular, our identifiability result implies that any alternative parameter triple consistent with the
observational and interventional steady state is of the form $\hat{\boldsymbol{\Lambda}} = c\boldsymbol{\Lambda}, 
\hat{\boldsymbol{b}} = c\boldsymbol{b}, 
\hat{\boldsymbol{D}} = c\boldsymbol{D},$
for some scalar $c>0$. For the steady-state mean, we have $\boldsymbol{\mu}=\boldsymbol{\Lambda}^{-1}\boldsymbol{b}$, and
$\hat{\boldsymbol{\mu}}
=\hat{\boldsymbol{\Lambda}}^{-1}\hat{\boldsymbol{b}}
=(c\boldsymbol{\Lambda})^{-1}(c\boldsymbol{b})
=\boldsymbol{\Lambda}^{-1}\boldsymbol{b}
=\boldsymbol{\mu},$
so the scaling cancels out. The same holds for the steady-state covariance where $\boldsymbol{\Sigma}$ is defined as the unique
solution of the Lyapunov equation
$\boldsymbol{\Lambda}\boldsymbol{\Sigma}+\boldsymbol{\Sigma}\boldsymbol{\Lambda}^\top=\boldsymbol{D},$
under the assumption that $\boldsymbol{\Lambda}$ is stable. Under the scaled parameters,
the covariance $\hat{\boldsymbol{\Sigma}}$ satisfies
$
\hat{\boldsymbol{\Lambda}}\hat{\boldsymbol{\Sigma}}
+\hat{\boldsymbol{\Sigma}}\hat{\boldsymbol{\Lambda}}^\top
=\hat{\boldsymbol{D}}
\;\Longleftrightarrow\;
c\boldsymbol{\Lambda}\hat{\boldsymbol{\Sigma}}
+\hat{\boldsymbol{\Sigma}}c\boldsymbol{\Lambda}^\top
=c\boldsymbol{D}
\;\Longleftrightarrow\;
\boldsymbol{\Lambda}\hat{\boldsymbol{\Sigma}}
+\hat{\boldsymbol{\Sigma}}\boldsymbol{\Lambda}^\top
=\boldsymbol{D}$.
By uniqueness of the solution to the Lyapunov equation for a stable drift matrix, we obtain
$\hat{\boldsymbol{\Sigma}}=\boldsymbol{\Sigma}$. The same arguments can be applied to an intervened drift
$\widetilde{\boldsymbol{\Lambda}}^{(i)}$, so the predicted moments for unseen
interventions are also invariant to the global scaling.
\end{remark}

\section{Learning Algorithm}
In finite-sample settings, plugging in empirical moments generally yields full-column-rank systems, which admit only the trivial solution of zero vector. 
Nevertheless, the identifiability result in Theorem \ref{th:Single SCC} motivates replacing true moments with empirical ones and relaxing the equations into a least-squares objective.
Specifically, in the observational case, the mean vector $\boldsymbol{\mu}$ and covariance matrix $\boldsymbol{\Sigma}$ satisfy the linear system in \eqref{eq:mean_steady}
and \eqref{eq:lyapanouv}, respectively. Moreover, for each intervention $i \in \mathcal{I}$, the mean vector $\boldsymbol{\mu}^{(i)}$ and $\boldsymbol{\Sigma}^{(i)}$ satisfy equations in \eqref{eq:mean_intervened} and \eqref{eq:Lyapanouv_intervened}, respectively.

For the set of free parameters $\Theta = (\operatorname{vec}(\boldsymbol{\Lambda}), \mathbf{b}, \mathbf{d})$, where $\mathbf{d}$ is the diagonal of matrix $\mathbf{D}$, we define the following least-square objective:
\begin{equation}
\begin{split}
&\mathcal{L}(\Theta)
= 
\alpha_O \left(
\left\| \boldsymbol{\Lambda}\,\hat{\boldsymbol{\mu}} - \mathbf{b} \right\|_2^2
+ \left\| \boldsymbol{\Lambda}\,\hat{\boldsymbol{\Sigma}} + \hat{\boldsymbol{\Sigma}}\,\boldsymbol{\Lambda}^{\!\top} - \mathbf{D} \right\|_F^2
\right) \\
&+ \alpha_I \sum_{i \in \mathcal{I}}
\left(
\left\| \widetilde{\boldsymbol{\Lambda}}^{(i)} \,\hat{\boldsymbol{\mu}}^{(i)} - \mathbf{b} \right\|_2^2
\right)\\
&+ \alpha_I \sum_{i \in \mathcal{I}}
\left( \left\| \widetilde{\boldsymbol{\Lambda}}^{(i)} \,\hat{\boldsymbol{\Sigma}}^{(i)} + \hat{\boldsymbol{\Sigma}}^{(i)} (\widetilde{\boldsymbol{\Lambda}}^{(i)})^{\!\top} - \mathbf{D} \right\|_F^2
\right),
\end{split}
\end{equation}
where \(\alpha_O, \alpha_I > 0\) are weighting coefficients; we often give a larger weight to observational terms since observational samples are often more abundant and their estimates are more accurate.
Moreover, $\hat{\boldsymbol{\mu}}$, $\hat{\boldsymbol{\Sigma}}$, $\hat{\boldsymbol{\mu}}^{(i)}$, $\hat{\boldsymbol{\Sigma}}^{(i)}$, $i\in \mathcal{I}$ are the unbiased estimates of first and second moments in the observational and interventional settings and $\|\cdot\|_F$ is the Frobenius norm.
To promote sparsity in the drift graph, we add an $\ell_1$ penalty on the off-diagonal elements of $\boldsymbol{\Lambda}$:
$\mathcal{R}(\boldsymbol{\Lambda}) = \gamma \sum_{i \neq j} |\Lambda_{ij}|,$
where $\gamma>0$ controls the sparsity level. Therefore, the optimization problem becomes
\begin{equation}
\min_{\Theta}
\ \mathcal{L}(\Theta) + \mathcal{R}(\boldsymbol{\Lambda}),
\end{equation}
\begingroup\color{black}
subject to positive stability of $\boldsymbol{\Lambda}$ and the intervened drifts, with $d_k > 0$ for all $k$. Because the moment equations are homogeneous in
\((\boldsymbol{\Lambda},\mathbf b,\mathbf D)\), we can impose a scale-fixing
normalization, such as \(\operatorname{tr}(\boldsymbol{\Lambda})=n\), to remove the global
scaling ambiguity. 
\endgroup

\section{Related Work}
Herein, we mainly review methods that perform inference from stationary distributions of dynamical systems, rather than from full trajectories\footnote{There are several surveys on causal discovery from temporal data (e.g., \cite{gong2024causal,hasansurvey})}.

\textbf{Identifiability and learning from steady state in SDE models.}
A line of work on graphical continuous Lyapunov models treats the steady-state covariance as the solution of a continuous Lyapunov equation. In this setting, \cite{varando2020graphical} proposed an $l_1$-regularized estimator to recover sparse drift structure from observational snapshots. \cite{dettling2023identifiability} subsequently analyzed identifiability in this framework, proving that when only observational covariances are available and the diffusion matrix is known, the drift is globally identifiable from the covariance if and only if the drift graph is simple, meaning it contains no directed two-cycles. While these results provide valuable insights, they are limited to observational data, cannot recover models with two-cycles in the drift, and do not handle unknown diffusion. In contrast, our work incorporates interventional data, allows unknown diagonal diffusion, and shows that a single intervention in each SCC (under some conditions on the DAG of SCCs) suffices for recovery up to a global scaling. More recently, \cite{dettling2024lasso} proposed a Lasso-based estimator for recovering the drift structure of linear SDEs from stationary observational covariances and analyzed its identifiability properties. However, their framework does not incorporate interventions. \textcolor{black}{Very recently, \cite{zweig2025towards} studied identifiability of linear SDEs under mean-shift interventions, assuming a low-rank drift matrix. In contrast, we allow general drift matrices and base identifiability on the SCC structure of the drift graph rather than on a low rank assumption.}

Other approaches aim to learn stationary SDE models without focusing on identifiability. \cite{lorch2024causal} proposed a kernel deviation-from-stationarity objective that measures how far a candidate SDE’s stationary distribution deviates from the empirical distribution. Their framework can accommodate cycles and generalizes well to unseen interventions, but its goal is density fitting in reproducing kernel Hilbert spaces rather than moment-based parameter identification. Our work differs by deriving explicit graphical conditions under which the OU parameters are recoverable from first and second moments.

Recent empirical work has also combined steady-state dynamics with interventions. \cite{rohbeck2024bicycle} introduced Bicycle, a model in which interventions alter a subset of parameters. Bicycle achieves good performance in both structure recovery and prediction under out-of-distribution interventions in single-cell datasets. However, it is designed primarily as a predictive model and provides identifiability guarantees only when interventions are performed on all coordinates except one. In contrast, we gave identifiability results under substantially weaker requirements, i.e., one intervention per SCC.

\cite{boege2025conditional} have studied the conditional independence relations implied by sparsity in the drift of stationary multivariate diffusions. Their results link graph structure to conditional independencies in the stationary state distribution. Nonetheless, this line of work does not address the recovery of drift or diffusion parameters, nor how interventions could enable identifiability. 

Finally, there exists related work for deterministic linear ODEs rather than stochastic OU models. For example, \cite{wang2024identifiability} analyzed the identifiability of linear ODE systems with hidden confounders from time series or discretely sampled trajectories and derived conditions under which latent confounding can be resolved. 

\textbf{Gene perturbation prediction from steady state.}
In systems biology, several methods aim to predict the effects of genetic perturbations directly from stationary data. For instance, \cite{sethuraman2023nodags} proposed NODAGS-Flow, which learns nonlinear cyclic causal structures from interventional steady-state gene expression by fitting residual normalizing flows, producing predictions and plausible graphs. While providing good performance in practice, NODAGS-Flow is a likelihood-based method without formal identifiability guarantees on recovering the parameters of the underlying system. 

Other works focus on high-performing predictive models. For instance, \cite{roohani2022gears} proposed GEAR, which combines graph neural networks with prior biological network information to predict transcriptional responses to single or multiple gene perturbations, showing improved generalization to unseen combinations. \cite{yu2025perturbnet} proposed PerturbNet, which uses conditional invertible flows to model the distributional effects of unseen perturbations. PerturBench \citep{wu2024perturbench} provides a unified benchmark suite for single-cell perturbation modeling, facilitating comparison between predictive models. 
\begin{figure*}[t]
  \centering
  \begin{subfigure}[t]{.25\textwidth}
    \centering
    \includegraphics[width=\linewidth]{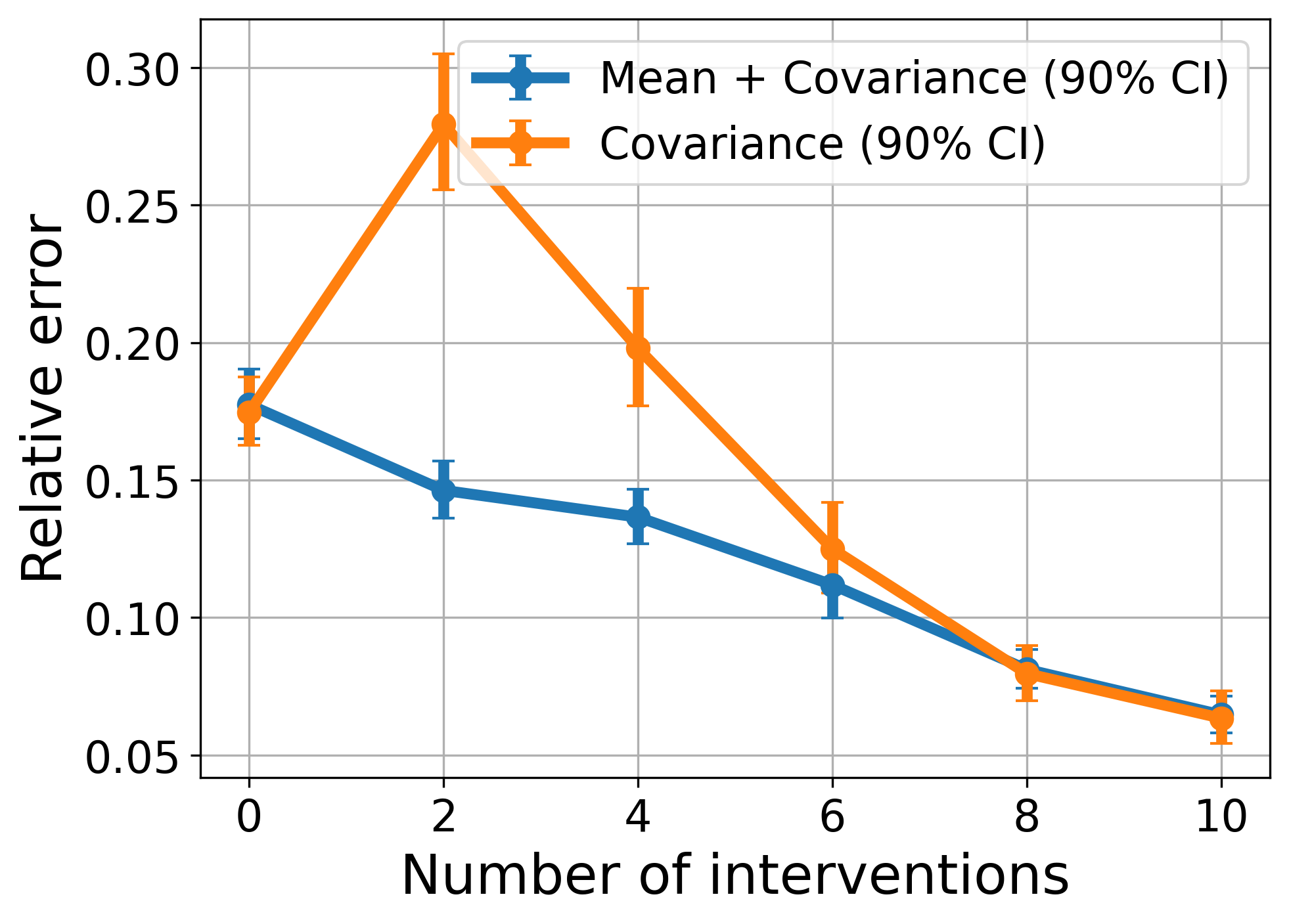}
    \caption{$\boldsymbol{\Lambda}$}
    \label{fig:sub1}
  \end{subfigure}%
  \begin{subfigure}[t]{.25\textwidth}
    \centering
    \includegraphics[width=\linewidth]{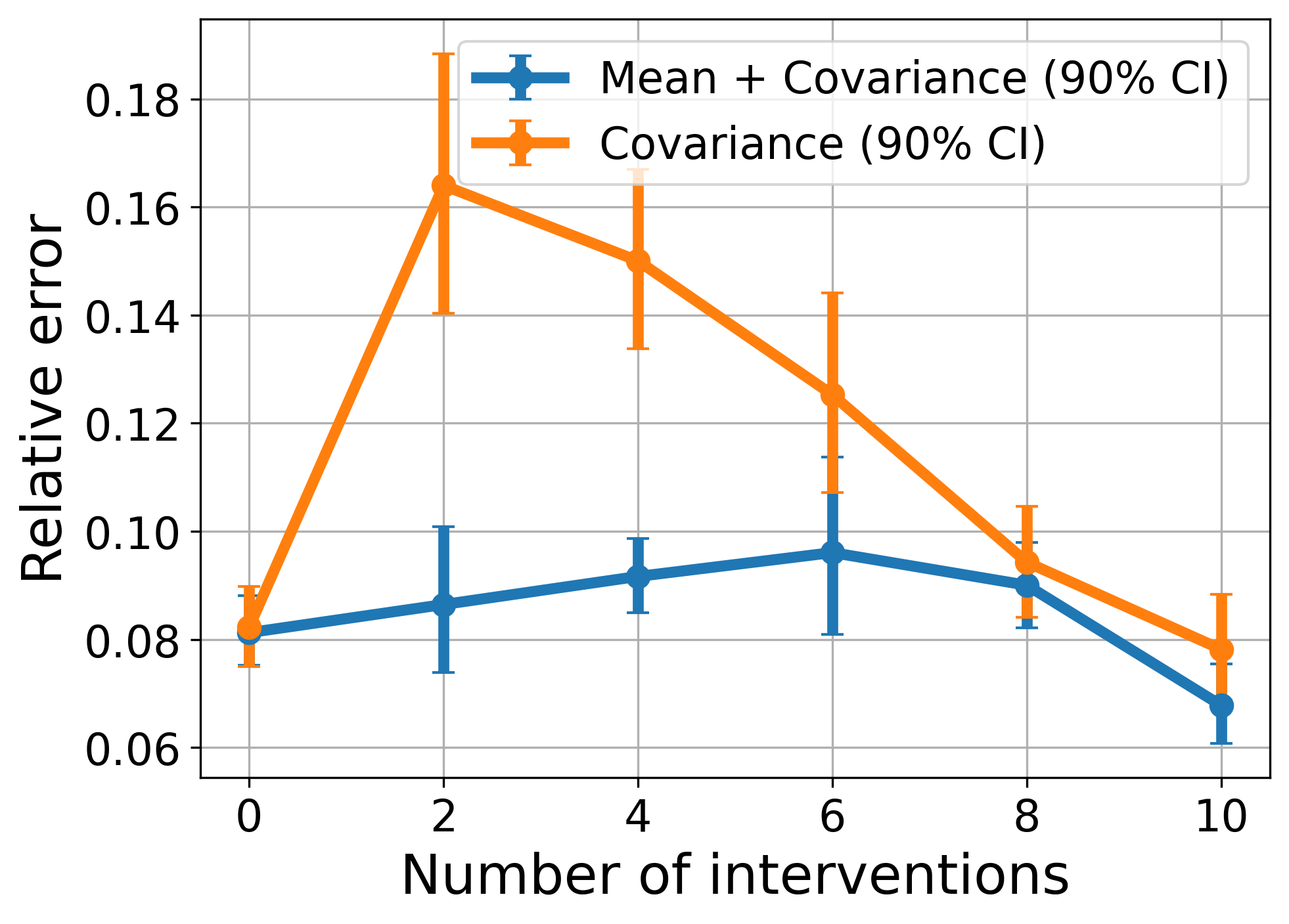}
    \caption{$\boldsymbol{D}$}
    \label{fig:sub2}
  \end{subfigure}%
  \begin{subfigure}[t]{.25\textwidth}
    \centering
    \includegraphics[width=\linewidth]{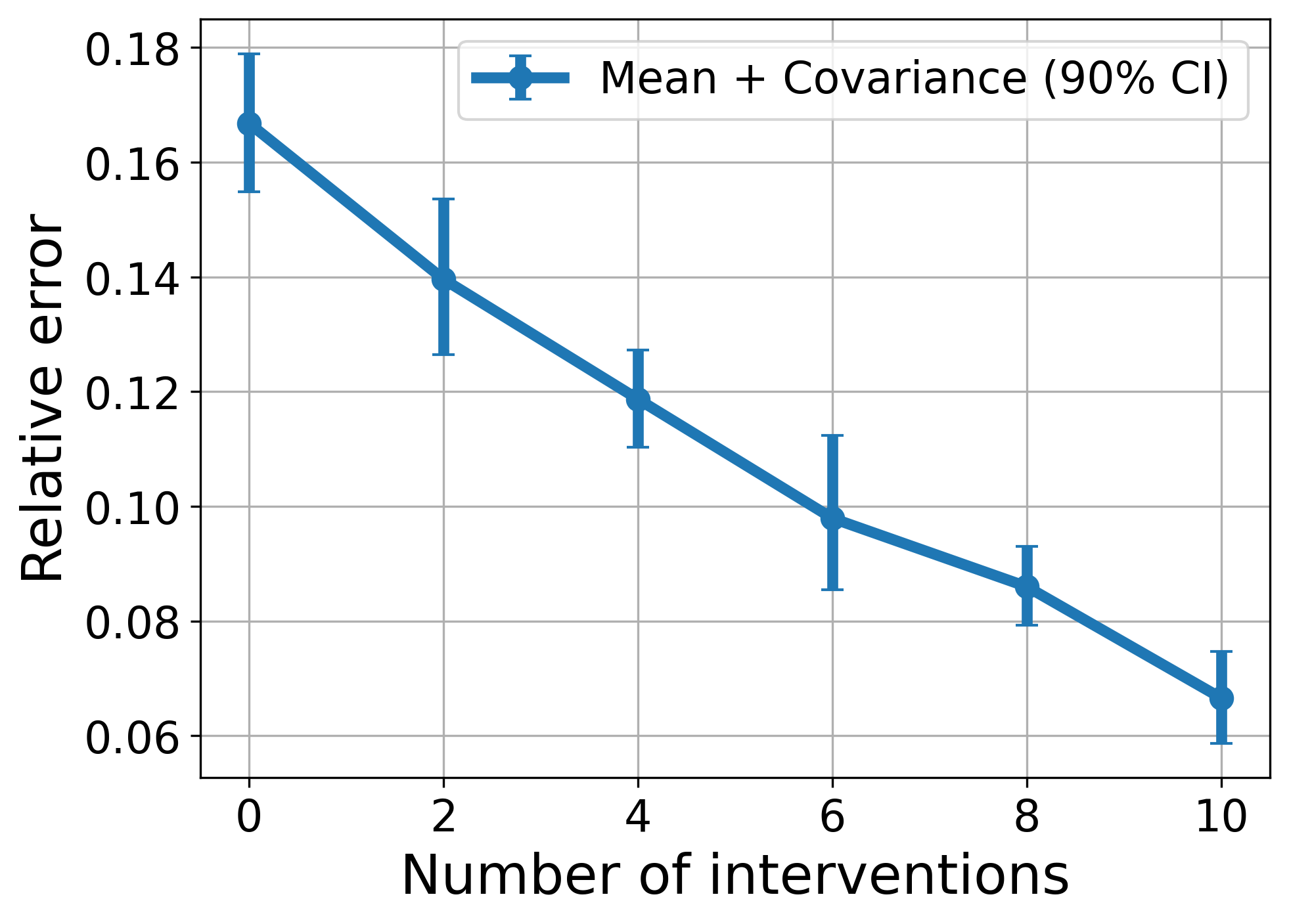}
    \caption{$\boldsymbol{b}$}
    \label{fig:sub3}
  \end{subfigure}%
  \begin{subfigure}[t]{.25\textwidth}
    \centering
    \includegraphics[width=\linewidth]{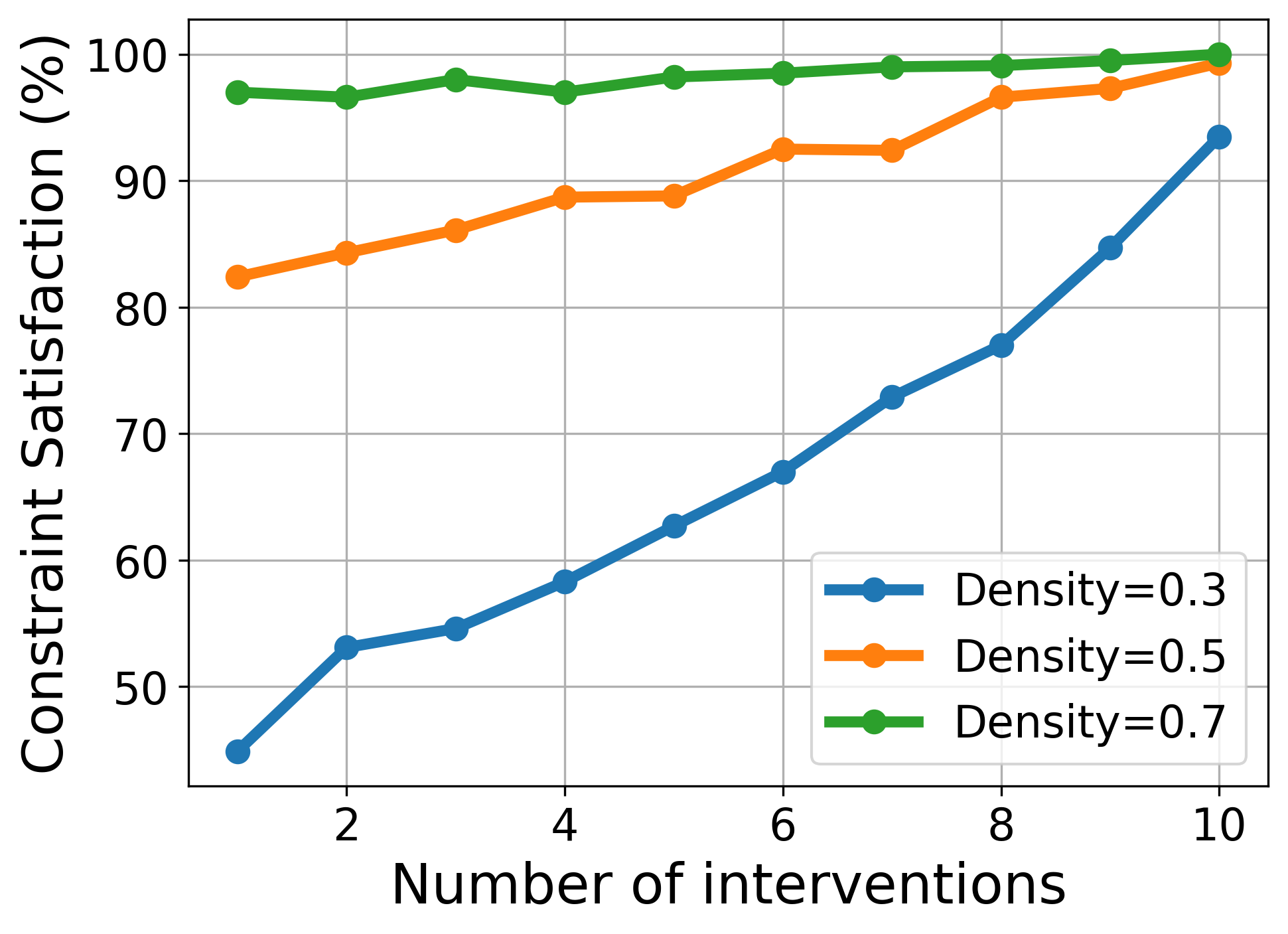}
    \caption{Constraints satisfaction}
    \label{fig:sub4}
  \end{subfigure}

  \caption{(a–c) Relative errors of estimated parameters of the OU process versus the number of interventions, and (d) percentage of instances satisfying graphical conditions in Theorem~\ref{th:multiSCC}.}
  \label{fig:all}
\end{figure*}

\section{Experiments}
\paragraph{Scope of empirical evaluation.}
The goal of this section is to validate the identifiability results (e.g.,  parameter recovery). Regarding comparisons, we focus on methods that assume linear SDE models and operate on
stationary data. In particular, \cite{rohbeck2024bicycle} considered a Lyapunov-based loss that
corresponds to our objective without the mean terms; we include this variant explicitly as an
ablation (\emph{Covariance only}). \cite{varando2020graphical} proposed an
$\ell_1$-regularized estimator based solely on observational covariances (without interventions),
and is therefore not designed for the interventional setting. \cite{lorch2024causal}
also proposed a loss for linear SDEs, but their interventions are implemented as mean shifts rather
than the hard interventions assumed in our work. We adapted their loss to our hard-intervention
setting and evaluated it on synthetic data. However, we did not observe
improvements in performance as the number of interventions increased, and hence did not
report its results.

\textbf{Synthetic Data.}
We generated synthetic datasets by simulating steady-state observations from a stable linear stochastic system. 
The drift matrix $\boldsymbol{\Lambda} \in \mathbb{R}^{n\times n}$ was initialized as a zero matrix, then each 
off-diagonal entry was set to a Gaussian random value with probability $\rho$ and left at zero otherwise, where $\rho$ 
is the desired density level. 
For each row, the diagonal entry was set to be larger than the sum of the absolute values of 
off-diagonal entries in that row by at least a positive margin, ensuring stability. 
The diffusion power $\mathbf{D}$ was 
diagonal with strictly positive entries sampled uniformly from $[d_{\min}=0.2, d_{\max}=0.4]$, and the bias vector $\mathbf{b}$ 
was drawn uniformly from $[b_{\min}=0.2, b_{\max}=1.5]$.
The observational steady-state mean and covariance were then computed by solving the corresponding linear and Lyapunov 
equations for the given $\boldsymbol{\Lambda}$, $\mathbf{b}$, and $\mathbf{D}$. Interventions were simulated by zeroing 
all off-diagonal entries in a selected row of $\boldsymbol{\Lambda}$ while keeping the diagonal entry unchanged and 
leaving $\mathbf{b}$ unchanged. For the observational setting and each intervention, samples were drawn from the corresponding multivariate normal distribution. All the implementations are available in the following link: \href{https://github.com/sabersalehk/OU_ID}{https://github.com/sabersalehk/OU\_ID}. More details of experiments are given in Appendix \ref{app:experiments}.

\begin{table*}[t]
\centering
\caption{Evaluation of DES and PDS on the three Perturb-seq datasets. The interventional mean is shaded to indicate oracle access to interventional data.}
\label{tab:realdata}
\begin{tabular}{lcccccc}
\toprule
 & \multicolumn{2}{c}{Co-Culture} & \multicolumn{2}{c}{Control} & \multicolumn{2}{c}{IFN-$\gamma$} \\
\cmidrule(lr){2-3} \cmidrule(lr){4-5} \cmidrule(lr){6-7}
Method & DES & PDS & DES & PDS & DES & PDS \\
\midrule
Observational mean       & 0.33 & 0.57 & 0.33 & 0.57 & 0.35 & 0.67 \\
Ours (``Covariance") & 0.51 & 0.43 & \textbf{0.46} & 0.43 & 0.42 & 0.43 \\
Ours (``Mean + Covariance")         & 0.43 & \textbf{0.67} & 0.33 & 0.63 & \textbf{0.47} & \textbf{0.70} \\
\rowcolor{gray!15} Interventional mean & \textbf{0.52} & 0.57 & 0.42 & \textbf{0.67} & 0.36 & \textbf{0.70} \\
\bottomrule
\end{tabular}
\end{table*}

In Figure~\ref{fig:all} (a-c), we report the estimation error of 
$\boldsymbol{\Lambda}$, $\mathbf{D}$, and $\mathbf{b}$ as a function of 
the number of interventions with $n=10$. 
To evaluate recovery up to scaling, we compute the scaling factor 
$c = \langle \hat{\mathbf{A}}, \mathbf{A} \rangle / \langle \mathbf{A}, \mathbf{A} \rangle$ for each parameter 
matrix/vector $\mathbf{A} \in \{\boldsymbol{\Lambda}, \mathbf{D}, \mathbf{b}\}$, 
where $\langle \mathbf{X}, \mathbf{Y} \rangle = \operatorname{tr}(\mathbf{X}^\top \mathbf{Y})$ for matrices and 
$\langle \mathbf{x}, \mathbf{y} \rangle = \mathbf{x}^\top \mathbf{y}$ for vectors. The relative error is then 
measured as $\|\hat{\mathbf{A}} - c\mathbf{A}\| / \|\mathbf{A}\|$. We apply this procedure separately 
to $\boldsymbol{\Lambda}$, $\mathbf{D}$, and $\mathbf{b}$. The results with 
90\% confidence intervals are given in blue curves with the legend 
``Mean and Covariance (90\% CI)''. As shown, for $\boldsymbol{\Lambda}$ and 
$\mathbf{b}$, the relative error decreases as the number of interventions increases.
For $\mathbf{D}$, we observe 
a small non-monotone effect: the error is initially low with no interventions, 
increases slightly for a few interventions, and then decreases again. 
One possible explanation is that $\mathbf{D}$ is already estimated accurately from observational data, so incorporating a small number of interventional data (which may contain fewer samples than the observational data) can worsen its estimation. As the number of interventions increases, however, the estimation of $\boldsymbol{\Lambda}$ improves, which in turn reduces the error in $\mathbf{D}$.

We also report results when the loss includes only the residuals of the 
Lyapunov equations (i.e., using covariances but not means) with the legend 
``Covariance (90\% CI),'' similar to \citep{dettling2024lasso,rohbeck2024bicycle}. 
This variant performs noticeably worse, highlighting that the mean terms are 
essential (please note that there is no curve for $\mathbf{b}$ in this case as there is no 
term for estimating it); mean estimates are typically much more 
accurate than covariance estimates, and incorporating them substantially 
improves recovery.

In Figure~\ref{fig:sub4}, we depict the percentage of instances of $\boldsymbol{\Lambda}$ satisfying the graphical conditions in Theorem~\ref{th:multiSCC} as a function of the number of interventions for different graph densities ($\rho$). For sparse graphs ($\rho=0.3$), less than 50\% of instances satisfy the conditions with a single intervention, but the percentage increases steadily as more interventions are added. In contrast, denser graphs ($\rho=0.5,0.7$) already satisfy the graphical conditions at a high rate with only a few interventions.

We also evaluated the SCC-recovery step from finite samples. Using the
reconstruction procedure implied by Theorem~\ref{th:learning_SCC_mu} and
Remark~\ref{remak:scc_recovery}, we estimated the response sets from empirical
mean shifts and recovered the SCC partition under both hard and soft
interventions. We measured recovery quality using the Adjusted Rand Index (ARI),
where \(1\) indicates exact recovery. As shown in Table~\ref{tab:scc_recovery},
SCC recovery improves steadily as the number of samples increases. Hard
interventions are more informative in this experiment, but soft interventions
also show a clear improvement with sample size. The details of experiments are given in Appendix \ref{app:scc_recovery_experiment}.

\begin{table}[h]
\centering
\caption{SCC recovery from estimated means. Recovery is measured by Adjusted Rand
Index (ARI), where \(1\) indicates exact recovery.}
\label{tab:scc_recovery}
\begin{tabular}{ccc}
\toprule
Sample size & Hard int. (ARI) & Soft int. (ARI) \\
\midrule
500   & 0.660 & 0.464 \\
1000  & 0.746 & 0.615 \\
5000  & 0.849 & 0.804 \\
10000 & 0.881 & 0.837 \\
\bottomrule
\end{tabular}
\end{table}

Our identifiability theory assumes diagonal diffusion, and this assumption
is mainly used for the parameter-identifiability result in
Theorem~\ref{th:multiSCC}. The learning objective,
however, can also be applied with a non-diagonal diffusion matrix. To evaluate
this empirically, we repeated the synthetic experiment in the same setup as
Figure~\ref{fig:all}, but allowed the diffusion matrix to be non-diagonal. The
results in Table~\ref{tab:nondiag_diffusion} show the same qualitative trend as
in the diagonal-diffusion setting: the relative errors for
\(\boldsymbol{\Lambda}\), \(\mathbf D\), and \(\mathbf b\) decrease as the number
of interventions increases. This suggests that, although the current
identifiability proof uses diagonal diffusion, the proposed moment-based
estimator remains empirically useful beyond that setting. The details of experiments are given in Appendix \ref{app:nondiag_diffusion}.

\begin{table}[h]
\centering
\caption{Parameter recovery with non-diagonal diffusion. Relative errors decrease
as the number of interventions increases.}
\label{tab:nondiag_diffusion}
\begin{tabular}{cccc}
\toprule
\# int. & Rel. err. \((\boldsymbol{\Lambda})\) & Rel. err. \((\mathbf D)\) & Rel. err. \((\mathbf b)\) \\
\midrule
0  & 0.2064 & 0.1811 & 0.1407 \\
2  & 0.1782 & 0.1582 & 0.1423 \\
4  & 0.1570 & 0.1449 & 0.1394 \\
6  & 0.1198 & 0.1188 & 0.1185 \\
8  & 0.0859 & 0.1011 & 0.0994 \\
10 & 0.0646 & 0.0719 & 0.0661 \\
\bottomrule
\end{tabular}
\end{table}

We further empirically investigate how conservative the graphical conditions in Theorem \ref{th:multiSCC} are in Appendix \ref{app:graphical conditions}, and study the impact of sample size on performance in Appendix~\ref{app:sample_size}.

\textbf{Real Data.} We assess our method on real-world data, leveraging three published single-cell perturbation screen datasets (Frangieh et al., 2021). Since the true causal graph is unknown in this setting, our evaluation focuses on generalization to unseen perturbations. Specifically, we consider a Perturb-seq dataset containing targeted CRISPR knock-out perturbations of 249 target genes in tumor-infiltrating lymphocytes (TILs) of melanoma patients. The perturbations were performed under three conditions, which we treat as separate datasets: a baseline culture of TILs in a neutral medium (``Control''), a culture of TILs with interferon-$\gamma$ added (``IFN-$\gamma$''), and a co-culture of TILs with patient-derived melanoma cells (``Co-Culture''). 

Following the setup in \citep{sethuraman2023nodags}, we restrict our analysis to the same subset of 61 genes and adopt their reported training/test split: $90\%$ of interventions are used for training and the remaining $10\%$ are held out for evaluation, with analyses performed separately for each dataset. Our goal is to predict the interventional mean $\boldsymbol{\mu}$ for unseen perturbations, using the estimated parameters and the steady-state equation for the mean. Note that global scaling is not an issue here, as it cancels out for predicting $\boldsymbol{\mu}$. 

For evaluation, we do not rely on Mean Absolute Error (MAE), as even the observational mean achieves a very close performance to the one using the interventional mean on the held-out set. Instead, following the recommendation in the Virtual Cell Challenge\footnote{\url{https://virtualcellchallenge.org/evaluation\#scoring}}, we report the Differential Expression Score (DES) and the Perturbation Differential Score (PDS), where higher values indicate better performance. DES measures agreement in identifying differentially expressed genes, while PDS measures a model's ability to distinguish between perturbations by ranking predictions according to their similarity to the true perturbational effect, regardless of their effect size. In Table \ref{tab:realdata}, we compare our full method (``Mean + Covariance") against the ``Covariance" only variant (similar to the approaches in \citep{dettling2024lasso,rohbeck2024bicycle}), and against the one using the interventional mean of the held-out set (which is not available to our method). As shown in Table~\ref{tab:realdata}, our method achieves comparable and in some cases even higher scores than the oracle baseline using the interventional mean, despite not having access to held-out interventional data.

\section{Conclusions and Future Work}

We studied recovery of multivariate OU parameters from steady-state observational and interventional data. 
Our main theoretical contribution shows that one intervention per SCC of the drift graph suffices for generic recovery of $(\boldsymbol{\Lambda},\mathbf{b},\mathbf{D})$ up to a single global scaling when the DAG over SCCs is connected with a unique root. 
The single-SCC case (Theorem~\ref{th:Single SCC}) yields a one-dimensional null space for the stacked moment equations, and the multi-SCC result (Theorem~\ref{th:multiSCC}) follows via a constructive, recursive decomposition that leverages mean shifts (Theorem~\ref{th:learning_SCC_mu}) to infer SCCs and a topological order. 
Building on these guarantees, we considered a regularized least-squares estimator and observed accurate parameter recovery in synthetic datasets. 
\\
While our theoretical results are developed for linear SDEs, we note that this setting is a canonical model class for studying causal inference in continuous-time systems where causal effects and parameter recovery can be analyzed rigorously. The identifiability results we obtain, can lead to some practical insights on when stationary data of interventions are sufficient to recover causal structure. This might be interesting for the wider ML community working on perturbation predictions.
\\
Our guarantees rely on some spectral/\textcolor{black}{rank} nondegeneracy assumptions. Numerical results suggest these hold generically for strongly connected drift graphs. A key avenue for future work is a formal genericity proof. Another direction is to understand how diffusion assumptions (e.g., not diagonal diffusion matrices) change the boundary between identifiable and non-identifiable regimes.

\bibliography{example_paper}
\bibliographystyle{icml2026}

\newpage
\appendix
\onecolumn

\section{Proofs}
\label{app:proofs}

\subsection{Forming the Linear System}
\label{app:linear_system}
\begingroup\color{black}
For completeness, write $m_t=\mathbb E[\mathbf x_t]$ and
$S_t=\operatorname{Cov}(\mathbf x_t)$. Taking expectations in the SDE
and applying It\^o's product rule to the centered second moment gives
\[
\dot m_t=-\boldsymbol\Lambda m_t+\mathbf b,\qquad
\dot S_t=-\boldsymbol\Lambda S_t-S_t\boldsymbol\Lambda^\top+\mathbf D.
\]
The last term is the Brownian quadratic-variation contribution.
At stationarity both derivatives vanish, yielding
\eqref{eq:mean_steady} and~\eqref{eq:lyapanouv}; the same calculation
with the intervened drift gives their interventional counterparts.

\endgroup
Throughout, $\mathbf{I}_{n}$ denotes the $n\times n$ identity,  
$\otimes$ is the Kronecker product,  
and $\mathbf e_{k}$ is the $k$-th canonical basis vector in $\mathbb R^{n}$.
\begingroup\color{black}
Explicitly, $A\otimes B=[A_{jk}B]$ and
$(A\odot B)_{jk}=A_{jk}B_{jk}$, where $\odot$ denotes the Hadamard
(entrywise) product. Column-wise vectorization satisfies
$\operatorname{vec}(AXB)=(B^\top\otimes A)\operatorname{vec}(X)$;
this identity produces the linear-system blocks below.
\endgroup

We define the following selector matrices:
\begin{itemize}
\item Introduce the matrix $\mathbf{P} \in \mathbb{R}^{\,n^{2}\times n}$
such that $\operatorname{vec}\!\bigl(\operatorname{diag}(\mathbf{d})\bigr)
= \mathbf{P}\,\mathbf{d},$
where the $k$-th column of $\mathbf{P}$ is given by $\mathbf{P}_{:,k} = \operatorname{vec}(\mathbf{E}_{kk}),$
and $\mathbf{E}_{kk}$ denotes the $n\times n$ matrix with a one in entry $(k,k)$ and zeros elsewhere.


\item
For a fixed row index $i$, let
$
\mathbf{E}_{i}:=\mathbf{I}_{n}\otimes\mathbf e_{i}^{\!\top}\in\mathbb R^{\,n\times n^{2}}$. Therefore, $\mathbf{E}_{i}\operatorname{vec}(\boldsymbol{\Lambda})
      =\bigl[\Lambda_{i1},\dots ,\Lambda_{in}\bigr]^{\!\top}$ extracts the entire $i$-th row of $\boldsymbol{\Lambda}$.
\item
$\mathbf{S}_{n}\in\mathbb R^{\,\frac{n(n+1)}2\times n^{2}}$ is the upper-triangular
elimination matrix.
\item Let $\mathbf{C}\in\mathbb{R}^{\,n^{2}\times n^{2}}$ denote the commutation matrix, so that
$\operatorname{vec}(\mathbf{X}^{\!\top})=\mathbf{C}\,\operatorname{vec}(\mathbf{X})$ for any $X\in\mathbb{R}^{n\times n}$.

\end{itemize}

For the equation of observational mean, we define:
\[
\mathbf{M}_{0}:=\Bigl[
        -\!\bigl(\boldsymbol{\mu}^{\!\top}\!\otimes \mathbf{I}_{n}\bigr)\;
        \Big|\;
        \mathbf{I}_{n}\;
        \Big|\;
        \mathbf{0}
       \Bigr]\in\mathbb{R}^{\,n\times p},
\]
where $p=n^2+2n$.

For the equation of the interventional mean (on coordinate $i$),
let $\mathbf{J}_{i}:=\mathbf{I}_n-\mathbf e_{i}\mathbf e_{i}^{\!\top}$. We define:
\[
\mathbf{M}_{1}:=\Bigl[
        -\!\bigl((\boldsymbol{\mu}^{(i)})^{\!\top}\!\otimes \mathbf{I}_{n}\bigr)
        +\mathbf e_{i}(\boldsymbol{\mu}^{(i)})^{\!\top}\mathbf{J}_{i}\mathbf{E}_{i}\;
        \Big|\;
        \mathbf{I}_{n}\;
        \Big|\;
        \mathbf{0}
       \Bigr]\in\mathbb{R}^{\,n\times p}.
\]

For the observational covariance, 
vectorising
$\boldsymbol{\Lambda}\boldsymbol{\Sigma}
 +\boldsymbol{\Sigma}\boldsymbol{\Lambda}^{\!\top}
 -\operatorname{diag}(\mathbf d)=0$
and selecting the upper-triangular part yields
\[
\mathbf{K}_{0}
=
\Bigl[
\mathbf{S}_{n}\bigl(
   \boldsymbol{\Sigma}\otimes \mathbf{I}_{n}
 + (\mathbf{I}_{n}\otimes \boldsymbol{\Sigma})\,\mathbf{C}
\bigr)
\ \Big|\ \mathbf{0}\ \Big|\ -\mathbf{S}_{n}\mathbf{P}
\Bigr]\in\mathbb{R}^{\,\frac{n(n+1)}2\times p}.
\]

For the interventional covariance (on coordinate $i$), we define:
\[
\mathbf{K}_{1}
=
\Bigl[
\mathbf{S}_{n}\bigl(
   \boldsymbol{\Sigma}^{(i)}\otimes \mathbf{I}_{n}
 + (\mathbf{I}_{n}\otimes \boldsymbol{\Sigma}^{(i)})\,\mathbf{C}
 - \bigl[(\mathbf{I}_{n}\otimes \mathbf{e}_{i})+(\mathbf{e}_{i}\otimes \mathbf{I}_{n})\bigr]
   \boldsymbol{\Sigma}^{(i)}\mathbf{J}_{i}\mathbf{E}_{i}
\bigr)
\ \Big|\ \mathbf{0}\ \Big|\ -\mathbf{S}_{n}\mathbf{P}
\Bigr]\in\mathbb{R}^{\,\frac{n(n+1)}2\times p}.
\]

Assembling all the equations:

\begin{equation}
\mathbf{A}:=
\begin{bmatrix}
 \mathbf{M}_{0}\\[2pt]
 \mathbf{M}_{1}\\[2pt]
 \mathbf{K}_{0}\\[2pt]
 \mathbf{K}_{1}
\end{bmatrix}
\in\mathbb{R}^{\,m\times p},
\label{eq:A}
\end{equation}
where $
m=n+n+\frac{n(n+1)}2+\frac{n(n+1)}2=n^{2}+3n$.

\subsection{Notion of Genericity}
\label{app:genericity}
Fix a directed graph $G$, and let
$\Theta_G \subset\mathbb{R}^{|E|+3n}$
denote the admissible parameter space consisting of all free parameters
$
\theta=(\boldsymbol{\Lambda},\mathbf b,\mathbf d)$
\begingroup\color{black}
compatible with $G$, with $\boldsymbol{\Lambda}$ and the drifts for
the performed interventions positive stable, and
$\mathbf D=\operatorname{diag}(\mathbf d)\succ0$.
\endgroup
Here, compatibility with $G$ means that the off-diagonal support of
$\boldsymbol{\Lambda}$ is a subset of $G$. The set $\Theta_G$ is open in
$\mathbb{R}^{|E|+3n}$. We say that a property holds \emph{generically} on
$\Theta_G$ if there exists a Lebesgue-measure-zero subset
$N_G\subset \Theta_G$
such that the property holds for every $\theta\in \Theta_G\setminus N_G$.
\begingroup\color{black}
Here $\Theta_G$ specifies the family of true parameters over which genericity is asserted. The class of competing models is specified separately in each theorem. For instance, competitors are unrestricted in graph support in Theorem \ref{th:Single SCC}, whereas Theorem \ref{th:multiSCC} explicitly requires compatibility with $G$.

\paragraph{Rational witnesses.}
We repeatedly use the fact that the zero set of a nonzero real polynomial
has Lebesgue measure zero \citep{mityagin2015zero}.
If $f=p/q$ is rational, a single parameter value with $q\ne0$ and
$f\ne0$ proves that $p$ is not the zero polynomial. We call such a
parameter choice a \emph{witness}; it establishes generic nonvanishing
where $f$ is defined. Matrix inverses and solutions of nonsingular finite
linear systems are rational by the adjugate formula.
A full-column-rank matrix has at least one nonzero maximal minor, so the
same argument gives generic full rank from one such witness.
Finitely many rational functions that are not identically zero can
be made nonzero simultaneously. Indeed, each function has a
measure-zero zero set, and the finite union of these sets still
has measure zero. Outside this union, all the functions are
nonzero at the same point. Their individual witnesses need not
coincide. Each witness only establishes that its corresponding
function is not identically zero. In particular, every nonempty
open admissible neighborhood on which the functions are defined
contains such a common point, because an open neighborhood has
positive Lebesgue measure and therefore cannot be contained in
the exceptional union.
\endgroup
\subsection{Proof of Theorem \ref{th:learning_SCC_mu}}
\label{app:learning_SCC}

\textcolor{black}{Throughout this proof, genericity is understood with respect to the admissible parameter space $\Theta_G$ for any given graph $G$, as defined in the genericity convention. The statement of this theorem only depends on $(\boldsymbol{\Lambda},\mathbf b)$, not on the diffusion parameters $\mathbf d$. Thus, any Lebesgue-measure-zero exceptional set in the free coordinates for $(\boldsymbol{\Lambda},\mathbf b)$ induces a Lebesgue-measure-zero exceptional set in the full admissible space $\Theta_G$. Moreover, by the model setup, the intervened matrix $\widetilde{\boldsymbol{\Lambda}}^{(i)}$ is invertible whenever the interventional mean $\boldsymbol{\mu}^{(i)}$ is considered.}

At steady state, $\boldsymbol{\mu}=\boldsymbol{\Lambda}^{-1}\mathbf{b}$ and $\boldsymbol{\mu}^{(i)}=(\widetilde{\boldsymbol{\Lambda}}^{(i)})^{-1}\mathbf{b}$.
Let $\mathbf{E}:=\widetilde{\boldsymbol{\Lambda}}^{(i)}-\boldsymbol{\Lambda}$; only the $i$-th row of $\mathbf{E}$ is nonzero, with $E_{ij}=-\Lambda_{ij}$ for $j\neq i$.
By the following equation,
\begin{equation}\label{eq:delta-mu}
\Delta\boldsymbol{\mu}
:= \boldsymbol{\mu}-\boldsymbol{\mu}^{(i)}
= \boldsymbol{\Lambda}^{-1}\mathbf{E}\,(\widetilde{\boldsymbol{\Lambda}}^{(i)})^{-1}\mathbf{b}
= \boldsymbol{\Lambda}^{-1}\mathbf{E}\,\boldsymbol{\mu}^{(i)}.
\end{equation}
Since only row $i$ of $\mathbf{E}$ is nonzero, $\mathbf{E}\mathbf{v} = s(\mathbf{v})\,\mathbf{e}_i$ for any vector $\mathbf{v}$, where
\[
s(\mathbf{v}) := -\sum_{j\neq i}\Lambda_{ij}\,v_j.
\]
Applying this to $\mathbf{v}=\boldsymbol{\mu}^{(i)}$ in \eqref{eq:delta-mu} gives
\begin{equation}\label{eq:key-factorization}
\Delta\boldsymbol{\mu} = s\!\big(\boldsymbol{\mu}^{(i)}\big)\,\boldsymbol{\Lambda}^{-1}\mathbf{e}_i.
\end{equation}

Permute coordinates by some permutation matrix that orders the SCCs topologically. 
Let $\mathbf{L}$ be the block lower--triangular drift matrix after topologically ordering the SCCs,
\[
\mathbf{L}=\begin{bmatrix}
\mathbf{L}^{(1)} & 0 & \cdots & 0\\
\mathbf{L}^{(2,1)} & \mathbf{L}^{(2)} & \ddots & \vdots\\
\vdots & \ddots & \ddots & 0\\
\mathbf{L}^{(K,1)} & \cdots & \mathbf{L}^{(K,K-1)} & \mathbf{L}^{(K)}
\end{bmatrix},
\]
with each diagonal block $\mathbf{L}^{(u)}$ invertible. Fix a block $r$ and a coordinate $i$ in block $r$.

\begingroup\color{black}
Let $R$ be the union of the SCCs reachable from the component containing
$i$, including that component. No allowed edge leads from $R$ to $R^c$.
Ordering $R^c$ before $R$ therefore gives a block lower--triangular drift
with zero $(R^c,R)$ block, and its inverse has the same zero block.
Consequently $(\mathbf L^{-1}\mathbf e_i)_j=0$ whenever $j\notin R$.
It remains to show that the entries for descendant coordinates are
generically nonzero.
\endgroup
\textcolor{black}{We show that if $t$ is a descendant of $r$ in the DAG over SCCs and $j$ is a coordinate in block $t$, then there exists an admissible parameter value $\theta^\star\in\Theta_G$ such that}
$\textcolor{black}{\big(\mathbf{L}(\theta^\star)^{-1}\mathbf{e}_i\big)_j\neq 0.}$
\textcolor{black}{Indeed, by the adjugate formula, this entry can be written as}
$
\textcolor{black}{
\big(\mathbf L(\theta)^{-1}\mathbf e_i\big)_j
=
\frac{p_{j,i}(\theta)}{\det \mathbf L(\theta)},
}$
\textcolor{black}{where $p_{j,i}(\theta)$ is a polynomial in the free entries of $\mathbf L(\theta)$. Since $\det \mathbf L(\theta)\neq 0$ on the admissible parameter space, showing one admissible point where the entry is nonzero shows that $p_{j,i}$ is not identically zero. Hence the zero set of the entry is contained in the zero set of a nonzero polynomial, and therefore has Lebesgue measure zero.}

\textcolor{black}{Since $j$ lies in a descendant block of the block containing $i$, there exists a coordinate-level directed path in $G$ from $i$ to $j$, say}
$
\textcolor{black}{i=v_0\to v_1\to\cdots\to v_\ell=j.}$
\textcolor{black}{If necessary, remove cycles so that the path is simple. We construct a one-parameter family $\mathbf{L}(\varepsilon)$ as follows. Choose $\alpha>0$ sufficiently large. Set all diagonal entries of $\mathbf{L}(\varepsilon)$ equal to $\alpha$. For each edge $v_{q-1}\to v_q$ on the selected path, set the corresponding entry $L_{v_q,v_{q-1}}(\varepsilon)$ equal to $1$. Set every other allowed off-diagonal entry of $\mathbf{L}(\varepsilon)$ equal to $\varepsilon$, and keep all forbidden entries equal to zero. For every $\varepsilon>0$, every allowed off-diagonal entry is nonzero, so $\mathbf{L}(\varepsilon)$ has the exact support prescribed by $G$. Moreover, by choosing $\alpha$ large and $\varepsilon$ sufficiently small, $\mathbf{L}(\varepsilon)$ is strictly row-diagonally dominant with positive diagonal entries, hence positive stable and in particular invertible. Thus, after choosing arbitrary admissible $\mathbf b$ and any $\mathbf d\succ0$, this construction gives a point $\theta(\varepsilon)\in\Theta_G$ for all sufficiently small $\varepsilon>0$.}

\textcolor{black}{Let $\mathbf{P}_{\mathrm{path}}$ be the matrix containing only the selected path entries, i.e., $(\mathbf{P}_{\mathrm{path}})_{v_q,v_{q-1}}=1$ for $q=1,\dots,\ell$ and all other entries zero. At $\varepsilon=0$,}
$\textcolor{black}{\mathbf{L}(0)=\alpha\mathbf{I}+\mathbf{P}_{\mathrm{path}}.}$
\textcolor{black}{Since the selected path is simple, $\mathbf{P}_{\mathrm{path}}$ is nilpotent,\footnote{A square matrix $\mathbf A$ is nilpotent if there exists an integer $m\ge 1$ such that $\mathbf A^m=\mathbf 0$. In this proof, $\mathbf P_{\mathrm{path}}$ is nilpotent because it only moves forward along the finite simple path $i=v_0\to v_1\to\cdots\to v_\ell=j$, so $\mathbf P_{\mathrm{path}}^{\ell+1}=\mathbf 0$.} and}
\[
\textcolor{black}{
(\alpha\mathbf{I}+\mathbf{P}_{\mathrm{path}})^{-1}
=
\alpha^{-1}\sum_{q\ge 0}(-\alpha^{-1}\mathbf{P}_{\mathrm{path}})^q,
}
\]
\textcolor{black}{where the sum is finite. Because $\mathbf{P}_{\mathrm{path}}^\ell \mathbf e_i=\mathbf e_j$, while no lower power maps $\mathbf e_i$ to $\mathbf e_j$, we get}
\[
\textcolor{black}{
\big(\mathbf{L}(0)^{-1}\mathbf e_i\big)_j
=
(-1)^\ell \alpha^{-(\ell+1)}
\neq 0.
}
\]
\textcolor{black}{By continuity, $\big(\mathbf{L}(\varepsilon)^{-1}\mathbf e_i\big)_j\neq 0$ for all sufficiently small $\varepsilon>0$. Hence we have exhibited admissible parameters with $(\mathbf{L}^{-1}\mathbf e_i)_j\neq 0$.}

\textcolor{black}{It follows that, for each fixed descendant coordinate $j\in\mathrm{Desc}(C)$, the exceptional set on which $(\mathbf{L}^{-1}\mathbf e_i)_j=0$ has Lebesgue measure zero in $\Theta_G$. Since there are only finitely many coordinates, taking the union over all descendant coordinates still gives a Lebesgue-measure-zero exceptional set. Therefore, generically,}
\[
\textcolor{black}{
\big(\mathbf{L}^{-1}\mathbf e_i\big)_j\neq 0
\quad\text{for every }j\in \mathrm{Desc}(C),
}
\]
\begingroup\color{black}
\textcolor{black}{where $\mathrm{Desc}(C)$ includes the component $C$ itself. On the other hand, as shown above, for $j\notin \mathrm{Desc}(C)$, the descendant-block argument gives $\big(\mathbf{L}^{-1}\mathbf e_i\big)_j=0$ deterministically.}
\endgroup

If for some $j\neq i$, we have $\Lambda_{ij}\neq 0$, then
$
s\!\big(\boldsymbol{\mu}^{(i)}\big) = -\sum_{j\neq i}\Lambda_{ij}\,\mu^{(i)}_j.$
\textcolor{black}{This is an analytic function of $(\boldsymbol{\Lambda},\mathbf b)$. Moreover, if the $i$-th row is non-null, then the row vector}
$
\textcolor{black}{
\mathbf r_i^\top:= -\sum_{j\neq i}\Lambda_{ij}\mathbf e_j^\top
}$
\textcolor{black}{is nonzero. Since}
$
\textcolor{black}{
s\!\big(\boldsymbol{\mu}^{(i)}\big)
=
\mathbf r_i^\top(\widetilde{\boldsymbol{\Lambda}}^{(i)})^{-1}\mathbf b,
}$
\textcolor{black}{and $(\widetilde{\boldsymbol{\Lambda}}^{(i)})^{-1}$ is invertible, the row vector
$\mathbf r_i^\top(\widetilde{\boldsymbol{\Lambda}}^{(i)})^{-1}$ is nonzero. Thus, for fixed $\boldsymbol{\Lambda}$, $s(\boldsymbol{\mu}^{(i)})$ is a nonzero linear function of $\mathbf b$. Hence it is not identically zero as a function of $(\boldsymbol{\Lambda},\mathbf b)$, and its zero set has Lebesgue measure zero in the admissible free-coordinate space. Therefore $s(\boldsymbol{\mu}^{(i)})\neq 0$ generically.}

Combining with \eqref{eq:key-factorization}, the support of $\Delta\boldsymbol{\mu}$ matches that of $\boldsymbol{\Lambda}^{-1}\mathbf{e}_i$ whenever $s(\boldsymbol{\mu}^{(i)})\neq 0$.
Thus, if for some $j\neq i$,  $\Lambda_{ij}\neq 0$, generically,
\[
\textcolor{black}{
\mu^{(i)}_k\neq \mu_k \ \Longleftrightarrow\ k\in \mathrm{Desc}(C),
}
\]
\textcolor{black}{where $\mathrm{Desc}(C)$ includes $C$ itself.}

\textcolor{black}{If $\Lambda_{ij}=0$ for all $j\neq i$, then $\mathbf E=\mathbf 0$, so $s(\boldsymbol{\mu}^{(i)})=0$ and $\Delta\boldsymbol{\mu}=\mathbf 0$. Hence $\boldsymbol{\mu}^{(i)}=\boldsymbol{\mu}$.}

\subsection{Proof of Theorem \ref{th:learning_SCC_mu} under a soft row intervention}
\label{app:learning_SCC_soft}

We only indicate the changes relative to the proof of
Theorem~\ref{th:learning_SCC_mu} for the hard intervention. The argument showing
the generic support of $\boldsymbol{\Lambda}^{-1}\mathbf e_i$ is unchanged.

At steady state, $
\boldsymbol{\mu}=\boldsymbol{\Lambda}^{-1}\mathbf b,
\boldsymbol{\mu}^{(i)}
=
(\widetilde{\boldsymbol{\Lambda}}^{(i)})^{-1}\mathbf b$.
For a soft row intervention, $\widetilde{\boldsymbol{\Lambda}}^{(i)}$ differs from
$\boldsymbol{\Lambda}$ only in row $i$. The entries in row $i$, including the
diagonal entry, may change. Let
$\mathbf E
:=
\widetilde{\boldsymbol{\Lambda}}^{(i)}-\boldsymbol{\Lambda}$.
Then only row $i$ of $\mathbf E$ is nonzero. Denote this row change by $\boldsymbol{\delta}_i^\top
:=
\mathbf e_i^\top\mathbf E
=
\mathbf e_i^\top
\big(
\widetilde{\boldsymbol{\Lambda}}^{(i)}
-
\boldsymbol{\Lambda}
\big)$.
We assume the intervention is non-null, i.e., $\boldsymbol{\delta}_i\neq 0$.
\begingroup\color{black}
The row change is fixed independently of $\mathbf b$.
\endgroup

As before,
\begin{equation}
\label{eq:delta-mu-soft-short}
\Delta\boldsymbol{\mu}
:=
\boldsymbol{\mu}-\boldsymbol{\mu}^{(i)}
=
\boldsymbol{\Lambda}^{-1}
\mathbf E
(\widetilde{\boldsymbol{\Lambda}}^{(i)})^{-1}\mathbf b
=
\boldsymbol{\Lambda}^{-1}
\mathbf E
\boldsymbol{\mu}^{(i)}.
\end{equation}
Since only row $i$ of $\mathbf E$ is nonzero, for any vector $\mathbf v$, $\mathbf E\mathbf v
=
\eta(\mathbf v)\mathbf e_i$,
where
\[
\eta(\mathbf v)
:=
\boldsymbol{\delta}_i^\top\mathbf v
=
\sum_{j=1}^n
\left(
\widetilde{\Lambda}^{(i)}_{ij}
-
\Lambda_{ij}
\right)v_j .
\]
Applying this to $\mathbf v=\boldsymbol{\mu}^{(i)}$ in
\eqref{eq:delta-mu-soft-short} gives
\begin{equation}
\label{eq:key-factorization-soft-short}
\Delta\boldsymbol{\mu}
=
\eta(\boldsymbol{\mu}^{(i)})
\boldsymbol{\Lambda}^{-1}\mathbf e_i .
\end{equation}

The rest of the proof uses the same generic-support claim proved in the hard
intervention case: for the SCC $C$ containing $i$,
\[
(\boldsymbol{\Lambda}^{-1}\mathbf e_i)_k=0
\quad\text{deterministically for }k\notin\mathrm{Desc}(C),
\]
and
\[
(\boldsymbol{\Lambda}^{-1}\mathbf e_i)_k\neq0
\quad\text{generically for every }k\in\mathrm{Desc}(C),
\]
where $\mathrm{Desc}(C)$ includes $C$ itself.

It remains only to check that the scalar factor in
\eqref{eq:key-factorization-soft-short} is generically nonzero. We have
$\eta(\boldsymbol{\mu}^{(i)})
=
\boldsymbol{\delta}_i^\top
(\widetilde{\boldsymbol{\Lambda}}^{(i)})^{-1}\mathbf b$.
Since $\boldsymbol{\delta}_i\neq0$ and
$\widetilde{\boldsymbol{\Lambda}}^{(i)}$ is invertible, the row vector
$
\boldsymbol{\delta}_i^\top
(\widetilde{\boldsymbol{\Lambda}}^{(i)})^{-1}$
is nonzero. Hence, for fixed
$(\boldsymbol{\Lambda},\widetilde{\boldsymbol{\Lambda}}^{(i)})$,
$\eta(\boldsymbol{\mu}^{(i)})$ is a nonzero linear function of $\mathbf b$.
Therefore its zero set is a Lebesgue-measure-zero exceptional set, and
$\eta(\boldsymbol{\mu}^{(i)})\neq0$
\begingroup\color{black}
generically. The exceptional set is measurable, and its section in
$\mathbf b$ at every fixed admissible drift/perturbation pair is a
hyperplane of measure zero. Fubini-Tonelli theorem \citep{folland2009guide} (Theorem 2.16 (b)), applied to the indicator
of this set, therefore makes it jointly null in the drift, perturbation,
and input parameters.
\endgroup

Combining this with \eqref{eq:key-factorization-soft-short}, the support of
$\Delta\boldsymbol{\mu}$ generically agrees with the support of
$\boldsymbol{\Lambda}^{-1}\mathbf e_i$. Therefore, for a non-null soft row
intervention on coordinate $i$,
\[
\mu_k^{(i)}\neq \mu_k
\quad\Longleftrightarrow\quad
k\in\mathrm{Desc}(C)
\]
generically, where $\mathrm{Desc}(C)$ includes $C$ itself.

If the soft intervention is null, i.e., $\boldsymbol{\delta}_i=\mathbf0$, then
$\mathbf E=\mathbf0$, so $\Delta\boldsymbol{\mu}=\mathbf0$ and
$\boldsymbol{\mu}^{(i)}=\boldsymbol{\mu}$.

\subsection{Recovering SCCs and a topological order over SCCs from mean changes}
\label{app:rec_SCC}

Assume at least one intervention is performed in every SCC. For each intervention on $i$, set
\[
R_i:=\{\,j:\ \mu^{(i)}_j\neq \mu_j\,\}.
\]
By Theorem~\ref{th:learning_SCC_mu}, generically,
\[
R_i=\mathrm{Desc}(C_i)\quad\text{if the intervention on $i$ is non-null},\qquad
R_i=\emptyset\quad\text{if it is null},
\]
where $C_i$ is the SCC containing $i$ and $\mathrm{Desc}(C)$ is the set of nodes lying in SCCs reachable from $C$ in the SCC--DAG, including $C$ itself.

\medskip
\noindent\textbf{Procedure.}
\begin{enumerate}
\item \emph{Singleton sources (null interventions).} If $R_i=\emptyset$, declare $\{i\}$ a singleton SCC with no incoming edges, i.e., a source SCC. Do \emph{not} merge different $i$ with $R_i=\emptyset$.

\item \emph{Group non-null response sets.} For the remaining interventions, group by equality of response sets: $i\sim i'$ iff $R_i=R_{i'}$. Each distinct nonempty response set corresponds to exactly one non-null SCC, but the equivalence class of intervention targets is not necessarily the whole SCC. We therefore use the distinct nonempty response sets as temporary SCC labels.

\item \emph{Recover the node set of each non-null SCC.} Let $\mathcal R$ be the collection of distinct nonempty response sets obtained in step~2. For each $R\in\mathcal R$, define
$
C(R)
:=
R\setminus \bigcup_{\substack{R'\in\mathcal R\\ R'\subsetneq R}} R'$.
Then $C(R)$ is exactly the SCC whose descendant set is $R$. 

\item \emph{Reachability order among non-null SCCs.} For two recovered non-null SCCs $C(R)$ and $C(R')$,
$C(R)\text{ reaches }C(R')
\Longleftrightarrow
R\supseteq R'$.
Thus strict reverse inclusion of response sets recovers the SCC-level reachability partial order. If desired, one may draw the cover relation by adding an edge from $C(R)$ to $C(R')$ when $R\supsetneq R'$ and there is no $R''\in\mathcal R$ such that $
R\supsetneq R''\supsetneq R'$.
\begingroup\color{black}
Among these non-null SCCs, reachability is a partial order. The full condensation DAG
may also have direct edges whose reachability is already implied by a
longer path. The procedure recovers the SCCs and a compatible topological
order, without identifying every direct edge between them.
\endgroup

\item \emph{Topological order.} Output all singleton sources from step~1 first, in any arbitrary order, and then output the recovered non-null SCCs in any order consistent with reverse inclusion of their response sets, i.e., larger response sets before smaller response sets. This yields a valid topological ordering of the SCCs.
\end{enumerate}

\subsection{Proof of Theorem \ref{th:Single SCC}}
\label{app:singleSCC}

\textcolor{black}{Throughout this proof, genericity is understood with respect to
the admissible parameter space $\Theta_G$ for any fixed graph $G$, as defined in
the genericity convention in Appendix~\ref{app:genericity}.}

Consider the following two admissible triples satisfying the linear system
\eqref{eq:A}:
\begin{itemize}
  \item True underlying triple:
  $(\boldsymbol{\Lambda}_\star,\, \mathbf{b}_\star,\, \mathbf{D}_\star)$.
  \item Alternative admissible triple:
  $(\boldsymbol{\Lambda}',\, \mathbf{b}',\, \mathbf{D}')$.
\end{itemize}

Define
\[
\Delta \boldsymbol{\Theta} := 
\begin{bmatrix}
\operatorname{vec}(\Delta \boldsymbol{\Lambda}) \\
\Delta \mathbf{b} \\
\Delta \mathbf{d}
\end{bmatrix}, \qquad
\Delta \boldsymbol{\Lambda} := \boldsymbol{\Lambda}' - \boldsymbol{\Lambda}_\star,
\quad
\Delta \mathbf{b} := \mathbf{b}' - \mathbf{b}_\star,
\quad
\Delta \mathbf{D} := \mathbf{D}' - \mathbf{D}_\star,
\]
where $\Delta \mathbf d$ is the diagonal of $\Delta\mathbf D$. Since
$\mathbf A\boldsymbol{\Theta}_\star=0$ and
$\mathbf A\boldsymbol{\Theta}'=0$, we have
$\mathbf A\Delta\boldsymbol{\Theta}=0$. \textcolor{black}{Our goal is to show that
the alternative triple is a scalar multiple of the true triple, i.e.,}
$\textcolor{black}{
\boldsymbol{\Lambda}'=c\boldsymbol{\Lambda}_\star,
\mathbf b'=c\mathbf b_\star,
\mathbf D'=c\mathbf D_\star
}$
\textcolor{black}{for some scalar $c>0$.}

Fix the intervened row index $i$. The observational blocks $\mathbf M_0$ and
$\mathbf K_0$ yield
\begin{equation}
\Delta \boldsymbol{\Lambda} \boldsymbol{\mu} = \Delta \mathbf b,
\qquad
\Delta \boldsymbol{\Lambda}\boldsymbol{\Sigma}
+
\boldsymbol{\Sigma}\Delta \boldsymbol{\Lambda}^\top
=
\Delta \mathbf D,
\label{eq:obs-diff}
\end{equation}
where $\boldsymbol{\mu}$ and $\boldsymbol{\Sigma}$ are the true observational
mean and covariance. Since the intervention zeros out row $i$ off-diagonals,
\[
\widetilde{\boldsymbol{\Lambda}}_\star^{(i)}
=
\mathbf J_i\odot\boldsymbol{\Lambda}_\star,
\qquad
\widetilde{\boldsymbol{\Lambda}}'^{(i)}
=
\mathbf J_i\odot\boldsymbol{\Lambda}',
\qquad
\mathbf J_i:=\mathbf 1\mathbf 1^\top-\mathbf e_i\mathbf 1^\top+\mathbf e_i\mathbf e_i^\top .
\]
Thus the interventional blocks $\mathbf M_1$ and $\mathbf K_1$ give
\begin{equation}
(\mathbf J_i\odot\Delta \boldsymbol{\Lambda})\boldsymbol{\mu}^{(i)}
=
\Delta \mathbf b,
\qquad
(\mathbf J_i\odot\Delta \boldsymbol{\Lambda})\boldsymbol{\Sigma}^{(i)}
+
\boldsymbol{\Sigma}^{(i)}
(\Delta \boldsymbol{\Lambda}^{\top}\odot\mathbf J_i^\top)
=
\Delta \mathbf D .
\label{eq:int-diff}
\end{equation}
Taking the $i$-th row of the first equation gives
\begin{equation}
(\Delta\boldsymbol{\Lambda})_{ii}\mu_i^{(i)}
=
\Delta b_i .
\label{eq:mu_int_row_i}
\end{equation}

\textcolor{black}{Outside the measure-zero exceptional set where $b_{\star i}=0$,
define}
$
\textcolor{black}{
c:=\frac{b_i'}{b_{\star i}}.
}$
\textcolor{black}{Define the scaled-difference variables}
\[
\textcolor{black}{
\boldsymbol{\Lambda}_\Delta
:=
\boldsymbol{\Lambda}'-c\boldsymbol{\Lambda}_\star,
\qquad
\mathbf b_\Delta
:=
\mathbf b'-c\mathbf b_\star,
\qquad
\mathbf D_\Delta
:=
\mathbf D'-c\mathbf D_\star .
}
\]
\textcolor{black}{Then $(\boldsymbol{\Lambda}_\Delta,\mathbf b_\Delta,\mathbf D_\Delta)$
still satisfies the homogeneous linear system, and $(b_\Delta)_i=0$. In what
follows, all moments
$\boldsymbol{\mu},\boldsymbol{\Sigma},\boldsymbol{\mu}^{(i)},\boldsymbol{\Sigma}^{(i)}$
remain the true moments generated by
$(\boldsymbol{\Lambda}_\star,\mathbf b_\star,\mathbf D_\star)$.}

For the true intervened system, the $i$-th row gives
$
\textcolor{black}{
\Lambda_{\star,ii}\mu_i^{(i)}=b_{\star i}.
}$
Since $b_{\star i}\neq0$, we have $\mu_i^{(i)}\neq0$. Applying
\eqref{eq:mu_int_row_i} to the ambiguity triple and using $(b_\Delta)_i=0$, we get
$
\textcolor{black}{
(\boldsymbol{\Lambda}_\Delta)_{ii}=0.
}$
Moreover, the $(i,i)$ entry of the interventional covariance equation gives
$2\Sigma^{(i)}_{ii}(\boldsymbol{\Lambda}_\Delta)_{ii}
=(\mathbf D_\Delta)_{ii}$, and hence
$
\textcolor{black}{
(\mathbf D_\Delta)_{ii}=0.
}$

For the ambiguity triple, write
\[
\widetilde{\boldsymbol{\Lambda}}_\Delta^{(i)}
=
\mathbf J_i\odot\boldsymbol{\Lambda}_\Delta
=
\begin{bmatrix}
0 & 0\\
(\boldsymbol{\Lambda}_\Delta)_{-i,i} &
(\boldsymbol{\Lambda}_\Delta)_{-i,-i}
\end{bmatrix}.
\]
Using the $(-i,i)$ block of the interventional covariance equation and the fact
that $\mathbf D_\Delta$ is diagonal, we obtain
\[
(\boldsymbol{\Lambda}_\Delta)_{-i,i}\boldsymbol{\Sigma}^{(i)}_{ii}
+
(\boldsymbol{\Lambda}_\Delta)_{-i,-i}\boldsymbol{\Sigma}^{(i)}_{-i,i}
=0.
\]
Thus
\begin{equation}
\label{eq:lambda-minus-i}
\textcolor{black}{
(\boldsymbol{\Lambda}_\Delta)_{-i,i}
=
-\frac{
(\boldsymbol{\Lambda}_\Delta)_{-i,-i}\boldsymbol{\Sigma}^{(i)}_{-i,i}
}{
\boldsymbol{\Sigma}^{(i)}_{ii}
}.
}
\end{equation}
The $(-i,-i)$ block then gives
\begin{equation}
\label{eq:lyap:Z}
\textcolor{black}{
(\boldsymbol{\Lambda}_\Delta)_{-i,-i}\mathbf Z
+
\mathbf Z(\boldsymbol{\Lambda}_\Delta)_{-i,-i}^{\!\top}
=
(\mathbf D_\Delta)_{-i,-i},
}
\end{equation}
where
\[
\mathbf Z
:=
\boldsymbol{\Sigma}^{(i)}_{-i,-i}
-
\frac{
\boldsymbol{\Sigma}^{(i)}_{-i,i}\boldsymbol{\Sigma}^{(i)}_{i,-i}
}{
\boldsymbol{\Sigma}^{(i)}_{ii}
}.
\]
The matrix $\mathbf Z$ is positive definite, since it is the Schur complement of
$\boldsymbol{\Sigma}^{(i)}_{ii}$ in $\boldsymbol{\Sigma}^{(i)}\succ0$.

Next, the observational and interventional covariance equations for the ambiguity
variables are
\begin{align}
\boldsymbol{\Lambda}_\Delta\boldsymbol{\Sigma}
+
\boldsymbol{\Sigma}\boldsymbol{\Lambda}_\Delta^\top
&=
\mathbf D_\Delta,
\label{eq:diff-lyap-obs}\\
\widetilde{\boldsymbol{\Lambda}}_\Delta^{(i)}\boldsymbol{\Sigma}^{(i)}
+
\boldsymbol{\Sigma}^{(i)}
\widetilde{\boldsymbol{\Lambda}}_\Delta^{(i)\top}
&=
\mathbf D_\Delta.
\label{eq:diff-lyap-int}
\end{align}
Define $\mathbf w_\Delta\in\mathbb R^n$ by
$(w_\Delta)_i=0$ and
$(\mathbf w_\Delta)_{-i}
=
(\boldsymbol{\Lambda}_\Delta)_{i,-i}^{\top}$, so that
$
\widetilde{\boldsymbol{\Lambda}}_\Delta^{(i)}
=
\boldsymbol{\Lambda}_\Delta-\mathbf e_i\mathbf w_\Delta^\top .$
Subtracting \eqref{eq:diff-lyap-int} from \eqref{eq:diff-lyap-obs}, with
$\boldsymbol{\Gamma}:=\boldsymbol{\Sigma}-\boldsymbol{\Sigma}^{(i)}$, gives
\begin{equation}
\label{eq:ambiguity-Gamma}
\textcolor{black}{
\boldsymbol{\Lambda}_\Delta\boldsymbol{\Gamma}
+
\boldsymbol{\Gamma}\boldsymbol{\Lambda}_\Delta^\top
+
\mathbf e_i\mathbf w_\Delta^\top\boldsymbol{\Sigma}^{(i)}
+
\boldsymbol{\Sigma}^{(i)}\mathbf w_\Delta\mathbf e_i^\top
=
0.
}
\end{equation}
Substituting
$\boldsymbol{\Lambda}_\Delta
=
\widetilde{\boldsymbol{\Lambda}}_\Delta^{(i)}
+
\mathbf e_i\mathbf w_\Delta^\top$
into \eqref{eq:ambiguity-Gamma}, and then taking the $(-i,-i)$ block, removes
all terms containing $\mathbf e_i$ and gives
\begin{equation}
\label{eq:Gamma-i,-i}
\left(
\widetilde{\boldsymbol{\Lambda}}_\Delta^{(i)}\boldsymbol{\Gamma}
+
\boldsymbol{\Gamma}\widetilde{\boldsymbol{\Lambda}}_\Delta^{(i)\top}
\right)_{-i,-i}
=0.
\end{equation}
Expanding this block,
\begin{equation}
\label{eq:gamma-block}
\textcolor{black}{
(\boldsymbol{\Lambda}_\Delta)_{-i,i}\boldsymbol{\Gamma}_{i,-i}
+
(\boldsymbol{\Lambda}_\Delta)_{-i,-i}\boldsymbol{\Gamma}_{-i,-i}
+
\boldsymbol{\Gamma}_{-i,i}(\boldsymbol{\Lambda}_\Delta)_{-i,i}^{\top}
+
\boldsymbol{\Gamma}_{-i,-i}(\boldsymbol{\Lambda}_\Delta)_{-i,-i}^{\top}
=
0.
}
\end{equation}
Using \eqref{eq:lambda-minus-i}, we obtain
\begin{equation}
\label{eq:Xi-eq}
\textcolor{black}{
(\boldsymbol{\Lambda}_\Delta)_{-i,-i}\boldsymbol{\Xi}
+
\boldsymbol{\Xi}^{\top}(\boldsymbol{\Lambda}_\Delta)_{-i,-i}^{\top}
=
0,
}
\end{equation}
where
\begin{equation}
\label{eq:Xi}
\boldsymbol{\Xi}
:=
\boldsymbol{\Gamma}_{-i,-i}
-
\frac{
\boldsymbol{\Sigma}^{(i)}_{-i,i}\boldsymbol{\Gamma}_{i,-i}
}{
\boldsymbol{\Sigma}^{(i)}_{ii}
}.
\end{equation}

\begingroup\color{black}
This Schur-type expression need not be symmetric or a covariance
matrix. By Lemma~\ref{lemma:Gamma_inv}, $\boldsymbol{\Xi}$ is invertible. Define
\endgroup
\[
\textcolor{black}{
\mathbf A:=\boldsymbol{\Xi}^{-1}\mathbf Z,
\qquad
\mathbf S:=(\boldsymbol{\Lambda}_\Delta)_{-i,-i}\boldsymbol{\Xi}.
}
\]
Then \eqref{eq:Xi-eq} implies $\mathbf S^\top=-\mathbf S$. Moreover,
\[
\textcolor{black}{
\mathbf S\mathbf A-\mathbf A^\top\mathbf S
=
(\boldsymbol{\Lambda}_\Delta)_{-i,-i}\mathbf Z
+
\mathbf Z(\boldsymbol{\Lambda}_\Delta)_{-i,-i}^{\top}
=
(\mathbf D_\Delta)_{-i,-i},
}
\]
where the last equality follows from \eqref{eq:lyap:Z}. Since
$(\mathbf D_\Delta)_{-i,-i}$ is diagonal,
\[
\textcolor{black}{
\operatorname{offdiag}(\mathbf S\mathbf A-\mathbf A^\top\mathbf S)=0.
}
\]

\begin{assumption}
\label{ass:moment-nondeg}
\begingroup\color{black}
Let $\mathbf A:=\boldsymbol{\Xi}^{-1}\mathbf Z$ and let $m=n-1$ be
the reduced block dimension (or $m=|T|-1$ for a target SCC $T$).
Here $m$ is local to this assumption. Consider the linear map
\endgroup
$\mathcal L_{\mathbf A}:\mathcal K_m\to\mathbb R^{m\times m}_{\mathrm{off}}$,
\begingroup\color{black}
where the codomain consists of matrices with zero diagonal and
\endgroup
$\mathcal K_m:=\{\mathbf S\in\mathbb R^{m\times m}:\mathbf S^\top=-\mathbf S\}$,
defined by
\[
\mathcal L_{\mathbf A}(\mathbf S)
=
\operatorname{offdiag}\!\big(\mathbf S\mathbf A-\mathbf A^\top\mathbf S\big).
\]
Let $\mathbf L_{\mathbf A}$ be the matrix representation of
\begingroup\color{black}
$\mathcal L_{\mathbf A}$ in any basis of $\mathcal K_m$.
Write $\sigma(M)$ for the spectrum (set of complex eigenvalues) of $M$ (please note that
this differs from the diffusion matrix $\boldsymbol\sigma$). We assume
\endgroup
$
0\notin\sigma(\mathbf L_{\mathbf A}^{\top}\mathbf L_{\mathbf A}).$
Equivalently, the only skew-symmetric matrix $\mathbf S$ satisfying
$\operatorname{offdiag}\!\big(\mathbf S\mathbf A-\mathbf A^\top\mathbf S\big)=0
$
is $\mathbf S=0$.
\end{assumption}

\begingroup\color{black}
This is an injectivity condition: the kernel of the map, namely
the set of inputs mapped to zero, is $\{0\}$. Equivalently,
$\mathbf L_{\mathbf A}$ has full column rank and
$\mathbf L_{\mathbf A}^\top\mathbf L_{\mathbf A}$ is positive definite.
It therefore rules out a nonzero skew-symmetric ambiguity in the moments.
\endgroup
By Assumption~\ref{ass:moment-nondeg}, $\mathbf S=0$. Since
$\mathbf S=(\boldsymbol{\Lambda}_\Delta)_{-i,-i}\boldsymbol{\Xi}$ and
$\boldsymbol{\Xi}$ is invertible,
$
\textcolor{black}{
(\boldsymbol{\Lambda}_\Delta)_{-i,-i}=0.
}$
Using \eqref{eq:lambda-minus-i}, we also obtain
$
\textcolor{black}{
(\boldsymbol{\Lambda}_\Delta)_{-i,i}=0.
}$
Together with $(\boldsymbol{\Lambda}_\Delta)_{ii}=0$, this shows that the only
possibly nonzero entries of $\boldsymbol{\Lambda}_\Delta$ are the off-diagonal
entries in row $i$. Moreover, from \eqref{eq:lyap:Z}, we get
$(\mathbf D_\Delta)_{-i,-i}=0$. Since we already showed
$(\mathbf D_\Delta)_{ii}=0$, it follows that
$
\textcolor{black}{
\mathbf D_\Delta=0.
}$

Now the observational covariance equation gives
$\textcolor{black}{
\boldsymbol{\Lambda}_\Delta\boldsymbol{\Sigma}
+
\boldsymbol{\Sigma}\boldsymbol{\Lambda}_\Delta^\top
=
0.
}$
Because only row $i$ of $\boldsymbol{\Lambda}_\Delta$ can be nonzero, write
$
\textcolor{black}{
\boldsymbol{\Lambda}_\Delta=\mathbf e_i\mathbf u^\top, u_i=0.
}
$
Then
$
\textcolor{black}{
\mathbf e_i\mathbf u^\top\boldsymbol{\Sigma}
+
\boldsymbol{\Sigma}\mathbf u\mathbf e_i^\top
=
0.
}$
Taking the $i$-th row gives
$
\textcolor{black}{
\mathbf u^\top\boldsymbol{\Sigma}
+
(\boldsymbol{\Sigma}\mathbf u)_i\mathbf e_i^\top
=
0.
}$
For every $j\neq i$, this implies
$(\mathbf u^\top\boldsymbol{\Sigma})_j=0$, while the $(i,i)$ entry gives
$2(\mathbf u^\top\boldsymbol{\Sigma})_i=0$. Therefore
$\mathbf u^\top\boldsymbol{\Sigma}=0$. Since $\boldsymbol{\Sigma}\succ0$ is
invertible, $\mathbf u=0$, and hence
$\textcolor{black}{
\boldsymbol{\Lambda}_\Delta=0.
}$
Finally, the observational mean equation gives
$\boldsymbol{\Lambda}_\Delta\boldsymbol{\mu}=\mathbf b_\Delta$, so
$\mathbf b_\Delta=0$. Hence the scaled-difference triple is zero:
$
\textcolor{black}{
\boldsymbol{\Lambda}'=c\boldsymbol{\Lambda}_\star,
\mathbf b'=c\mathbf b_\star,
\mathbf D'=c\mathbf D_\star.
}$
Since $\mathbf D'$ and $\mathbf D_\star$ are positive diagonal matrices, we have
$c>0$. This completes the proof.

\begin{assumption}
\label{assum:lambda_W}
\textcolor{black}{Consider right/left eigenbases of the true drift
$\boldsymbol{\Lambda}_\star$:}
$\textcolor{black}{
\boldsymbol{\Lambda}_\star\mathbf r_{\star,\ell}
=
\lambda_{\star,\ell}\mathbf r_{\star,\ell},
\boldsymbol{\rho}_{\star,k}^{\!\top}\boldsymbol{\Lambda}_\star
=
\lambda_{\star,k}\boldsymbol{\rho}_{\star,k}^{\!\top},
\boldsymbol{\rho}_{\star,k}^{\!\top}\mathbf r_{\star,\ell}
=
\delta_{k\ell}.
}$
\textcolor{black}{Let}
$
\textcolor{black}{
\mathbf R_\star:=[\mathbf r_{\star,1}\,\cdots\,\mathbf r_{\star,n}],
\mathbf P_\star:=[\boldsymbol{\rho}_{\star,1}\,\cdots\,\boldsymbol{\rho}_{\star,n}],
}$
\begingroup\color{black}
\textcolor{black}{so that $\mathbf P_\star^{\!\top}\mathbf R_\star=\mathbf I$.
Since the eigenvalues are distinct, $\mathbf R_\star$ is invertible.
We choose the rows of $\mathbf R_\star^{-1}$ as the left eigenvectors.
Here $\top$ denotes ordinary transpose, without complex conjugation.
Let $\mathbf w_\star\in\mathbb R^n$ be the true intervened-row vector, defined by
$(w_\star)_i=0$ and $(\mathbf w_\star)_{-i}
=(\boldsymbol{\Lambda}_\star)_{i,-i}^{\top}$, so that}
\endgroup
$\textcolor{black}{
\widetilde{\boldsymbol{\Lambda}}_\star^{(i)}
=
\boldsymbol{\Lambda}_\star-\mathbf e_i\mathbf w_\star^\top .
}$
\begin{enumerate}
    \item \textcolor{black}{We assume that $\boldsymbol{\Lambda}_\star$ has simple
    spectrum, i.e., $\lambda_{\star,k}\neq\lambda_{\star,\ell}$ for $k\neq \ell$.}

    \item \textcolor{black}{Let $\mathbf W_\star\in\mathbb C^{n\times n}$ be defined by}
    \[
    \textcolor{black}{
    (\mathbf W_\star)_{k\ell}
    :=
    \frac{
    (\boldsymbol{\rho}_{\star,k}^{\!\top}\mathbf e_i)
    \big(\mathbf w_\star^{\!\top}\boldsymbol{\Sigma}^{(i)}
    \boldsymbol{\rho}_{\star,\ell}\big)
    }{
    \lambda_{\star,k}+\lambda_{\star,\ell}
    }.
    }
    \]
    \textcolor{black}{We assume}
    \[
    \textcolor{black}{
    0\notin\sigma(\mathbf W_\star+\mathbf W_\star^{\!\top}).
    }
    \]

    \item \textcolor{black}{With
    $\boldsymbol{\Gamma}:=\boldsymbol{\Sigma}-\boldsymbol{\Sigma}^{(i)}$, we assume:}
    $
    \textcolor{black}{
    \mathbf e_i^{\!\top}\boldsymbol{\Gamma}^{-1}
    \boldsymbol{\Sigma}^{(i)}\mathbf e_i\neq0 .
    }$
\end{enumerate}
\end{assumption}

\begin{lemma}
\label{lemma:Gamma_inv}
Under Assumption~\ref{assum:lambda_W}, the matrix $\boldsymbol{\Xi}$ is invertible.
\end{lemma}

\begin{proof}
Since
$\textcolor{black}{
\widetilde{\boldsymbol{\Lambda}}_\star^{(i)}
=
\boldsymbol{\Lambda}_\star-\mathbf e_i\mathbf w_\star^\top,
}$
\textcolor{black}{subtracting the true observational and interventional Lyapunov
equations gives}
\begin{equation}
\label{eq:true-Gamma-Sylvester}
\textcolor{black}{
\boldsymbol{\Lambda}_\star\boldsymbol{\Gamma}
+
\boldsymbol{\Gamma}\boldsymbol{\Lambda}_\star^\top
+
\mathbf e_i\mathbf w_\star^\top\boldsymbol{\Sigma}^{(i)}
+
\boldsymbol{\Sigma}^{(i)}\mathbf w_\star\mathbf e_i^\top
=
0.
}
\end{equation}

\textcolor{black}{Let $\mathbf U$ solve the Sylvester equation}
\begin{equation}
\label{eq:U-Syl-star}
\textcolor{black}{
\boldsymbol{\Lambda}_\star\mathbf U
+
\mathbf U\boldsymbol{\Lambda}_\star^{\!\top}
=
-\mathbf e_i\big(\mathbf w_\star^{\!\top}\boldsymbol{\Sigma}^{(i)}\big).
}
\end{equation}
\begingroup\color{black}
\textcolor{black}{The Sylvester operator $X\mapsto AX+XB$ is
invertible precisely when $A$ and $-B$ have disjoint spectra.
Here $A=\boldsymbol\Lambda_\star$ and $B=\boldsymbol\Lambda_\star^\top$;
positive stability makes every eigenvalue sum nonzero. Thus this
operator is invertible. Transposing \eqref{eq:U-Syl-star} and adding
gives}
\endgroup
\[
\textcolor{black}{
\boldsymbol{\Lambda}_\star(\mathbf U+\mathbf U^\top)
+
(\mathbf U+\mathbf U^\top)\boldsymbol{\Lambda}_\star^\top
=
-\mathbf e_i\mathbf w_\star^\top\boldsymbol{\Sigma}^{(i)}
-
\boldsymbol{\Sigma}^{(i)}\mathbf w_\star\mathbf e_i^\top .
}
\]
\textcolor{black}{Comparing with \eqref{eq:true-Gamma-Sylvester} and using
uniqueness gives}
\begin{equation}
\label{eq:Gamma-U-star}
\textcolor{black}{
\boldsymbol{\Gamma}=\mathbf U+\mathbf U^\top .
}
\end{equation}

\textcolor{black}{Expand $\mathbf U$ in the right-eigenvector basis on both sides:}
\[
\textcolor{black}{
\mathbf U
=
\sum_{k,\ell}u_{k\ell}\,
\mathbf r_{\star,k}\mathbf r_{\star,\ell}^{\!\top},
}
\]
\textcolor{black}{where 
$u_{k\ell}
=
\boldsymbol{\rho}_{\star,k}^{\!\top}\mathbf U
\boldsymbol{\rho}_{\star,\ell}$. Multiplying \eqref{eq:U-Syl-star} on the left
by $\boldsymbol{\rho}_{\star,k}^{\!\top}$ and on the right by
$\boldsymbol{\rho}_{\star,\ell}$ gives}
\[
\textcolor{black}{
(\lambda_{\star,k}+\lambda_{\star,\ell})u_{k\ell}
=
-
(\boldsymbol{\rho}_{\star,k}^{\!\top}\mathbf e_i)
\big(\mathbf w_\star^{\!\top}\boldsymbol{\Sigma}^{(i)}
\boldsymbol{\rho}_{\star,\ell}\big).
}
\]
\textcolor{black}{Thus $u_{k\ell}=-(\mathbf W_\star)_{k\ell}$, and}
\begin{equation}
\label{eq:U-modal-star}
\textcolor{black}{
\mathbf U=-\mathbf R_\star\mathbf W_\star\mathbf R_\star^\top .
}
\end{equation}
\textcolor{black}{Using \eqref{eq:Gamma-U-star},}
\[
\textcolor{black}{
\boldsymbol{\Gamma}
=
-\mathbf R_\star(\mathbf W_\star+\mathbf W_\star^\top)\mathbf R_\star^\top .
}
\]
\textcolor{black}{By Assumption~\ref{assum:lambda_W}(2),
$\mathbf W_\star+\mathbf W_\star^\top$ is invertible. Since
$\mathbf R_\star$ is invertible, $\boldsymbol{\Gamma}$ is invertible.}

Let $J=\{1,\dots,n\}\setminus\{i\}$. Define
\[
\mathbf M_i
:=
\begin{bmatrix}
\boldsymbol{\Sigma}^{(i)}_{ii} & \boldsymbol{\Gamma}_{i,J}\\[2pt]
\boldsymbol{\Sigma}^{(i)}_{J,i} & \boldsymbol{\Gamma}_{J,J}
\end{bmatrix}.
\]
\begingroup\color{black}
\textcolor{black}{This matrix is obtained from $\boldsymbol{\Gamma}$ by replacing
its $i$-th column with $\boldsymbol{\Sigma}^{(i)}\mathbf e_i$, followed
by the same row and column permutation placing $i$ first; this
permutation leaves the determinant unchanged. By Cramer's rule,}
\endgroup
\[
\textcolor{black}{
\det(\mathbf M_i)
=
\det(\boldsymbol{\Gamma})\,
\mathbf e_i^\top\boldsymbol{\Gamma}^{-1}
\boldsymbol{\Sigma}^{(i)}\mathbf e_i .
}
\]
\begingroup\color{black}
\textcolor{black}{By Assumption~\ref{assum:lambda_W}(3), $\det(\mathbf M_i)\neq0$.
On the other hand, the Schur determinant identity gives,}
\endgroup
\[
\textcolor{black}{
\det(\mathbf M_i)
=
\boldsymbol{\Sigma}^{(i)}_{ii}\det(\boldsymbol{\Xi}).
}
\]
\textcolor{black}{Since $\boldsymbol{\Sigma}^{(i)}_{ii}>0$, we get
$\det(\boldsymbol{\Xi})\neq0$. Thus $\boldsymbol{\Xi}$ is invertible.}
\end{proof}

\subsection{Proof of Theorem \ref{th:multiSCC}}
\label{app:multiSCC}
\begingroup\color{black}
Let $\mathbf{T}$ be the target block. For a nonsingleton root SCC, its
\endgroup
observational and interventional marginals are sub-blocks of the observed
\begingroup\color{black}
moments. The isolated-realization hypothesis makes the $\boldsymbol\Xi$
invertibility and skew-kernel conditions used in the proof of
Theorem~\ref{th:Single SCC} generic, by the rational-rank argument below.
That proof therefore identifies the root parameters up to a positive common
scale. If the root SCC is a singleton, its observational equations
$\lambda\mu=b$ and $2\lambda\Sigma=d$, with $\lambda>0$, give the same conclusion.
\endgroup
Suppose the parameters of block $\mathbf{P}$ satisfy
$\mathbf{b}_{\mathbf{P}}=c\,\bar{\mathbf{b}}_{\mathbf{P}}$ and
$\mathbf{D}_{\mathbf{P}}=c\,\bar{\mathbf{D}}_{\mathbf{P}}$ for a scalar $c>0$ and known representatives
$\bar{\mathbf{b}}_{\mathbf{P}},\bar{\mathbf{D}}_{\mathbf{P}}$. Define the residual covariances and cross terms
\[
\mathbf{Z}:=\boldsymbol{\Sigma}_{\mathbf{T}\mid \mathbf{P}},\quad
\mathbf{Z}^{(i)}:=\boldsymbol{\Sigma}^{(i)}_{\mathbf{T}\mid \mathbf{P}},\quad
\mathbf{B}:=\boldsymbol{\Sigma}_{\mathbf{T}\mathbf{P}}\boldsymbol{\Sigma}_{\mathbf{P}\mathbf{P}}^{-1},\quad
\mathbf{B}^{(i)}:=\boldsymbol{\Sigma}^{(i)}_{\mathbf{T}\mathbf{P}}\big(\boldsymbol{\Sigma}^{(i)}_{\mathbf{P}\mathbf{P}}\big)^{-1}.
\]
\textcolor{black}{Define the parent mean contribution}
$\textcolor{black}{
\mathbf h_{\mathbf P}
:=
\bar{\mathbf b}_{\mathbf P}
-
\bar{\mathbf D}_{\mathbf P}
\boldsymbol{\Sigma}_{\mathbf P\mathbf P}^{-1}
\boldsymbol{\mu}_{\mathbf P}.
}$
\textcolor{black}{Then set}
\[
\textcolor{black}{
\mathbf{v}:=\mathbf{B}\mathbf h_{\mathbf P},\quad
\mathbf{v}^{(i)}:=\mathbf{B}^{(i)}\mathbf h_{\mathbf P},
}
\qquad
\mathbf{Q}:=\mathbf{B}\,\bar{\mathbf{D}}_{\mathbf{P}}\,\mathbf{B}^\top,\quad
\mathbf{Q}^{(i)}:=\mathbf{B}^{(i)}\bar{\mathbf{D}}_{\mathbf{P}}\,\mathbf{B}^{(i)\top}.
\]
\textcolor{black}{Here we used that the intervention is inside the target block $\mathbf T$, so the upstream parent moments $\boldsymbol{\mu}_{\mathbf P}$ and $\boldsymbol{\Sigma}_{\mathbf P\mathbf P}$ are unchanged by the intervention.}
Indeed, under the hard intervention, row \(i\) of the full drift has no
off-diagonal entries. In particular, the \(i\)-th row of
\(\widetilde{\boldsymbol{\Lambda}}^{(i)}_{\mathbf T\mathbf P}\) is zero and the
\(i\)-th row of
\(\widetilde{\boldsymbol{\Lambda}}^{(i)}_{\mathbf T\mathbf T}\) is
\((\lambda_{ii},0)\). Taking the \(i\)-th row of the interventional
cross-covariance Lyapunov block gives
\[
\lambda_{ii}\boldsymbol{\Sigma}^{(i)}_{i,\mathbf P}
+
\boldsymbol{\Sigma}^{(i)}_{i,\mathbf P}
\boldsymbol{\Lambda}_{\mathbf P\mathbf P}^{\top}
=0.
\]
Since \(\lambda_{ii}\mathbf I+\boldsymbol{\Lambda}_{\mathbf P\mathbf P}^{\top}\)
is invertible, we obtain
\(\boldsymbol{\Sigma}^{(i)}_{i,\mathbf P}=0\). Hence the \(i\)-th row of
\(\mathbf B^{(i)}
=
\boldsymbol{\Sigma}^{(i)}_{\mathbf T\mathbf P}
(\boldsymbol{\Sigma}^{(i)}_{\mathbf P\mathbf P})^{-1}\)
is zero, and therefore
\((\mathbf Q^{(i)})_{i,-i}=(\mathbf Q^{(i)})_{-i,i}=0\).


Partition the drift on $\mathbf{T}$ as
\[
\boldsymbol{\Lambda}_{\mathbf{T}\mathbf{T}}=
\begin{bmatrix}
\lambda_{ii} & \boldsymbol{\lambda}_{i,-i}\\
\boldsymbol{\lambda}_{-i,i} & \boldsymbol{\Lambda}_{-i,-i}
\end{bmatrix},
\qquad
\widetilde{\boldsymbol{\Lambda}}^{(i)}_{\mathbf{T}\mathbf{T}}=
\begin{bmatrix}
\lambda_{ii} & 0\\
\boldsymbol{\lambda}_{-i,i} & \boldsymbol{\Lambda}_{-i,-i}
\end{bmatrix},
\]
with unknown diagonal $\mathbf{D}_{\mathbf{T}}$ and vector $\mathbf{b}_{\mathbf{T}}$.

\textcolor{black}{We first justify the residual mean equation. From the block mean equations}
\[
\textcolor{black}{
\boldsymbol{\Lambda}_{\mathbf P\mathbf P}\boldsymbol{\mu}_{\mathbf P}
=
\mathbf b_{\mathbf P},
\qquad
\boldsymbol{\Lambda}_{\mathbf T\mathbf P}\boldsymbol{\mu}_{\mathbf P}
+
\boldsymbol{\Lambda}_{\mathbf T\mathbf T}\boldsymbol{\mu}_{\mathbf T}
=
\mathbf b_{\mathbf T},
}
\]
\textcolor{black}{and from}
\[
\textcolor{black}{
\boldsymbol{\mu}_{\mathbf T}
=
\boldsymbol{\mu}_{\mathbf T\mid\mathbf P}
+
\mathbf B\boldsymbol{\mu}_{\mathbf P},
}
\]
\textcolor{black}{we obtain}
\[
\textcolor{black}{
\boldsymbol{\Lambda}_{\mathbf T\mathbf T}\boldsymbol{\mu}_{\mathbf T\mid\mathbf P}
=
\mathbf b_{\mathbf T}
-
(\boldsymbol{\Lambda}_{\mathbf T\mathbf P}
+
\boldsymbol{\Lambda}_{\mathbf T\mathbf T}\mathbf B)
\boldsymbol{\mu}_{\mathbf P}.
}
\]
\textcolor{black}{The cross-covariance block of the Lyapunov equation gives}
\[
\textcolor{black}{
\boldsymbol{\Lambda}_{\mathbf T\mathbf P}\boldsymbol{\Sigma}_{\mathbf P\mathbf P}
+
\boldsymbol{\Lambda}_{\mathbf T\mathbf T}\boldsymbol{\Sigma}_{\mathbf T\mathbf P}
+
\boldsymbol{\Sigma}_{\mathbf T\mathbf P}\boldsymbol{\Lambda}_{\mathbf P\mathbf P}^{\top}
=
0.
}
\]
\textcolor{black}{Using $\mathbf B=\boldsymbol{\Sigma}_{\mathbf T\mathbf P}\boldsymbol{\Sigma}_{\mathbf P\mathbf P}^{-1}$, this implies}
\[
\textcolor{black}{
\boldsymbol{\Lambda}_{\mathbf T\mathbf P}
+
\boldsymbol{\Lambda}_{\mathbf T\mathbf T}\mathbf B
=
-\mathbf B\boldsymbol{\Sigma}_{\mathbf P\mathbf P}
\boldsymbol{\Lambda}_{\mathbf P\mathbf P}^{\top}
\boldsymbol{\Sigma}_{\mathbf P\mathbf P}^{-1}.
}
\]
\textcolor{black}{The parent Lyapunov equation}
\[
\textcolor{black}{
\boldsymbol{\Lambda}_{\mathbf P\mathbf P}\boldsymbol{\Sigma}_{\mathbf P\mathbf P}
+
\boldsymbol{\Sigma}_{\mathbf P\mathbf P}\boldsymbol{\Lambda}_{\mathbf P\mathbf P}^{\top}
=
\mathbf D_{\mathbf P}
}
\]
\textcolor{black}{then gives}
\[
\textcolor{black}{
\boldsymbol{\Lambda}_{\mathbf T\mathbf P}
+
\boldsymbol{\Lambda}_{\mathbf T\mathbf T}\mathbf B
=
\mathbf B\boldsymbol{\Lambda}_{\mathbf P\mathbf P}
-
\mathbf B\mathbf D_{\mathbf P}
\boldsymbol{\Sigma}_{\mathbf P\mathbf P}^{-1}.
}
\]
\textcolor{black}{Since $\mathbf b_{\mathbf P}=c\bar{\mathbf b}_{\mathbf P}$ and $\mathbf D_{\mathbf P}=c\bar{\mathbf D}_{\mathbf P}$, we get}
\[
\textcolor{black}{
\boldsymbol{\Lambda}_{\mathbf T\mathbf T}\boldsymbol{\mu}_{\mathbf T\mid\mathbf P}
=
\mathbf b_{\mathbf T}
-
c\,\mathbf B
\left(
\bar{\mathbf b}_{\mathbf P}
-
\bar{\mathbf D}_{\mathbf P}
\boldsymbol{\Sigma}_{\mathbf P\mathbf P}^{-1}
\boldsymbol{\mu}_{\mathbf P}
\right)
=
\mathbf b_{\mathbf T}-c\mathbf v.
}
\]

On $\mathbf{T}$, the stationary equations are
\[
\boldsymbol{\Lambda}_{\mathbf{T}\mathbf{T}}\mathbf{Z}
+
\mathbf{Z}\boldsymbol{\Lambda}_{\mathbf{T}\mathbf{T}}^\top
=
\mathbf{D}_{\mathbf{T}} + c\,\mathbf{Q},
\qquad
\boldsymbol{\Lambda}_{\mathbf{T}\mathbf{T}}\boldsymbol{\mu}_{\mathbf{T}\mid \mathbf{P}}
=
\mathbf{b}_{\mathbf{T}} - c\,\mathbf{v},
\]
and, under intervention,
\begin{equation}
\widetilde{\boldsymbol{\Lambda}}^{(i)}_{\mathbf{T}\mathbf{T}}\mathbf{Z}^{(i)}
+
\mathbf{Z}^{(i)}\widetilde{\boldsymbol{\Lambda}}^{(i)\top}_{\mathbf{T}\mathbf{T}}
=
\mathbf{D}_{\mathbf{T}} + c\,\mathbf{Q}^{(i)},
\qquad
\widetilde{\boldsymbol{\Lambda}}^{(i)}_{\mathbf{T}\mathbf{T}}\boldsymbol{\mu}^{(i)}_{\mathbf{T}\mid \mathbf{P}}
=
\mathbf{b}_{\mathbf{T}} - c\,\mathbf{v}^{(i)}.
\label{eq:Multi_Z(i)}
\end{equation}
Subtracting gives the difference relations
\[
\boldsymbol{\Lambda}_{\mathbf{T}\mathbf{T}}\mathbf{Z}
+
\mathbf{Z}\boldsymbol{\Lambda}_{\mathbf{T}\mathbf{T}}^\top
-
\big(
\widetilde{\boldsymbol{\Lambda}}^{(i)}_{\mathbf{T}\mathbf{T}}\mathbf{Z}^{(i)}
+
\mathbf{Z}^{(i)}\widetilde{\boldsymbol{\Lambda}}^{(i)\top}_{\mathbf{T}\mathbf{T}}
\big)
=
c\,(\mathbf{Q}-\mathbf{Q}^{(i)}),
\]
\[
\boldsymbol{\Lambda}_{\mathbf{T}\mathbf{T}}\boldsymbol{\mu}_{\mathbf{T}\mid \mathbf{P}}
-
\widetilde{\boldsymbol{\Lambda}}^{(i)}_{\mathbf{T}\mathbf{T}}\boldsymbol{\mu}^{(i)}_{\mathbf{T}\mid \mathbf{P}}
=
c\,(\mathbf{v}^{(i)}-\mathbf{v}).
\]

\paragraph{Interventional \textcolor{black}{$(-i,i)$} block.}
Taking the \textcolor{black}{$(-i,i)$} block of the interventional Lyapunov equation and using that
$\mathbf{D}_{\mathbf{T}}$ is diagonal and \textcolor{black}{$(\mathbf{Q}^{(i)})_{-i,i}=0$} yields
\[
\lambda_{ii}\,\mathbf{Z}^{(i)}_{-i,i}
+
\boldsymbol{\lambda}_{-i,i}\,Z^{(i)}_{ii}
+
\boldsymbol{\Lambda}_{-i,-i}\,\mathbf{Z}^{(i)}_{-i,i}
=0.
\]
Since $Z^{(i)}_{ii}>0$, write
\[
\boldsymbol{u}:=\frac{\mathbf{Z}^{(i)}_{-i,i}}{Z^{(i)}_{ii}}
\]
to obtain
\begin{equation}
\boldsymbol{\lambda}_{-i,i}
+
(\boldsymbol{\Lambda}_{-i,-i}+\lambda_{ii}\mathbf{I})\,\boldsymbol{u}
=0.
\label{eq:lambda-i,i}
\end{equation}

\paragraph{The $\boldsymbol{\Xi}$-equation on $(-i,-i)$.}
Let $\boldsymbol{\Gamma}:=\mathbf{Z}-\mathbf{Z}^{(i)}$ and define the data-only matrix
\[
\boldsymbol{\Xi}:=\boldsymbol{\Gamma}_{-i,-i} - \boldsymbol{u}\,\boldsymbol{\Gamma}_{i,-i}.
\]
From the $(-i,-i)$ block of the difference Lyapunov equation and the identity above,
\begin{equation}
\boldsymbol{\Lambda}_{-i,-i}\,\boldsymbol{\Xi}
+
\boldsymbol{\Xi}^\top\boldsymbol{\Lambda}_{-i,-i}^\top
=
c\,(\mathbf{Q}-\mathbf{Q}^{(i)})_{-i,-i}
+
\lambda_{ii}\,
\Big(
\boldsymbol{u}\,\boldsymbol{\Gamma}_{i,-i}
+
\boldsymbol{\Gamma}_{-i,i}\,\boldsymbol{u}^\top
\Big).
\label{eq:multi_Xi}
\end{equation}

From \eqref{eq:Multi_Z(i)} and \eqref{eq:lambda-i,i}, we derive
\begin{equation}
\operatorname{offdiag}\!\Big(
\boldsymbol{\Lambda}_{-i,-i}\,\mathbf S^{(i)}
+
\mathbf S^{(i)}\,\boldsymbol{\Lambda}_{-i,-i}^{\!\top}
\Big)
-
\lambda_{ii}\,\operatorname{offdiag}\!\Big(
\boldsymbol{u}\,\mathbf Z^{(i)}_{i,-i}
+
\mathbf Z^{(i)}_{-i,i}\,\boldsymbol{u}^{\!\top}
\Big)
=
\operatorname{offdiag}\!\big(c\,\mathbf Q^{(i)}_{-i,-i}\big),
\label{eq:S-eq}
\end{equation}
with
\[
\mathbf S^{(i)}
:=
\mathbf Z^{(i)}_{-i,-i}
-
\boldsymbol{u}\,\mathbf Z^{(i)}_{i,-i}
\succ0.
\]
Define the linear maps
\[
\mathcal A_S(\mathbf X)
:=
\operatorname{offdiag}(\mathbf X\mathbf S^{(i)}+\mathbf S^{(i)}\mathbf X^{\!\top}),
\qquad
\mathcal A_\Xi(\mathbf X)
:=
\mathbf X\boldsymbol{\Xi}+\boldsymbol{\Xi}^{\!\top}\mathbf X^{\!\top}.
\]
Equations \eqref{eq:multi_Xi} and \eqref{eq:S-eq} form a stacked linear system in $\boldsymbol{\Lambda}_{-i,-i}$,
\[
\begin{bmatrix}
\mathcal A_S \\[2pt] \mathcal A_\Xi
\end{bmatrix}\!(\boldsymbol{\Lambda}_{-i,-i})
=
c\,
\begin{bmatrix}
\operatorname{offdiag}(\mathbf Q^{(i)}_{-i,-i})\\[2pt]
(\mathbf Q-\mathbf Q^{(i)})_{-i,-i}
\end{bmatrix}
+
\lambda_{ii}\,
\begin{bmatrix}
\operatorname{offdiag}\!\big(\boldsymbol{u}\,\mathbf Z^{(i)}_{i,-i}+\mathbf Z^{(i)}_{-i,i}\,\boldsymbol{u}^{\!\top}\big)\\[2pt]
\boldsymbol{u}\,\boldsymbol{\Gamma}_{i,-i} + \boldsymbol{\Gamma}_{-i,i}\,\boldsymbol{u}^{\!\top}
\end{bmatrix}.
\]

Since both the true and competing drift matrices are compatible with the
fixed graph $G$, their relevant blocks belong to the same allowed linear
space. 
Let $\mathcal V_G$ be the linear space of matrices with the allowed sparsity pattern
of $\boldsymbol{\Lambda}_{-i,-i}$. Define
\[
\mathcal L_{\rm blk}(\mathbf X)
:=
\begin{bmatrix}
\operatorname{offdiag}(\mathbf X\mathbf S^{(i)}+\mathbf S^{(i)}\mathbf X^\top)\\[2pt]
\mathbf X\boldsymbol{\Xi}+\boldsymbol{\Xi}^{\top}\mathbf X^\top
\end{bmatrix},
\qquad \mathbf X\in\mathcal V_G.
\]
\begingroup\color{black}
For a nonsingleton target SCC, choose the isolated realization supplied by
the theorem's hypothesis and set its incoming couplings to zero. Embed it
in the full graph by taking the remaining variables to be independent scalar
OU processes. The residual target moments then equal the isolated moments.
For a singleton target, $\mathcal V_G$ has dimension zero, so the stacked
operator is injective automatically.

\endgroup
Let $\mathbf L_{\rm blk}$ be any matrix representation of $\mathcal L_{\rm blk}$. At the point \(\boldsymbol{\Lambda}_{\mathbf T\mathbf P}=0\), the
residual target-block equations reduce to the single-SCC equations. Hence, under
the single-SCC nondegeneracy assumptions, the stacked operator
\(\mathcal L_{\rm blk}\) is injective at this point. Since the entries of a
matrix representation of \(\mathcal L_{\rm blk}\) depend continuously, indeed
rationally, on the model parameters, this full-column-rank condition is not
identically violated and therefore holds generically. Thus, generically, the
stacked linear system uniquely determines \(\boldsymbol{\Lambda}_{-i,-i}\) from
\begingroup\color{black}
its right-hand side. A \emph{rank certificate} here is one nonzero
maximal minor of $\mathbf L_{\rm blk}$, and the genericity conclusion
uses the rational-witness argument in Appendix~\ref{app:genericity}.
The single-SCC assumptions in fact give injectivity on the full reduced
matrix space at the isolated witness, and hence on its subspace
$\mathcal V_G$.
\endgroup
 Since the right-hand side is affine in $c$ and $\lambda_{ii}$, there exist data-dependent matrices $\mathbf K_0,\mathbf K_2$ such that
\begin{equation}
\boldsymbol{\Lambda}_{-i,-i}
=
c\,\mathbf K_0
+
\lambda_{ii}\,\mathbf K_2.
\label{eq:Lii-affine}
\end{equation}
\begingroup\color{black}
For a fixed parent scale $c$, this is affine dependence on
$\lambda_{ii}$.
\endgroup
\textcolor{black}{From \eqref{eq:lambda-i,i},}
\[
\textcolor{black}{
\boldsymbol{\lambda}_{-i,i}
=
-\big(c\mathbf K_0+\lambda_{ii}\mathbf K_2+\lambda_{ii}\mathbf I\big)\boldsymbol u .
}
\]

\textcolor{black}{Using the observational $(i,-i)$ block of the Lyapunov equation,}
\[
\textcolor{black}{
\lambda_{ii}\mathbf Z_{i,-i}
+
\boldsymbol{\lambda}_{i,-i}\mathbf Z_{-i,-i}
+
Z_{ii}\boldsymbol{\lambda}_{-i,i}^{\top}
+
\mathbf Z_{i,-i}\boldsymbol{\Lambda}_{-i,-i}^{\top}
=
c\mathbf Q_{i,-i}.
}
\]
\textcolor{black}{Since $\mathbf Z_{-i,-i}\succ0$, we obtain}
\[
\textcolor{black}{
\boldsymbol{\lambda}_{i,-i}
=
\big(
c\mathbf Q_{i,-i}
-\lambda_{ii}\mathbf Z_{i,-i}
-
Z_{ii}\boldsymbol{\lambda}_{-i,i}^{\top}
-
\mathbf Z_{i,-i}\boldsymbol{\Lambda}_{-i,-i}^{\top}
\big)\mathbf Z_{-i,-i}^{-1}.
}
\]
\textcolor{black}{Substituting the expressions for $\boldsymbol{\Lambda}_{-i,-i}$ and $\boldsymbol{\lambda}_{-i,i}$ gives}
\begin{equation}
\textcolor{black}{
\boldsymbol{\lambda}_{i,-i}
=
\mathbf A_0+\lambda_{ii}\mathbf A_1,
}
\label{eq:lambda-i-minus-affine}
\end{equation}
\textcolor{black}{where}
\[
\textcolor{black}{
\mathbf A_0
=
c\Big(
\mathbf Q_{i,-i}
+
Z_{ii}\boldsymbol u^\top\mathbf K_0^\top
-
\mathbf Z_{i,-i}\mathbf K_0^\top
\Big)\mathbf Z_{-i,-i}^{-1},
}
\]
and
\[
\textcolor{black}{
\mathbf A_1
=
\Big(
-\mathbf Z_{i,-i}
+
Z_{ii}\boldsymbol u^\top(\mathbf I+\mathbf K_2^\top)
-
\mathbf Z_{i,-i}\mathbf K_2^\top
\Big)\mathbf Z_{-i,-i}^{-1}.
}
\]
\textcolor{black}{Thus $\mathbf A_0$ scales linearly with $c$, while $\mathbf K_2$ and $\mathbf A_1$ are data-dependent and do not scale with $c$.}

From the $i$-th component of the mean difference, let
\[
\Delta\mu_i
:=
(\boldsymbol{\mu}_{\mathbf{T}\mid \mathbf{P}})_i
-
(\boldsymbol{\mu}^{(i)}_{\mathbf{T}\mid \mathbf{P}})_i,
\qquad
\Delta v_i:=v^{(i)}_i-v_i.
\]
Then
\[
\lambda_{ii}\,\Delta\mu_i
+
\boldsymbol{\lambda}_{i,-i}(\boldsymbol{\mu}_{\mathbf{T}\mid \mathbf{P}})_{-i}
=
c\,\Delta v_i.
\]
Insert $\boldsymbol{\lambda}_{i,-i}=\mathbf{A}_0+\lambda_{ii}\mathbf{A}_1$ and group the terms to obtain
\[
\lambda_{ii}
\big(
\Delta\mu_i
+
\mathbf{A}_1(\boldsymbol{\mu}_{\mathbf{T}\mid \mathbf{P}})_{-i}
\big)
=
c\Delta v_i
-
\mathbf{A}_0(\boldsymbol{\mu}_{\mathbf{T}\mid \mathbf{P}})_{-i}.
\]

\begingroup\color{black}
The denominator below is generically nonzero. To justify the parameter
choice, write $A=\boldsymbol\Lambda_{\mathbf P\mathbf P}$ and
$S=\boldsymbol\Sigma_{\mathbf P\mathbf P}$ and normalize the parent scale
to one for this witness. The parent Lyapunov equation gives
\[
\mathbf h_{\mathbf P}
=\bigl(I-\mathbf D_{\mathbf P}S^{-1}A^{-1}\bigr)\mathbf b_{\mathbf P}
=-SA^\top S^{-1}A^{-1}\mathbf b_{\mathbf P}.
\]
This map is invertible. An allowed incoming edge followed by a simple
directed path inside $\mathbf T$ to $i$ gives a stable sparse realization
with $B_{ij}\ne0$ for some $j\in\mathbf P$.\footnote{\color{black} For the sparse path realization, take the parent drift and covariance to
be $I$, with parent diffusion $2I$, and target drift $2I-N$, where $N$
has unit entries along the chosen path $k=k_0\to\cdots\to k_m=i$ and
is zero elsewhere. Set
$\boldsymbol\Lambda_{\mathbf T\mathbf P}=-\eta e_ke_j^\top$ with
$\eta\ne0$ and use positive diagonal target diffusion. These drifts,
including hard-intervened drifts, are triangular after ordering the path
vertices and have positive diagonals. The cross-covariance equation gives
$B_{ij}=\eta/3^{m+1}\ne0$.} Both this entry and a
full-column-rank certificate of $\mathcal L_{\rm blk}$ are rational and
have nonzero witnesses, which may be different parameter choices.
By Appendix~\ref{app:genericity}, their nonvanishing can be made simultaneous
in an arbitrarily small admissible neighborhood of the decoupled rank
witness. Fix such covariance parameters, choose $\mathbf b_{\mathbf P}$
so that $v_i\ne0$, and take $\mathbf b_{\mathbf T}=\mathbf v$.
Then $\boldsymbol\mu_{\mathbf T\mid\mathbf P}=0$, while
$v_i^{(i)}=0$ and the interventional mean equation give
$\lambda_{ii}(\boldsymbol\mu^{(i)}_{\mathbf T\mid\mathbf P})_i=v_i$.
Consequently the denominator equals $-v_i/\lambda_{ii}\ne0$.
The denominator is rational on the full-rank admissible domain, so its
zero set is nongeneric. Restoring a parent scale $c$ gives
$-cv_i/\lambda_{ii}$ at the corresponding witness.
\endgroup
Hence $\lambda_{ii}$ is identified as
\[
\lambda_{ii}
=
\frac{
c\Delta v_i
-
\mathbf{A}_0(\boldsymbol{\mu}_{\mathbf{T}\mid \mathbf{P}})_{-i}
}{
\Delta\mu_i
+
\mathbf{A}_1(\boldsymbol{\mu}_{\mathbf{T}\mid \mathbf{P}})_{-i}
},
\]
where it is determined up to the same scale $c$.

With $\lambda_{ii}$ identified, the previous equations give
\[
\boldsymbol{\Lambda}_{-i,-i}
=
c\mathbf{K}_0
+
\lambda_{ii}\,\mathbf{K}_2,
\qquad
\boldsymbol{\lambda}_{-i,i}
=
-\Big(c\mathbf{K}_0+\lambda_{ii}\,\mathbf{K}_2+\lambda_{ii}\mathbf{I}\Big)\boldsymbol{u},
\qquad
\boldsymbol{\lambda}_{i,-i}
=
\mathbf{A}_0+\lambda_{ii}\mathbf{A}_1.
\]
Thus the entire $\boldsymbol{\Lambda}_{\mathbf{T}\mathbf{T}}$ is determined with the same scaling $c$.

Finally,
\[
\textcolor{black}{
\mathbf{D}_{\mathbf{T}}
=
\mathrm{diag}\big(
\boldsymbol{\Lambda}_{\mathbf{T}\mathbf{T}}\mathbf{Z}
+
\mathbf{Z}\boldsymbol{\Lambda}_{\mathbf{T}\mathbf{T}}^\top
-
c\,\mathbf{Q}
\big),
}
\]
and
\[
\textcolor{black}{
\mathbf{b}_{\mathbf{T}}
=
\boldsymbol{\Lambda}_{\mathbf{T}\mathbf{T}}
\boldsymbol{\mu}_{\mathbf{T}\mid \mathbf{P}}
+
c\,\mathbf{v}.
}
\]
Thus $(\boldsymbol{\Lambda}_{\mathbf{T}\mathbf{T}},\mathbf{D}_{\mathbf{T}},\mathbf{b}_{\mathbf{T}})$ are identified up to the same multiplicative scale $c$.

\subsection{An Example with Multiple Root SCCs}
\label{app:Example}
Consider an OU process with state indices $1,2,3,4$, where $1,2,3$ are root singletons and $4$ is a child singleton. Let
\[
\boldsymbol{\Lambda}=\begin{bmatrix}
\lambda_1&0&0&0\\
0&\lambda_2&0&0\\
0&0&\lambda_3&0\\
a_1&a_2&a_3&\lambda_4
\end{bmatrix},
\qquad
\mathbf{D}=\operatorname{diag}(d_1,d_2,d_3,d_4),
\]
and assume $\mathbf{D}$ is diagonal and each $\lambda_i>0$. Suppose we observe the steady-state mean and covariance $(\boldsymbol{\mu},\boldsymbol{\Sigma})$ and also the interventional moments $(\boldsymbol{\mu}^{(4)},\boldsymbol{\Sigma}^{(4)})$ under an intervention on node $4$ that zeros the entries $a_i$s (please note that interventions on root give the same mean and covariance as the observational ones).

For each root $i\in\{1,2,3\}$, the scalar OU gives
\[
\mu_i=\frac{b_i}{\lambda_i},\qquad s_i:=\Sigma_{ii}=\frac{d_i}{2\lambda_i}.
\]
Let $t_i:=\Sigma_{i4}$ denote the observed cross-covariances with the child. The off-diagonal Lyapunov equations yield
\begin{equation}
(\lambda_i+\lambda_4)t_i+s_i a_i=0
\quad\Longleftrightarrow\quad
a_i=-\frac{\lambda_i+\lambda_4}{s_i}\,t_i.
\label{eq:ai_from_ti}
\end{equation}
The child's observed mean and variance satisfy
\begin{align}
\lambda_4\mu_4+\sum_{i=1}^3 a_i\mu_i&=b_4, \label{eq:child_mean}\\
2\lambda_4 s_4+2\sum_{i=1}^3 a_i t_i&=d_4. \label{eq:child_var}
\end{align}
Under the intervention on $4$ (which zeros the $a_i$s), we also observe
\begin{equation}
\mu_4^{(4)}=\frac{b_4}{\lambda_4},\qquad
2\lambda_4 s_4^{(4)}=d_4,\qquad
t_i^{(4)}=0.
\label{eq:int_child}
\end{equation}

Now, fix any positive $(u_1,u_2,u_3)$ and a positive $v$ (to be chosen). Define
\[
\lambda_i'=u_i\lambda_i,\quad b_i'=u_i b_i,\quad d_i'=u_i d_i \ \ (i=1,2,3),\qquad
\lambda_4'=v\lambda_4,\quad b_4'=v b_4,\quad d_4'=v d_4,
\]
and choose $a_i'$ to preserve the observed $t_i$:
\begin{equation}
(\lambda_i'+\lambda_4')t_i+s_i a_i'=0
\quad\Longleftrightarrow\quad
a_i'=-\frac{u_i\lambda_i+v\lambda_4}{s_i}\,t_i.
\label{eq:ai_prime}
\end{equation}
Then for each root, $\mu_i'=\frac{b_i'}{\lambda_i'}=\mu_i$ and $s_i'=\frac{d_i'}{2\lambda_i'}=s_i$, while \eqref{eq:ai_prime} ensures the same $t_i$.

For the interventional moments on $4$, scaling $(\lambda_4,b_4,d_4)$ by the common factor $v$ preserves
$\mu_4^{(4)}=\frac{b_4}{\lambda_4}$ and $s_4^{(4)}=\frac{d_4}{2\lambda_4}$, and $t_i^{(4)}=0$ still holds since the $a_i$ are zeroed under intervention. Thus all interventional moments match \eqref{eq:int_child}.

It remains to enforce that the primed parameters also satisfy the child's \emph{observational} equations \eqref{eq:child_mean}--\eqref{eq:child_var}. Subtracting $v$ times the original \eqref{eq:child_var} from the primed version gives
\[
2\lambda_4's_4-2v\lambda_4 s_4+2\sum_{i=1}^3 (a_i'-v a_i)t_i=0
\ \Longleftrightarrow\
\sum_{i=1}^3 (a_i'-v a_i)t_i=0.
\]
Using \eqref{eq:ai_from_ti} and \eqref{eq:ai_prime},
\[
a_i'-v a_i
=-\frac{u_i\lambda_i+v\lambda_4}{s_i}t_i+v\frac{\lambda_i+\lambda_4}{s_i}t_i
=\frac{(v-u_i)\lambda_i}{s_i}\,t_i.
\]
Hence
\begin{equation}
\sum_{i=1}^3 (v-u_i)\,\lambda_i\,\frac{t_i^2}{s_i}=0
\quad\Longleftrightarrow\quad
v\sum_{i=1}^3 \alpha_i=\sum_{i=1}^3 u_i\alpha_i,
\qquad
\alpha_i:=\lambda_i\frac{t_i^2}{s_i}>0.
\label{eq:v_from_var}
\end{equation}
Thus
\begin{equation}
v=\frac{\sum_i u_i\alpha_i}{\sum_i \alpha_i}.
\label{eq:v_var_formula}
\end{equation}

Similarly, subtracting $v$ times \eqref{eq:child_mean} from the primed version gives
\[
\lambda_4'\mu_4-v\lambda_4\mu_4+\sum_{i=1}^3 (a_i'-v a_i)\mu_i=0
\ \Longleftrightarrow\
\sum_{i=1}^3 (v-u_i)\,\lambda_i\,\frac{t_i}{s_i}\,\mu_i=0,
\]
i.e.,
\begin{equation}
v\sum_{i=1}^3 \beta_i=\sum_{i=1}^3 u_i\beta_i,
\qquad
\beta_i:=\lambda_i\frac{t_i}{s_i}\,\mu_i=\frac{b_i t_i}{s_i}.
\label{eq:v_from_mean}
\end{equation}
Thus also
\begin{equation}
v=\frac{\sum_i u_i\beta_i}{\sum_i \beta_i}.
\label{eq:v_mean_formula}
\end{equation}

Equating \eqref{eq:v_var_formula} and \eqref{eq:v_mean_formula} yields a single linear constraint on $(u_1,u_2,u_3)$:
\begin{equation}
\frac{\sum_i u_i\alpha_i}{\sum_i \alpha_i}=\frac{\sum_i u_i\beta_i}{\sum_i \beta_i}
\ \Longleftrightarrow\
\sum_{i=1}^3 u_i\Big(\alpha_i\sum_j \beta_j-\beta_i\sum_j \alpha_j\Big)=0.
\label{eq:hyperplane}
\end{equation}
\begingroup\color{black}
Let $\gamma_i:=\alpha_i\sum_j \beta_j-\beta_i\sum_j \alpha_j$. Unless $(\gamma_1,\gamma_2,\gamma_3)=(0,0,0)$, \eqref{eq:hyperplane} defines a nontrivial hyperplane in $\mathbb{R}^3$, namely the
two-dimensional solution space of one nonzero linear equation (Note that the condition $\gamma_i=0$ for all $i$ is nongeneric). 
\endgroup
Indeed, it requires $\frac{\beta_1}{\alpha_1}=\frac{\beta_2}{\alpha_2}=\frac{\beta_3}{\alpha_3}$, i.e.,
$\frac{b_i}{\lambda_i t_i}$ is constant across $i$.
Moreover,
\(\sum_i\gamma_i=0\), so the hyperplane contains \((1,1,1)\). Therefore it
contains positive choices of \((u_1,u_2,u_3)\) arbitrarily close to
\((1,1,1)\), including choices that are not all equal.
This imposes two independent algebraic constraints on the continuous parameters 
$(b_i,\lambda_i,t_i)$, hence holds only on a measure-zero subset of the parameter space. 
Therefore, the set of positive $(u_1,u_2,u_3)$ satisfying \eqref{eq:hyperplane} has (at least) two degrees of freedom. For any such $(u_1,u_2,u_3)$, take $v$ from \eqref{eq:v_var_formula} (which equals \eqref{eq:v_mean_formula}) and define $a_i'$ by \eqref{eq:ai_prime}. 
Unless \(u_1=u_2=u_3=v\), the transformation is not a global rescaling of $(\boldsymbol{\Lambda},\mathbf{b},\mathbf{D})$, yet all moments (observational and interventional under node-$4$ intervention) coincide. Hence, identifiability up to a single global scale fails.

\section{Details of Experiments}
\label{app:experiments}
\subsection{Synthetic Data}\label{app:synthetic}

We generate synthetic datasets from a stable OU model. For each setup $|\mathcal{I}|\in\{0,2,4,6,8,10\}$ (number of single-node interventions) and $40$ instances, we fix $n=10$ variables and sample:
\begin{itemize}
\item \textbf{Drift matrix} $\boldsymbol{\Lambda}$: start from zeros; for $i\neq j$, set $\Lambda_{ij}\sim\mathcal{N}(0.2,\sigma^2)$ with probability $\rho$ and $0$ otherwise (default $\rho=0.3$, $\sigma=0.8$). Enforce positive stability by making each row strictly diagonally dominant:
\[
\Lambda_{ii}\leftarrow \max\!\bigl(\texttt{DIAG\_MIN},\,\sum_{j}|\Lambda_{ij}| + 0.2 + u_i\bigr),\quad u_i\sim\mathrm{Unif}[0,0.3],\ \texttt{DIAG\_MIN}=0.8.
\]
\item \textbf{Bias vector and diffusion matrix}: $\mathbf{b}\sim\mathrm{Unif}[0.2,1.5]^n$; $\mathbf{D}=\operatorname{diag}(\mathbf{d})$ with $d_k\sim\mathrm{Unif}[0.2,0.4]$.
\end{itemize}
Observational moments are computed via $\boldsymbol{\mu}=\boldsymbol{\Lambda}^{-1}\mathbf{b}$ and
$\boldsymbol{\Lambda}\boldsymbol{\Sigma}+\boldsymbol{\Sigma}\boldsymbol{\Lambda}^{\!\top}=\mathbf{D}$
(continuous-time Lyapunov). An intervention on node $k$ is implemented by zeroing row $k$ of $\boldsymbol{\Lambda}$ and restoring its diagonal entry (self-regulation preserved), leaving $\mathbf{b}$ and $\mathbf{D}$ unchanged; intervened moments $(\boldsymbol{\mu}^{(k)},\boldsymbol{\Sigma}^{(k)})$ are computed analogously. We then draw snapshots from the corresponding Gaussians (by default $40{,}000$ observational samples and $20{,}000$ samples per intervention).

\subsection{Real Data}\label{app:realdata}
\paragraph{Library-size normalization.}
For each cell $k\in\{1,\dots,T_c\}$ in context/intervention $c$ with raw counts $\{x_k^{\,i}\}_{i=1}^N$ (genes $i=1,\dots,N$), define the library size
$s_k=\sum_{i=1}^N x_k^{\,i}$. We normalize by fractions $f_k^{\,i}=x_k^{\,i}/s_k$ and then scale to a common size
\[
\tilde{x}_k^{\,i}\;=\;10^{4}\,f_k^{\,i}\;=\;10^{4}\,\frac{x_k^{\,i}}{s_k}.
\]
All downstream statistics are computed on $\tilde{x}_k^{\,i}$. We denote the cell vector
$\tilde{\mathbf{x}}_k=(\tilde{x}_k^{\,1},\dots,\tilde{x}_k^{\,N})^\top\in\mathbb{R}^N$.

\paragraph{Empirical moments with shrinkage.}
Given a context $c$ with $T_c$ cells $\{\tilde{\mathbf{x}}_k\}_{k=1}^{T_c}$, the empirical mean and (shrunk) covariance are
\[
\hat{\boldsymbol{\mu}}_c=\frac{1}{T_c}\sum_{k=1}^{T_c}\tilde{\mathbf{x}}_k,\qquad
\hat{\boldsymbol{\Sigma}}_{c}=(1-\eta_c)\,\hat{\boldsymbol{\Sigma}}_{c,\mathrm{sample}}+\eta_c\,\operatorname{diag}\!\big(\hat{\boldsymbol{\Sigma}}_{c,\mathrm{sample}}\big),
\]
where $\hat{\boldsymbol{\Sigma}}_{c,\mathrm{sample}}=\frac{1}{T_c-1}\sum_{k=1}^{T_c}(\tilde{\mathbf{x}}_k-\hat{\boldsymbol{\mu}}_c)(\tilde{\mathbf{x}}_k-\hat{\boldsymbol{\mu}}_c)^\top$ and
$\eta_c=\min\!\bigl(0.3,\;10/(T_c-1)\bigr)$.
\begingroup\color{black}
\endgroup

\paragraph{Model.}
We fit an OU model with parameters $(\boldsymbol{\Lambda},\mathbf{b},\boldsymbol{D})$ on wild-type and interventional data.
\begingroup\color{black}
\endgroup

\paragraph{Train/test split and metrics.}
Similar to previous work \cite{rohbeck2024bicycle}, interventions are split into train context IDs $\{0,\dots,54\}$ and held-out test context IDs $\{55,\dots,60\}$.  
\begingroup\color{black}
\endgroup

\subsection{Empirical study of graphical conditions}
\label{app:graphical conditions}
We emphasize that our identifiability result provides \emph{sufficient} conditions.
Thus, there may exist linear SDEs whose parameters are identifiable up to a global scaling even if
they do not satisfy the graphical conditions in Theorem \ref{th:multiSCC}. To empirically study the conservativeness of the sufficient conditions in
Theorem~\ref{th:multiSCC}, we constructed a family of OU models whose condensation
graph forms a star over SCCs. Each SCC
consists of a two-node directed cycle, with a single root SCC and $m$ leaf
SCCs. For each leaf SCC, we add exactly one directed edge from a randomly
chosen node in the root SCC to a randomly chosen node in the leaf SCC. Interventions are restricted to
leaf SCCs and are applied only to the node within each leaf SCC that does not
receive the incoming edge from the root SCC. For a fixed number of interventions $|\mathcal{I}| = k$, we randomly select
$k$ distinct leaf SCCs and apply one hard intervention per selected SCC.
For each configuration, we evaluate whether the model parameters are
identifiable up to a global scaling by analyzing the stacked linear system
induced by the stationary first- and second-moment equations across all
observational and interventional settings. Identifiability is assessed
numerically via a spectral-gap criterion where we compute the singular values of
the stacked system and declare identifiability if there is a clear separation
between the smallest singular value and the remaining spectrum, corresponding
to a one-dimensional nullspace. In all experiments, we used a threshold of $10^{10}$ for the ratio between the second-smallest and smallest singular values.
\begingroup\color{black}
\endgroup

Figure~\ref{fig:test_cond} shows the fraction of identifiable instances as a function of the number of
interventions.
We observe a sharp transition from non-identifiability to full identifiability when the number of interventions becomes equal to the number of leaf components (The dashed vertical lines shows the number of leaf SCCs). This result indicates that
the theoretical bound from Theorem \ref{th:multiSCC} is close to tight in this setting.

\begin{figure}
    \centering
    \includegraphics[width=0.5\linewidth]{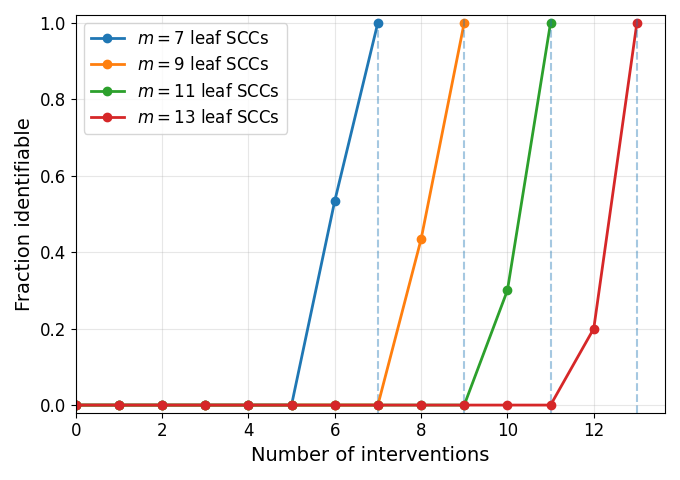}
    \caption{Identifiability versus number of interventions in the star condensation DAG.}
    \label{fig:test_cond}
\end{figure}

\begin{figure}
    \centering
    \includegraphics[width=0.5\linewidth]{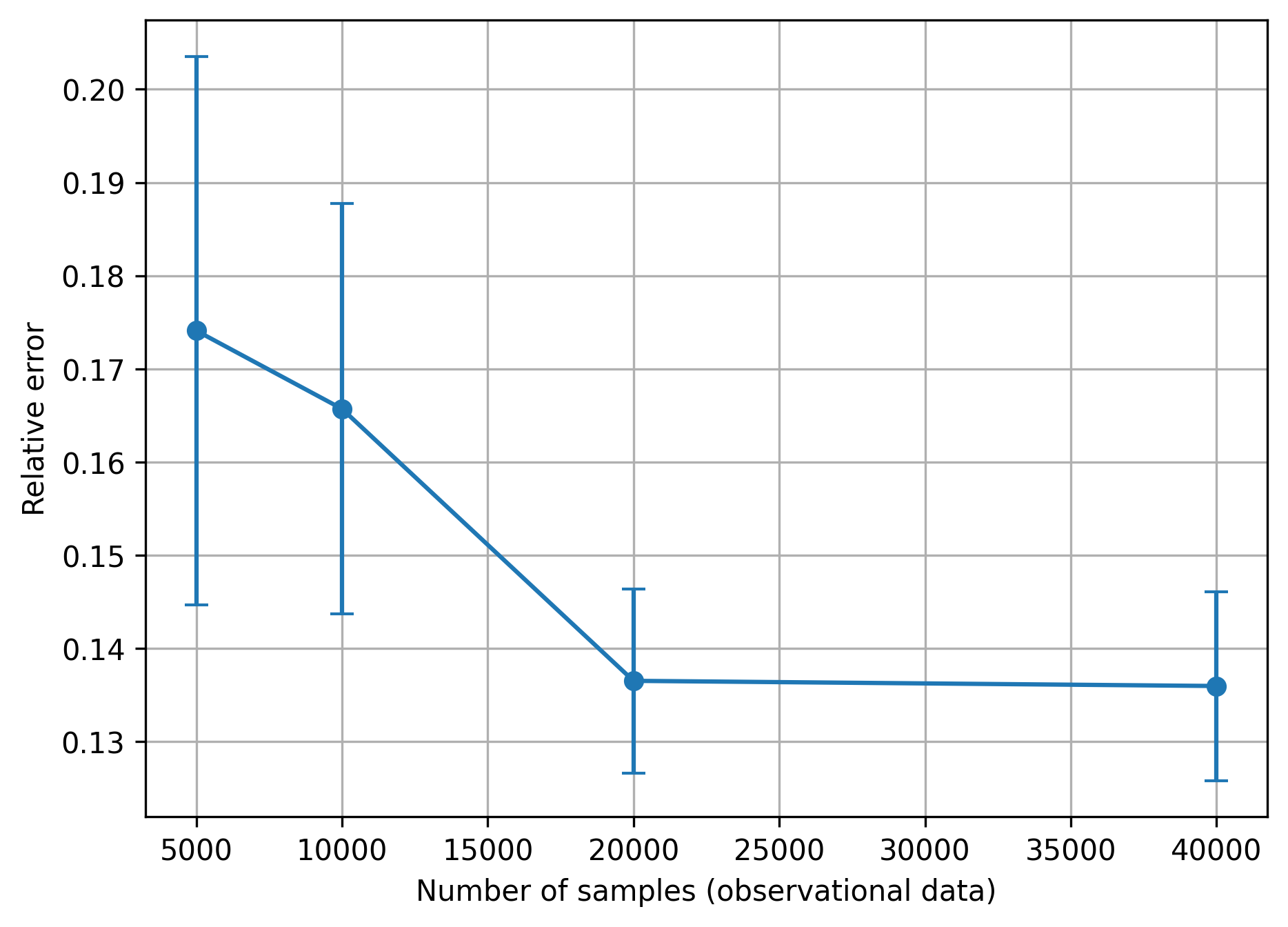}
    \caption{The average relative error for recovering $\mathbf{\Lambda}$ (up to some global scaling) versus the sample size (observational data).}
    \label{Fig:sample}
\end{figure}

\subsection{Effect of sample size on relative error}
\label{app:sample_size}
To assess the sensitivity of our estimator to the number of steady-state samples, we conducted an
experiment in which we varied the number of observational samples used to estimate
$(\hat{\boldsymbol{\mu}},\hat{\mathbf{\Sigma}})$. For each trial, the number of samples per intervention was set to
half of the observational sample size. We evaluated the relative error of recovering
$\boldsymbol{\Lambda}$ for observational sample sizes
$\{5\mathrm{k}, 10\mathrm{k}, 20\mathrm{k}, 40\mathrm{k}\}$ over 30 random instances with $n=10$. As shown in
Fig. \ref{Fig:sample}, the estimation error decreases markedly when increasing the number of samples from
$5\mathrm{k}$ to $20\mathrm{k}$ and then stabilizes, indicating that the estimator achieves
accurate recovery with a moderate number of steady-state samples. This behavior is expected, since
both empirical means and covariances become sufficiently accurate in this regime, after which
additional samples yield diminishing returns.

\subsection{Finite-sample SCC recovery}
\label{app:scc_recovery_experiment}

We evaluate the SCC-recovery procedure implied by
Theorem~\ref{th:learning_SCC_mu} and Remark~\ref{remak:scc_recovery} in a
finite-sample setting. For each trial, we generate a random OU system with
\(n=10\) nodes. The nodes are first partitioned into random SCCs of size between
\(1\) and \(3\). The condensation graph is then generated as a random DAG by
allowing edges only from earlier SCCs to later SCCs, with edge probability
\(0.3\). Within each non-singleton SCC, we include a directed cycle to ensure
strong connectivity and add extra within-SCC edges with probability \(0.2\).
For each edge in the condensation DAG, we add one randomly selected inter-SCC
edge.

The drift matrix \(\boldsymbol{\Lambda}\) is made positive stable by strict row
diagonal dominance: after assigning all off-diagonal weights, we set
\(\Lambda_{ii}=\sum_{j\neq i}|\Lambda_{ij}|+0.5\). The input vector is sampled
as \(b_i\sim \mathrm{Unif}[0.5,1.5]\), and the diagonal diffusion entries are
sampled as \(d_i\sim \mathrm{Unif}[0.2,0.5]\). The observational stationary mean
and covariance are computed from
\(\boldsymbol{\mu}=\boldsymbol{\Lambda}^{-1}\mathbf b\) and the Lyapunov
equation
$
\boldsymbol{\Lambda}\boldsymbol{\Sigma}
+
\boldsymbol{\Sigma}\boldsymbol{\Lambda}^{\top}
=
\mathbf D$.

We use one intervention target per SCC, chosen as one representative node from
that SCC. For hard interventions, we zero all off-diagonal entries in the
intervened row while keeping the diagonal entry fixed. For soft interventions,
we multiply the off-diagonal entries in the intervened row by a factor
\(\alpha=0.4\), again keeping the diagonal entry fixed. For both observational
and interventional settings, samples are drawn from the corresponding stationary
Gaussian distributions.

To estimate the response set
\(R_i=\{j:\mu_j^{(i)}\neq\mu_j\}\), we perform coordinate-wise two-sample
\begingroup\color{black}
\(z\)-tests comparing observational and interventional samples.
\endgroup
The estimated response sets are used to reconstruct SCCs as follows: empty response sets are treated as
singleton source SCCs; identical non-empty response sets are grouped together;
strict inclusions among response sets recover the SCC-level partial order; and
subtracting descendant response sets recovers the individual SCC node sets.

We repeat this procedure over \(100\) independently generated OU systems. The
quality of the recovered SCC partition is measured by the Adjusted Rand Index
(ARI), where \(1\) indicates exact recovery. We report results for both hard and
soft interventions as a function of sample size.

\subsection{Experiment with non-diagonal diffusion}
\label{app:nondiag_diffusion}
In this experiment, the data generator was modified to sample a full positive
definite diffusion matrix \(\boldsymbol{\Omega}\) instead of a diagonal
\(\mathbf D\). Specifically, we generated a random matrix \(\mathbf A\), formed
\(\mathbf A\mathbf A^{\top}\), normalized its scale, and added a positive
diagonal jitter. The steady-state covariance was then obtained by solving the
full Lyapunov equation
$
\boldsymbol{\Lambda}\boldsymbol{\Sigma}
+
\boldsymbol{\Sigma}\boldsymbol{\Lambda}^{\top}
=
\boldsymbol{\Omega}$.
Thus, the data-generating process no longer satisfies the diagonal diffusion
assumption. The learning model was also modified to estimate a full positive
definite diffusion matrix. We parameterized
\(\boldsymbol{\Omega}\) as
\(\boldsymbol{\Omega}=\mathbf M\mathbf M^{\top}+\epsilon\mathbf I\), where
\(\mathbf M\) is a learned lower-triangular matrix. Accordingly, the covariance
loss was changed from separate diagonal and off-diagonal residuals to the full
matrix residual
$
\boldsymbol{\Lambda}^{(c)}\widehat{\boldsymbol{\Sigma}}^{(c)}
+
\widehat{\boldsymbol{\Sigma}}^{(c)}
\boldsymbol{\Lambda}^{(c)\top}
-
\boldsymbol{\Omega},$
summed over the observational and interventional contexts \(c\). These changes
ensure that both data generation and optimization use a full diffusion matrix,
rather than treating diffusion as diagonal.

\section{About the Spectral/Rank Assumptions}
\label{app:Spec_assumption}
\subsection{Numerical Analysis of Spectral/Rank Assumptions}
\begin{figure}[t]
  \centering
  \begin{subfigure}[t]{.24\textwidth}
    \centering
    \includegraphics[width=\linewidth]{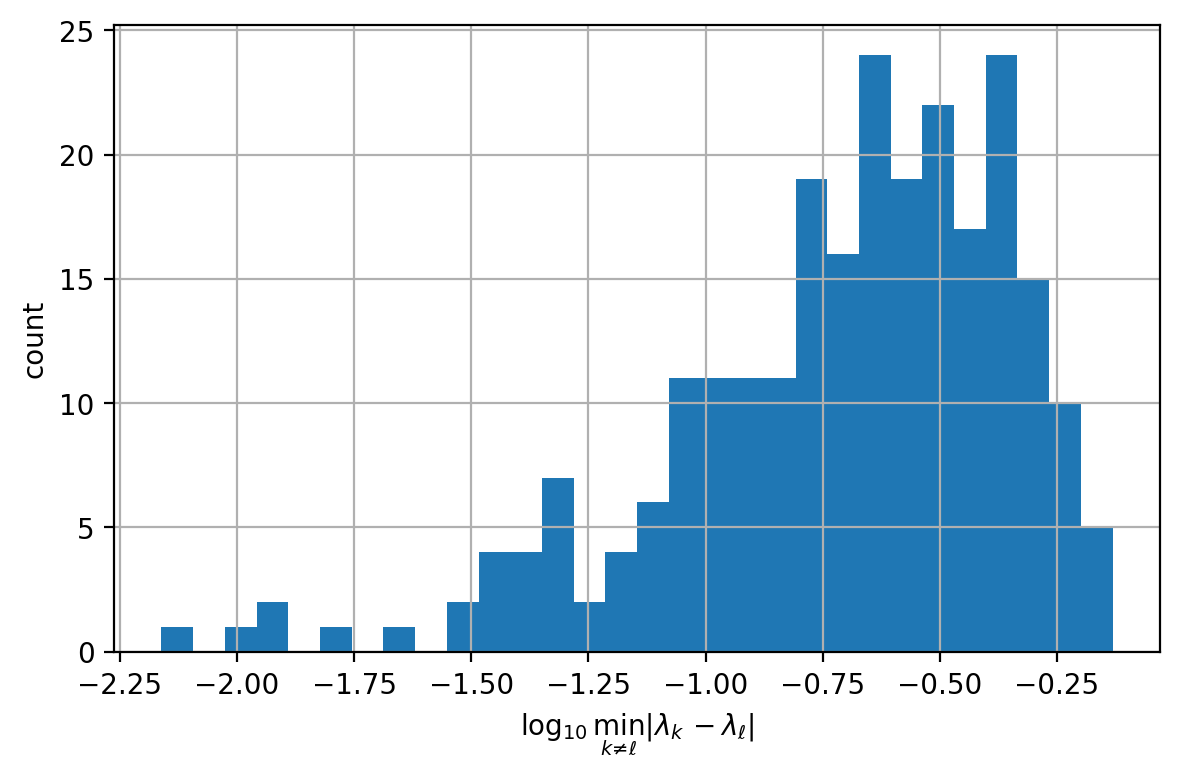}
    \caption{Min. eigenvalue spacing of $\boldsymbol{\Lambda}$}
    \label{figapp:sub1}
  \end{subfigure}%
  \begin{subfigure}[t]{.24\textwidth}
    \centering
    \includegraphics[width=\linewidth]{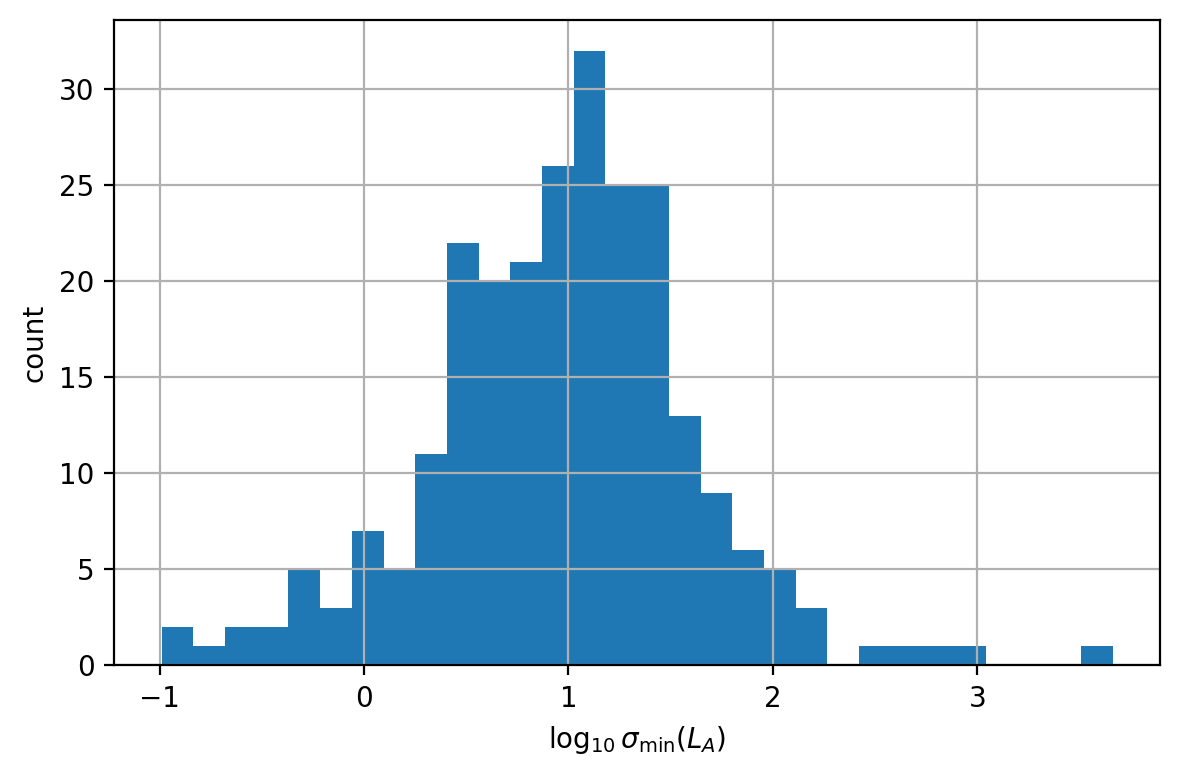}
    \caption{Min. singular value of $\mathbf{L}_{\boldsymbol{A}}$}
    \label{figapp:sub2}
  \end{subfigure}%
  \begin{subfigure}[t]{.24\textwidth}
    \centering
    \includegraphics[width=\linewidth]{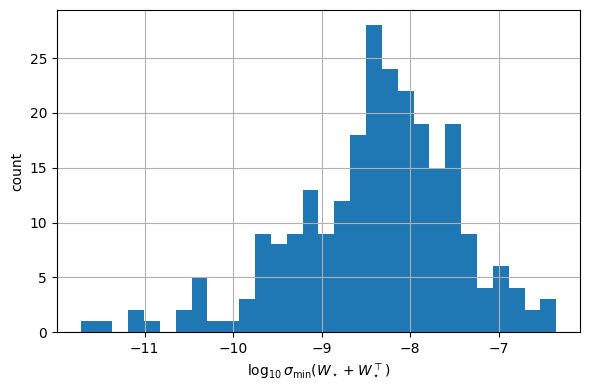}
    \caption{Assum. \ref{assum:lambda_W}(2).}
    \label{figapp:sub3}
  \end{subfigure}
  \begin{subfigure}[t]{.24\textwidth}
    \centering
    \includegraphics[width=\linewidth]{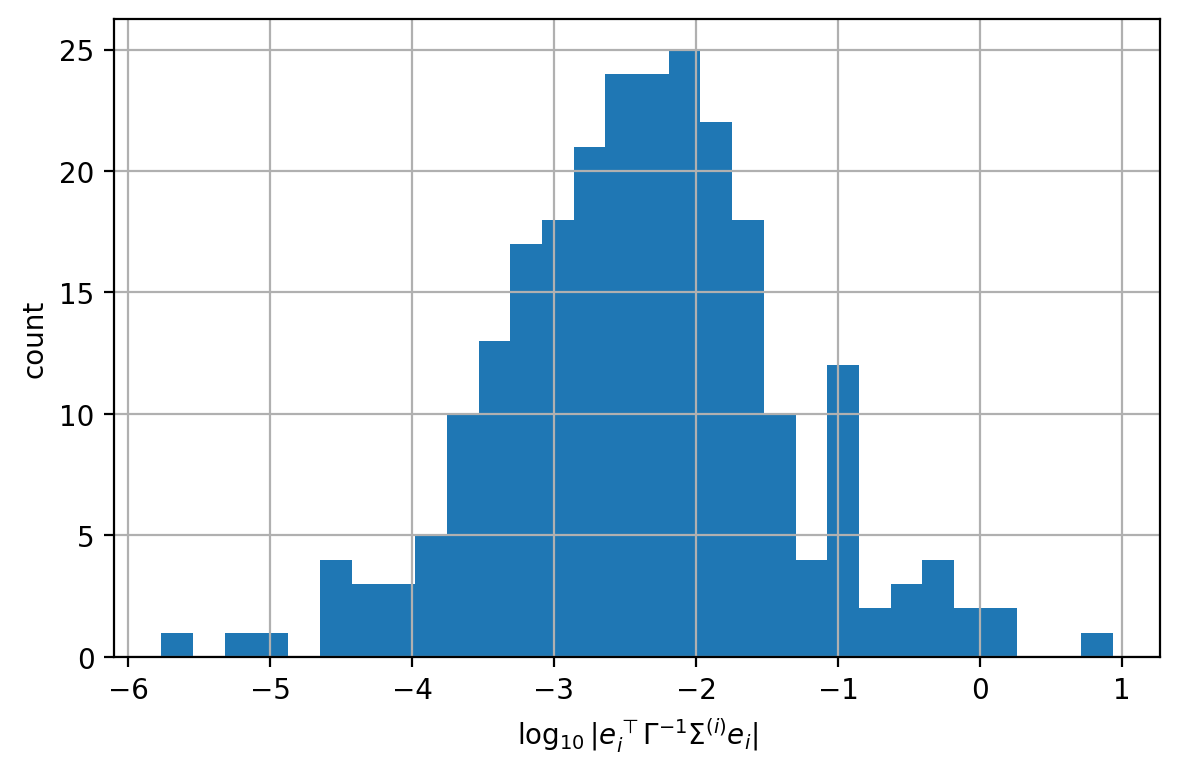}
    \caption{Assum. \ref{assum:lambda_W}(3).}
    \label{figapp:sub4}
  \end{subfigure}
  \caption{Numerical probe of spectral/rank assumptions: a) The histogram of minimum pairwise eigenvalue spacing of $\boldsymbol{\Lambda}$. b) The histogram of minimum singular of $\mathbf{L}_{\boldsymbol{A}}$. c) The condition in Assumption \ref{assum:lambda_W}(2) d) The condition in Assumption \ref{assum:lambda_W}(3) (The values on the x-axis are in $\log_{10}$ scale)  }
  \label{figapp:all}
\end{figure}

We assessed the spectral nondegeneracy assumptions by Monte Carlo analysis on random, strongly connected OU models. In each of $250$ trials, we sampled a drift matrix $\boldsymbol{\Lambda}\in\mathbb{R}^{10\times10}$ with off–diagonal pattern drawn i.i.d.\ at density $0.3$ (Gaussian weights, standard deviation $0.8$), and set each diagonal entry strictly larger than the $\ell_{1}$ row sum (ensuring positive stability). Strong connectivity was checked on the directed graph with edge $j\!\to\! i$ iff $\Lambda_{ij}\neq 0$. We drew a diagonal diffusion $\mathbf{D}$ with $d_k\sim\mathrm{Unif}[0.2,0.4]$, selected an intervention coordinate $i$ uniformly at random, and formed the intervened drift by zeroing row $i$’s off–diagonals (keeping $\Lambda_{ii}$). We then solved the continuous–time Lyapunov equations to obtain the observational and intervened covariances $\boldsymbol{\Sigma}$ and $\boldsymbol{\Sigma}^{(i)}$, formed $\boldsymbol{\Gamma}:=\boldsymbol{\Sigma}-\boldsymbol{\Sigma}^{(i)}$, and built the $( -i,-i)$ Schur blocks $\boldsymbol{\Xi}$ and $\mathbf{Z}$ used in the proofs, with $\mathbf{A}:=\boldsymbol{\Xi}^{-1}\mathbf{Z}$. We evaluated four quantities: (i) the simple-spectrum condition for
\(\boldsymbol{\Lambda}\), measured by
\(\min_{k\neq \ell}|\lambda_k-\lambda_\ell|\); (ii) the single-SCC
skew-kernel condition in Assumption~\ref{ass:moment-nondeg}, measured by
\(\sigma_{\min}(\mathbf L_{\mathbf A})\), where \(\mathbf L_{\mathbf A}\)
represents
\(\mathbf S\mapsto
\operatorname{offdiag}(\mathbf S\mathbf A-\mathbf A^\top\mathbf S)\)
on the skew-symmetric subspace; (iii) the condition in Assumption \ref{assum:lambda_W}(2) and (iv) the condition in Assumption \ref{assum:lambda_W}(3),
measured by
\(|\mathbf e_i^\top\boldsymbol{\Gamma}^{-1}
\boldsymbol{\Sigma}^{(i)}\mathbf e_i|\).
We report the resulting histograms in Figure~\ref{figapp:all}. In all \(250\)
\begingroup\color{black}
trials, the four quantities were strictly separated from zero. This is numerical evidence for the sampled models, not a proof of genericity for every SCC pattern.
\endgroup

\subsection{About the role of spectral/rank assumptions in the proofs}
\label{app:spec_assumption_role}
- Assumption~\ref{ass:moment-nondeg} holds on a nonempty open subset of
the ambient \(\mathbf A\)-space. For example, if
\(\mathbf A=\operatorname{diag}(a_1,\dots,a_m)\) with
\(a_p\neq a_q\) for \(p\neq q\), then for any skew-symmetric \(\mathbf S\),
\((\mathbf S\mathbf A-\mathbf A^\top\mathbf S)_{pq}
=(a_q-a_p)S_{pq}\) for \(p\neq q\). Hence
\(\operatorname{offdiag}(\mathbf S\mathbf A-\mathbf A^\top\mathbf S)=0\)
implies \(S_{pq}=0\) for all \(p\neq q\), and since \(\mathbf S\) is
skew-symmetric, this gives \(\mathbf S=0\).
Moreover, this injectivity is stable under sufficiently small off-diagonal
perturbations. Indeed, write
\(\mathbf A=\mathbf D+\mathbf E\), where
\(\mathbf D=\operatorname{diag}(a_1,\dots,a_m)\) and
\(\delta:=\min_{p\neq q}|a_p-a_q|>0\). For any skew-symmetric \(\mathbf S\),
$
\|\operatorname{offdiag}(\mathbf S\mathbf D-\mathbf D\mathbf S)\|_F
\geq
\delta\|\mathbf S\|_F$,
whereas
$
\|\operatorname{offdiag}(\mathbf S\mathbf E-\mathbf E^\top\mathbf S)\|_F
\leq
2\|\mathbf E\|_2\|\mathbf S\|_F$.
Therefore, if \(2\|\mathbf E\|_2<\delta\), then
\(\operatorname{offdiag}(\mathbf S\mathbf A-\mathbf A^\top\mathbf S)=0\)
forces \(\mathbf S=0\). Thus Assumption~\ref{ass:moment-nondeg} holds not only
for diagonal matrices with separated diagonal entries, but also for all
sufficiently small off-diagonal perturbations of such matrices.
In the model, however, \(\mathbf A\) is determined by the induced moments, so
we state the condition directly as a checkable spectral condition on the induced
operator \(\mathcal L_{\mathbf A}\). The perturbation argument above suggests a
 direction for finding a witnesses. For instance, one may look
for parameter regimes in a single SCC with suitably tuned weights,
where the moment-derived matrix \(\boldsymbol{\Xi}^{-1}\mathbf Z\) becomes close
to a diagonal matrix with separated diagonal entries. In such a regime,
Assumption~\ref{ass:moment-nondeg} holds. 
For small SCCs the condition simplifies. When \(|C|=2\), we have \(m=1\), so the condition is vacuous. When \(|C|=3\), we
have \(m=2\), and the condition reduces to the scalar requirement
\(A_{22}\neq A_{11}\).

\begingroup\color{black}
- Assumption~\ref{assum:lambda_W} guarantees invertibility of the
covariance difference
$\boldsymbol\Gamma=\boldsymbol\Sigma-\boldsymbol\Sigma^{(i)}$.
Its factors have a Popov--Belevitch--Hautus (PBH) \citep{hespanha2009linear} interpretation.
The PBH criterion characterizes
controllability and observability of a linear system
\[
\dot{\mathbf z}=\mathbf M\mathbf z+\mathbf g u,
\qquad y=\mathbf h^{\top}\mathbf z.
\]
The pair $(\mathbf M,\mathbf g)$ is controllable if and only if
$\boldsymbol\ell^{\top}\mathbf g\ne0$ for every nonzero left
eigenvector $\boldsymbol\ell^{\top}$ of $\mathbf M$.
Similarly, $(\mathbf h^{\top},\mathbf M)$ is observable if and only if
$\mathbf h^{\top}\mathbf r\ne0$ for every nonzero right eigenvector
$\mathbf r$ of $\mathbf M$. These tests express whether the input
can excite every mode and whether the output can detect every mode.

The structural versions ask whether these properties hold generically
when the permitted entries of the matrix/input/output pair vary
independently. The graphical criterion for structural controllability
requires that every state be reachable from an input and that the
bipartite graph of $[\mathbf M\ \mathbf g]$ admit a matching of size $n$,
where $n$ is the number of states. The latter means selecting $n$
permitted entries in distinct rows and distinct columns.
Structural observability has the dual test: every state must have a
directed path to an output, and the bipartite graph of
$[\mathbf M^{\top}\ \mathbf h]^{\top}$ must admit a matching of size $n$ \citep{lin1974structural}.

In our setting, the first factor
$\boldsymbol\rho_{\star,k}^{\top}\mathbf e_i$ in
$(\mathbf W_\star)_{k\ell}$ is the PBH controllability coefficient
for $(-\boldsymbol\Lambda_\star,\mathbf e_i)$: it determines whether
an input acting on the intervened coordinate $i$ can excite drift
mode $k$.
For the second factor, set
$\mathbf h=\boldsymbol\Sigma^{(i)}\mathbf w_\star$.
Since
\[
\boldsymbol\Lambda_\star^{\top}\boldsymbol\rho_{\star,\ell}
=\lambda_{\star,\ell}\boldsymbol\rho_{\star,\ell},
\]
the coefficient
$\mathbf w_\star^{\top}\boldsymbol\Sigma^{(i)}
\boldsymbol\rho_{\star,\ell}=\mathbf h^{\top}
\boldsymbol\rho_{\star,\ell}$ is the PBH observability coefficient
for $(\mathbf h^{\top},-\boldsymbol\Lambda_\star^{\top})$:
it determines whether this covariance-derived output detects mode
$\ell$ of the transpose dynamics. Under the assumed simple spectrum,
nonvanishing of all the first factors is equivalent to controllability
of the first pair, and nonvanishing of all the second factors is
equivalent to observability of the second pair.

For these graphical tests, our independently free diagonal drift
entries supply a self-loop at every state, so the matching conditions
are automatic. Within a single SCC, every vertex is reachable from
the intervened node $i$. Hence $(-\boldsymbol\Lambda_\star,\mathbf e_i)$
is structurally controllable, and the first PBH factors are generically
nonzero under the assumed simple spectrum.
For the transpose dynamics, reverse the drift edges and check that
every vertex can reach a vertex connected to the output. In a single
SCC this reachability condition holds for any nonempty output pattern.
It establishes structural observability when the permitted output
coefficients are independently free.

These interpretations do not alone prove genericity of
Assumption~\ref{assum:lambda_W}. In particular,
$\mathbf h=\boldsymbol\Sigma^{(i)}\mathbf w_\star$ depends on the
model parameters and cannot be varied independently in a structural
genericity argument. We retain the explicit condition
$0\notin\sigma(\mathbf W_\star+\mathbf W_\star^{\top})$,
which guarantees invertibility of
$\boldsymbol\Gamma=\boldsymbol\Sigma-\boldsymbol\Sigma^{(i)}$ through
$
\boldsymbol\Gamma=-\mathbf R_\star
(\mathbf W_\star+\mathbf W_\star^{\top})\mathbf R_\star^{\top}$.
\endgroup

- The scalar condition
\(\mathbf e_i^\top\boldsymbol{\Gamma}^{-1}
\boldsymbol{\Sigma}^{(i)}\mathbf e_i\neq0\) is used in the proof to pass from
invertibility of \(\boldsymbol{\Gamma}:=\boldsymbol{\Sigma}-\boldsymbol{\Sigma}^{(i)}\)
to invertibility of the Schur-type matrix \(\boldsymbol{\Xi}\).
\begingroup\color{black}
Consider this condition on the directed-cycle graph
\(1\to2\to\cdots\to n\to1\), with intervention on node \(1\).
The intervention removes the edge \(n\to1\), leaving the directed
chain \(1\to2\to\cdots\to n\). Thus the covariance effect generated by
\endgroup
the intervention can propagate through all coordinates of the block.
By Cramer's rule,
\[
\mathbf e_1^\top\boldsymbol{\Gamma}^{-1}
\boldsymbol{\Sigma}^{(1)}\mathbf e_1
=
\frac{
\det\!\big[
\boldsymbol{\Sigma}^{(1)}\mathbf e_1,\,
\boldsymbol{\Gamma}\mathbf e_2,\ldots,
\boldsymbol{\Gamma}\mathbf e_n
\big]
}{
\det(\boldsymbol{\Gamma})
}.
\]
On the directed cycles, both the numerator and denominator are rational
functions of the cycle weights, self-loops, and diffusion variances. Since the
intervened graph contains a directed chain from the intervened node to all other
nodes, the columns appearing in the Cramer numerator are not forced to vanish by
the graph structure. However, this does not by itself prove that the determinant
is nonzero. To prove genericity on directed cycles, or on general SCCs, one
would still need to exhibit one admissible parameter value for which the condition holds.

\paragraph{Witness for directed cycle}
\begingroup\color{black}
When comparing drifts supported on the same directed cycle, one can avoid the auxiliary
\endgroup
\(\mathbf A=\boldsymbol{\Xi}^{-1}\mathbf Z\) skew-kernel argument used in the
general single-SCC proof. Instead, we work directly with the
\(\boldsymbol{\Xi}\)-equation
$
\boldsymbol{\Lambda}_{\Delta,-i,-i}\boldsymbol{\Xi}
+
\boldsymbol{\Xi}^{\top}
\boldsymbol{\Lambda}_{\Delta,-i,-i}^{\top}
=
0$,
which was derived in \eqref{eq:Xi-eq}. When the SCC is the directed cycle
\(1\to2\to\cdots\to n\to1\) and the intervention is on node \(1\), the hard
intervention removes the edge \(n\to1\). Hence, on the remaining nodes
\(2,\dots,n\), the support of
$
\mathbf X:=\boldsymbol{\Lambda}_{\Delta,-1,-1}$
is lower bidiagonal, corresponding to the directed chain
$
2\to3\to\cdots\to n$.
For this special support pattern, the adjacent-minor condition on
\begingroup\color{black}
\(\boldsymbol{\Xi}\) given below forces \(\mathbf X=0\). This gives a direct
injectivity argument for cycle-supported competitors; it does not establish
Assumption~\ref{ass:moment-nondeg} for unrestricted competing drifts.
\endgroup

Consider the directed cycle
$
1\to2\to\cdots\to n\to1$,
with intervention on node \(1\). After the hard intervention, the edge
\(n\to1\) is removed, and the remaining support on nodes \(2,\dots,n\) is the
directed chain
$
2\to3\to\cdots\to n$.
Thus the ambiguity
$
\mathbf X:=\boldsymbol{\Lambda}_{\Delta,-1,-1}$
has lower-bidiagonal support in the natural ordering \(2,\dots,n\). That is, if
\begingroup\color{black}
\(m=n-1\), reduced index $p$ corresponds to original node $p+1$, and
\endgroup
$
X_{pq}=0 
\text{ unless } 
p=q\ \text{or}\ p=q+1 .
$

Recall that the proof gives the equation
$
\mathbf X\boldsymbol{\Xi}
+
\boldsymbol{\Xi}^{\top}\mathbf X^\top
=
0$.
The following lemma gives a direct condition under which this equation forces
\(\mathbf X=0\).

\begin{lemma}
\label{lem:xi_chain_injectivity}
Let \(\boldsymbol{\Xi}\in\mathbb R^{m\times m}\) be arbitrary, not necessarily
symmetric. Suppose \(\mathbf X\in\mathbb R^{m\times m}\) is lower bidiagonal:
$
X_{pq}=0
\qquad\text{unless}\qquad
p=q\ \text{or}\ p=q+1$.
If
$
\mathbf X\boldsymbol{\Xi}
+
\boldsymbol{\Xi}^{\top}\mathbf X^\top
=
0$,
and
$
\Xi_{11}\neq0$,
and, for every \(p=1,\dots,m-1\),
$
\Delta_p(\boldsymbol{\Xi})
:=
\Xi_{pp}\Xi_{p+1,p+1}
-
\Xi_{p,p+1}\Xi_{p+1,p}
\neq0$,
then \(\mathbf X=0\).
\end{lemma}
\begin{proof}
Write
$
x_p:=X_{pp}, p=1,\dots,m,$
and
$
\ell_p:=X_{p+1,p}, p=1,\dots,m-1.
$
These are the only possibly nonzero entries of \(\mathbf X\). Let
$
\mathbf E:=
\mathbf X\boldsymbol{\Xi}
+
\boldsymbol{\Xi}^{\top}\mathbf X^\top .
$
By assumption, \(\mathbf E=0\). The \((1,1)\)-entry gives
$
0=E_{11}=2x_1\Xi_{11}$.
Since \(\Xi_{11}\neq0\), we get \(x_1=0\).

Now suppose, inductively, that for some \(p\in\{1,\dots,m-1\}\),
\[
x_1=\cdots=x_p=0,
\qquad
\ell_1=\cdots=\ell_{p-1}=0.
\]
Then row \(p\) of \(\mathbf X\) is zero, while row \(p+1\) has only two
possibly nonzero entries, \(\ell_p\) in column \(p\) and \(x_{p+1}\) in column
\(p+1\). Hence the \((p,p+1)\)-entry of \(\mathbf E=0\) gives
\[
\ell_p\Xi_{pp}
+
x_{p+1}\Xi_{p+1,p}
=0,
\]
and the \((p+1,p+1)\)-entry gives
\[
\ell_p\Xi_{p,p+1}
+
x_{p+1}\Xi_{p+1,p+1}
=0.
\]
Therefore,
\[
\begin{bmatrix}
\Xi_{pp} & \Xi_{p+1,p}\\
\Xi_{p,p+1} & \Xi_{p+1,p+1}
\end{bmatrix}
\begin{bmatrix}
\ell_p\\
x_{p+1}
\end{bmatrix}
=
0.
\]
Its determinant is
\[
\Delta_p(\boldsymbol{\Xi})
=
\Xi_{pp}\Xi_{p+1,p+1}
-
\Xi_{p,p+1}\Xi_{p+1,p},
\]
which is nonzero by assumption. Thus \(\ell_p=0\) and \(x_{p+1}=0\). The claim
follows by induction.
\end{proof}

For a directed cycle, define
\[
\kappa_{\Xi}
:=
\Xi_{11}
\prod_{p=1}^{m-1}
\left(
\Xi_{pp}\Xi_{p+1,p+1}
-
\Xi_{p,p+1}\Xi_{p+1,p}
\right).
\]
By Lemma~\ref{lem:xi_chain_injectivity}, if \(\kappa_\Xi\neq0\), then the
\(\boldsymbol{\Xi}\)-equation forces the directed-chain ambiguity
\(\boldsymbol{\Lambda}_{\Delta,-1,-1}\) to vanish.

We now show that \(\kappa_{\Xi}\) is not identically zero on the directed-cycle
parameter space, for every cycle length \(n\). Since \(\boldsymbol{\Xi}\) is
obtained from solutions of Lyapunov equations, its entries are rational
functions of the cycle weights, self-loops, and diffusion variances. Hence each
factor in \(\kappa_\Xi\) is a rational function of the model parameters.
\begingroup\color{black}
For these rational witnesses, it suffices to work on the algebraic domain
where both Lyapunov operators are nonsingular and
\(\Sigma^{(1)}_{11}\neq0\); positivity is not required at the intermediate
witness. A nonzero rational function cannot vanish on the nonempty open
admissible parameter domain, so these witnesses imply the stated generic
conclusion for stationary OU models. Here ``algebraic domain'' means
only that the displayed rational expressions are defined; a Lyapunov
solution there need not be a positive covariance. The witness establishes
a nonzero numerator polynomial. Appendix~\ref{app:genericity} then
transfers its generic nonvanishing to the open admissible OU domain.
\endgroup

We use induction on \(n\). The base cases are direct. For \(n=2\), take
\[
\boldsymbol{\Lambda}
=
\begin{bmatrix}
2 & 1\\
1 & 2
\end{bmatrix},
\qquad
\mathbf D=\operatorname{diag}(1,2).
\]
A direct calculation gives
\[
\boldsymbol{\Xi}=\left[\frac{3}{64}\right],
\]
so \(\kappa_\Xi\neq0\). For \(n=3\), take
\[
\boldsymbol{\Lambda}
=
\begin{bmatrix}
2 & 0 & 1\\
1 & 2 & 0\\
0 & 1 & 2
\end{bmatrix},
\qquad
\mathbf D=\operatorname{diag}(1,2,3).
\]
For the intervention on node \(1\), direct calculation gives
\[
\boldsymbol{\Xi}
=
\begin{bmatrix}
\frac{1}{4032} & -\frac{1}{8064}\\[3pt]
\frac{19}{384} & -\frac{265}{21504}
\end{bmatrix}.
\]
Thus
\[
\Xi_{11}=\frac{1}{4032}\neq0,
\qquad
\det(\boldsymbol{\Xi})=\frac{89}{28901376}\neq0.
\]
Hence \(\kappa_\Xi\neq0\) for the 3-cycle.

Assume now that the claim holds for the \(n\)-cycle. We argue it for the
\((n+1)\)-cycle. Let
\[
P_p^{(n+1)}(\theta)
:=
\Xi_{pp}^{(n+1)}\Xi_{p+1,p+1}^{(n+1)}
-
\Xi_{p,p+1}^{(n+1)}\Xi_{p+1,p}^{(n+1)},
\qquad p=1,\dots,n-1,
\]
and let \(P_0^{(n+1)}(\theta):=\Xi_{11}^{(n+1)}\). We will show that each
\(P_p^{(n+1)}\) is not identically zero as a rational function of the
\((n+1)\)-cycle parameters.

First consider \(p=0,\dots,n-2\), i.e., the factors involving only the old
\begingroup\color{black}
reduced nodes \(2,\dots,n\). By the induction hypothesis and rational genericity, choose an
\(n\)-cycle in this nonsingular algebraic domain whose corresponding \(\boldsymbol{\Xi}^{(n)}\)-factors are all
\endgroup
nonzero. We now insert node \(n+1\) as a fast relay between \(n\) and \(1\).

Let \(\boldsymbol{\Lambda}^{\rm old}\in\mathbb R^{n\times n}\) be the drift of
the old \(n\)-cycle, and let
$\omega:=\Lambda^{\rm old}_{1n}$
be the closing-edge weight \(n\to1\). Define
$
\mathbf A:=\boldsymbol{\Lambda}^{\rm old}
-\omega\mathbf e_1\mathbf e_n^\top$,
so that \(\mathbf A\) is the old drift with the closing edge removed. Write
$
\mathbf x:=(x_1,\dots,x_n)^\top,
z:=x_{n+1}$.
\begingroup\color{black}
For parameters \(\alpha,c>0\) and \(\beta\in\mathbb R\) satisfying
\endgroup
$
\frac{\alpha\beta}{c}=\omega$,
define the relay drift
\[
\boldsymbol{\Lambda}_\varepsilon
=
\begin{bmatrix}
\mathbf A & \beta\mathbf e_1\\[2pt]
-\alpha\varepsilon^{-1}\mathbf e_n^\top & c\varepsilon^{-1}
\end{bmatrix}.
\]
The relay coordinate has drift
$
-\frac{c}{\varepsilon}z+\frac{\alpha}{\varepsilon}x_n
=
-\frac{c}{\varepsilon}
\left(z-\frac{\alpha}{c}x_n\right)$,
so as \(\varepsilon\to0\), the relay enforces
$
z\approx \frac{\alpha}{c}x_n$.
Therefore the edge \(n+1\to1\) contributes
$
-\beta z
\approx
-\frac{\alpha\beta}{c}x_n
=
-\omega x_n$,
which is exactly the contribution of the old closing edge \(n\to1\).

We verify this at the covariance level. Let
\[
\boldsymbol{\Sigma}_\varepsilon
=
\begin{bmatrix}
\mathbf S_\varepsilon & \mathbf r_\varepsilon\\
\mathbf r_\varepsilon^\top & s_\varepsilon
\end{bmatrix}
\]
\begingroup\color{black}
be the solution of the Lyapunov equation for \((\mathbf x,z)\), with
\(\mathbf S_\varepsilon\) the old-variable block and
\(\mathbf r_\varepsilon\) the cross block. Keep the diagonal diffusion
parameters fixed as \(\varepsilon\to0\), with block form
\endgroup
\[
\mathbf D_\varepsilon
=
\begin{bmatrix}
\mathbf D_x&0\\
0&d_z
\end{bmatrix}.
\]
The \((\mathbf x,z)\)-block of the Lyapunov equation gives
\[
\mathbf A\mathbf r_\varepsilon
+
\beta\mathbf e_1s_\varepsilon
-
\alpha\varepsilon^{-1}\mathbf S_\varepsilon\mathbf e_n
+
c\varepsilon^{-1}\mathbf r_\varepsilon
=
0.
\]
Multiplying by \(\varepsilon\) and rearranging,
\[
\mathbf r_\varepsilon
=
\frac{\alpha}{c}\mathbf S_\varepsilon\mathbf e_n
-
\frac{\varepsilon}{c}
\left(
\mathbf A\mathbf r_\varepsilon+\beta\mathbf e_1s_\varepsilon
\right).
\]
\begingroup\color{black}
To justify boundedness, also multiply the \((z,z)\)-equation by
\(\varepsilon\), obtaining
\[
2c s_\varepsilon-2\alpha\mathbf e_n^\top\mathbf r_\varepsilon
=\varepsilon d_z.
\]
At \(\varepsilon=0\), the scaled cross and scalar equations give
\(\mathbf r=(\alpha/c)\mathbf S\mathbf e_n\) and
\(s=(\alpha/c)^2\mathbf e_n^\top\mathbf S\mathbf e_n\).
Substitution into the old-variable equation gives exactly the old Lyapunov
equation. Its nonsingularity therefore makes the full scaled linear system
nonsingular at zero: its homogeneous system has only the zero solution.
In general, if $K(\varepsilon)$ is continuous and $K(0)$ is
invertible, then $K(\varepsilon)^{-1}$ is continuous near zero, as follows
from the adjugate formula and continuity of the nonzero determinant.
Applied here to the matrix of the scaled linear equations, with its
continuous right-hand side, this gives a bounded solution near zero.
The invertible matrix in this argument is the equation operator, not
the limiting covariance.
For \(\varepsilon\neq0\), scaling the equations is reversible. Hence
\endgroup
\[
\mathbf r_\varepsilon
=
\frac{\alpha}{c}\mathbf S_\varepsilon\mathbf e_n
+
O(\varepsilon).
\]
\begingroup\color{black}
Here $R_\varepsilon=O(\varepsilon)$ means
$\|R_\varepsilon\|\le C\varepsilon$ for sufficiently small positive
$\varepsilon$, with the other parameters fixed. We may use any fixed
norm on this finite-dimensional space.
\endgroup

The \((\mathbf x,\mathbf x)\)-block of the Lyapunov equation is
\[
\mathbf A\mathbf S_\varepsilon
+
\mathbf S_\varepsilon\mathbf A^\top
+
\beta\mathbf e_1\mathbf r_\varepsilon^\top
+
\beta\mathbf r_\varepsilon\mathbf e_1^\top
=
\mathbf D_x.
\]
Substituting the previous relation gives
\[
\left(
\mathbf A+\frac{\alpha\beta}{c}\mathbf e_1\mathbf e_n^\top
\right)\mathbf S_\varepsilon
+
\mathbf S_\varepsilon
\left(
\mathbf A+\frac{\alpha\beta}{c}\mathbf e_1\mathbf e_n^\top
\right)^\top
=
\mathbf D_x+O(\varepsilon).
\]
Since \(\alpha\beta/c=\omega\), the drift in parentheses is
\[
\mathbf A+\omega\mathbf e_1\mathbf e_n^\top
=
\boldsymbol{\Lambda}^{\rm old}.
\]
Therefore, by uniqueness and continuity of Lyapunov solutions,
\[
\mathbf S_\varepsilon
\longrightarrow
\boldsymbol{\Sigma}^{\rm old}
\qquad
\text{as }\varepsilon\to0.
\]

The same argument applies to the intervened covariance. Under the intervention
on node \(1\), the edge \(n+1\to1\) is removed, so the block
\(\beta\mathbf e_1\) is set to zero. Thus the effective closing edge also
disappears, matching the intervened old \(n\)-cycle. Hence the observational and
interventional covariance blocks on the old variables converge to those of the
old \(n\)-cycle.

Consequently, the corresponding old block of
\(\boldsymbol{\Xi}_\varepsilon\) converges to the old
\(\boldsymbol{\Xi}^{(n)}\)-matrix. Therefore, for \(p=0,\dots,n-2\), the factors
\(P_p^{(n+1)}\) converge to the nonzero factors of the old \(n\)-cycle. Hence
each \(P_p^{(n+1)}\), \(p=0,\dots,n-2\), is not identically zero as a rational
function of the \((n+1)\)-cycle parameters.

It remains to show that the last factor
\[
P_{n-1}^{(n+1)}
=
\Xi_{n-1,n-1}^{(n+1)}\Xi_{n,n}^{(n+1)}
-
\Xi_{n-1,n}^{(n+1)}\Xi_{n,n-1}^{(n+1)}
\]
\begingroup\color{black}
is not identically zero. Start from the displayed 3-cycle on the retained
nodes \(\{1,n,n+1\}\), whose adjacent minor is nonzero. Insert the intermediate
nodes \(2,\dots,n-1\) one at a time along the edge from \(1\) to \(n\).
The same one-relay calculation applies to any subdivided edge: removing the
incoming edge at node \(1\) commutes with this construction.
Indeed, the inserted relays subdivide the outgoing path from $1$ to $n$,
leaving the incoming edge $n+1\to1$ untouched. Cutting that edge before
or after insertion gives the same model, so the covariance convergence
applies to both observational and interventional equations. At each insertion,
choose a sufficiently small positive relay time after fixing the preceding
parameters. Continuity preserves the nonzero minor on the two retained nodes
\(\{n,n+1\}\), together with the two Lyapunov nonsingularity conditions and
the nonzero interventional variance at node \(1\). After the finitely many
insertions, the graph is the \((n+1)\)-cycle and that retained minor is exactly
\(P_{n-1}^{(n+1)}\): nodes $\{n,n+1\}$ occupy reduced positions
$\{n-1,n\}$. Thus this final rational factor is not identically zero.
\endgroup

We have shown that every factor \(P_p^{(n+1)}\), \(p=0,\dots,n-1\), is a
nonzero rational function. Since the product of finitely many nonzero rational
functions is again nonzero, we conclude that
\[
P_0^{(n+1)}P_1^{(n+1)}\cdots P_{n-1}^{(n+1)}
\]
is not identically zero. Hence there exists an admissible \((n+1)\)-cycle
parameter value for which all adjacent \(\boldsymbol{\Xi}\)-minors are nonzero.
This completes the induction.

Consequently, for every directed cycle length \(n\), the adjacent-\(\boldsymbol{\Xi}\)
condition holds generically on the directed-cycle parameter subfamily. By
Lemma~\ref{lem:xi_chain_injectivity}, the equation
\[
\boldsymbol{\Lambda}_{\Delta,-1,-1}\boldsymbol{\Xi}
+
\boldsymbol{\Xi}^{\top}
\boldsymbol{\Lambda}_{\Delta,-1,-1}^{\top}
=0
\]
has no nonzero solution with directed-chain support, generically on this
subfamily.

\end{document}